\documentclass[11pt, a4paper]{article}

\usepackage{XCharter}

\usepackage[margin=1.5cm]{geometry}

\usepackage[dvipsnames, table]{xcolor}
\usepackage[colorlinks=true, citecolor=Blue, linkcolor=BrickRed]{hyperref}

\usepackage{amsfonts}
\usepackage[tbtags]{mathtools}
\usepackage{physics}
\usepackage{nicefrac}
\usepackage{empheq}
\usepackage{textcomp} % to use \textpm
\usepackage{bbm}

\usepackage{amsthm}
\usepackage{thmtools}
\usepackage{thm-restate}

\newtheorem{theorem}{Theorem}

\newtheorem{lemma}[theorem]{Lemma}
\newtheorem{corollary}[theorem]{Corollary}

\newtheorem{remark}{Remark}
\newtheorem{definition}{Definition}

\usepackage[normalem]{ulem}

\usepackage{dirtytalk}

\usepackage[capitalize,noabbrev]{cleveref}
\crefname{assumption}{Assumption}{Assumptions}
\crefname{equation}{Eq.}{Eqs.}

\usepackage{algorithm}
\usepackage{algpseudocodex}

\usepackage[subrefformat=parens,labelformat=parens]{subcaption}
\usepackage{pgfplots}
\usepackage{pgfplotstable}
\pgfplotsset{compat=1.17}
\usetikzlibrary{calc}
\usepgfplotslibrary{external}
\usepackage{booktabs}
\usepackage{multirow}
\usepackage{tablefootnote}

\usepackage{enumitem}

\usepackage{microtype}
\usepackage{changepage}

\usepackage{pifont}

\usepackage[style=authoryear, natbib=true, sorting=nyt, maxbibnames=9, maxcitenames=2, uniquelist=false, backend=biber]{biblatex}
\newif\ifhighlight

\highlightfalse
\ifhighlight
    \newcommand{\revise}[1]{\textcolor{Magenta}{#1}}
    \newcommand{\remove}[1]{\sout{#1}}
\else
    \newcommand{\revise}[1]{#1}
    \newcommand{\remove}[1]{}
\fi

\usepackage{authblk}

\begin{document}

    \title{Task Weighting in Meta-learning with Trajectory Optimisation}
    
    \author[1]{Cuong Nguyen\thanks{cuong.nguyen@adelaide.edu.au}}
    \author[2]{Thanh-Toan Do\thanks{Toan.Do@monash.edu}}
    \author[3]{Gustavo Carneiro\thanks{g.carneiro@surrey.ac.uk}}
    \affil[1]{School of Computer Science, University of Adelaide, Australia}
    \affil[2]{Department of Data Science and AI, Monash University, Australia}
    \affil[3]{Centre for Vision, Speech and Signal Processing, University of Surrey, United Kingdom}
    
    \date{\vspace{-1em}}
    
    \maketitle
    
    \begin{abstract}
    Developing meta-learning algorithms that are un-biased toward a subset of training tasks often requires hand-designed criteria to weight tasks, potentially resulting in sub-optimal solutions. In this paper, we introduce a new principled and fully-automated task-weighting algorithm for meta-learning methods. By considering the weights of tasks within the same mini-batch as an action, and the meta-parameter of interest as the system state, we cast the task-weighting meta-learning problem to a trajectory optimisation and employ the iterative linear quadratic regulator to determine the optimal action or weights of tasks. We theoretically show that the proposed algorithm converges to an \(\epsilon_{0}\)-stationary point, and empirically demonstrate that the proposed approach out-performs common hand-engineering weighting methods in two few-shot learning benchmarks.
\end{abstract}
    
    \section{Introduction}
\label{sec:introduction}
    Meta-learning has been studied from the early 1990s~\citep{schmidhuber1987evolutionary, naik1992meta, thrun1998learning} and recently gained a renewed %\sout{research of}
    interest with the use of deep learning methods that achieves remarkable state-of-art results in several few-shot learning benchmarks~\citep{vinyals2016matching, finn2017model, snell2017prototypical, nichol2018first, ravi2018amortized, allen2019infinite, khodak2019adaptive, baik2020meta, flennerhag2020meta}. However, the majority of existing meta-learning algorithms simply minimise the average loss evaluated on validation subsets of training tasks, implicitly assuming \remove{that all training tasks are evenly distributed} \revise{task-balance (analogous to class balance in single-task learning)}. This assumption is hardly true in practice, and potentially biases the trained meta-learning models toward tasks observed more frequently during training, and consequently, resulting in a large variation of performance when evaluating on different subsets of testing tasks as shown in~\citep[\figureautorefname~1]{dhillon2019baseline} and \citep[Figure 1]{nguyen2021probabilistic}.

    One way to address such issue is to exploit the diversity of training tasks, so that the trained meta-learning models can generalise to a wider range of testing tasks. In fact, various studies in task relatedness or task similarity have shown that learning from certain tasks may facilitate the generalisation of meta-learning models~\citep{thrun1996discovering, zamir2018taskonomy, achille2019task2vec, nguyen2021probabilistic}. This suggests the design of a re-weighting mechanism to diversify the contribution of each training task when training a meta-learning model of interest. Existing re-weighting methods mostly rely on either hand-crafted criteria to determine those weights~\citep{collins2020task}, or additional validation tasks to learn the re-weighting factors of interest~\citep{xu2021meta}. Such ad-hoc development of re-weighting mechanisms motivates us to design a more principled approach to re-weight tasks for meta-learning. \revise{We note that there are also studies learning to balance the contribution of each training task, e.g., learning to balance~\citep{lee2020learning}. However, such method focuses on the task-adaptation step (also known as inner loop), while our interest is to explicitly weight the contribution of each task at the meta-learning step (also known as outer-loop).}

    In this paper, we present a new principled and fully-automated task-weighting algorithm, called \textbf{t}rajectory \textbf{o}ptimisation based task \textbf{w}eighting for meta-learning (TOW). We note that TOW is not a meta-learning method, but a task weighting framework that can be integrated into existing meta-learning algorithms to circumvent the problematic assumption about the even distribution of training tasks. Our contributions can be summarised as follows:
    \begin{itemize}%[topsep=0.125em]
        \item We propose to cast the task-weighting problem in meta-learning to a finite-horizon discrete-time trajectory optimisation with state denoted by the meta-parameter and action by the re-weight factors of tasks, and solve such problem using the iterative linear quadratic regulator.
        \item We prove that under the conditions of boundedness and smoothness of the loss function used, TOW converges to a particular \(\epsilon_{0}\)-stationary point.
        \item We demonstrate TOW's functionality by incorporating it into two common meta-learning algorithms, namely MAML~\citep{finn2017model} and Prototypical Networks~\citep{snell2017prototypical}, and showing that TOW enables meta-learning methods to converge with fewer \remove{training tasks} \revise{number of iterations} and achieves higher prediction accuracy than some common task re-weighting mechanisms in the literature.
    \end{itemize}
    % TOW models the training phase in meta-learning as a finite-horizon discrete-time trajectory optimisation with state denoted by the meta-parameter and action by the re-weight factors of tasks. The meta-learning problem can, therefore, be solved using iterative linear quadratic regulator~\citep{todorov2005generalized, tassa2012synthesis} -- an optimisation method in optimal control. We prove that under the conditions of boundedness and smoothness of the loss function used, TOW converges to a particular \(\epsilon\)-stationary point. We demonstrate TOW's functionality by incorporating it into two common meta-learning algorithms, namely MAML~\citep{finn2017model} and Prototypical Networks~\citep{snell2017prototypical}. The empirical results show that TOW enables meta-learning methods to converge faster using fewer training tasks and achieves higher prediction accuracy than some common task re-weighting mechanisms in the literature.
    \section{Background}
\label{sec:background}
    \subsection{Trajectory optimisation}
    \label{sec:TO_problem_statement}
        Given continuous state \(\mathbf{x} \in \mathbb{R}^{D}\) and action \(\mathbf{u} \in \mathbb{R}^{M}\), the objective of a trajectory optimisation is to find an optimal sequence of actions \(\{\mathbf{u}_{t}^{*}\}_{t=1}^{T}\) that minimises the total cost:
        \begin{equation}
            \begin{split}
                \min_{ \{ \mathbf{u}_{t}\}_{t=1}^{T}} & \sum_{t = 1}^{T} c(\mathbf{x}_{t}, \mathbf{u}_{t}) \quad \text{s.t. } \mathbf{x}_{t + 1} = f(\mathbf{x}_{t}, \mathbf{u}_{t}),
            \end{split}
            \label{eq:TO_problem}
        \end{equation}
        where \(c(\mathbf{x}_{t}, \mathbf{u}_{t})\) and \(f(\mathbf{x}_{t}, \mathbf{u}_{t})\) are the cost function and the state-transition dynamics at time step \(t\), respectively. These functions are assumed to be twice differentiable. In addition, the initial state \(\mathbf{x}_{1}\) is given, and \emph{trajectory optimisation} means finding the optimal sequence of actions \(\{ \mathbf{u}_{t}^{*} \}_{t = 1}^{T}\) for a particular \(\mathbf{x}_{1}\), not for all possible initial states.

        % One could solve the constrained optimisation in \eqref{eq:TO_problem} by substituting \(\mathbf{x}_{t + 1}\) in the state-transition dynamics into the cost function. The resulting unconstrained optimisation can then be solved using a non-linear optimisation solver to find the local optimal for weighting vector \(\mathbf{u}_{t}\). However, the substitution of \(\mathbf{x}_{t + 1}\) results in many deeply-composite terms, easily leading to gradient vanishing or gradient explosion if gradient-based optimiser is used.
        
        In trajectory optimisation, the finite-horizon discrete-time problem shown in \eqref{eq:TO_problem} can be solved approximately by iterative methods, such as differential dynamic programming (DDP)~\citep{jacobson1970differential} or iterative linear quadratic regulator (iLQR)~\citep{todorov2005generalized,tassa2012synthesis}. These methods rely on a local approximation of the state-transition dynamics and cost function using Taylor series about a nominal trajectory \(\{\hat{\mathbf{x}}_{t}, \hat{\mathbf{u}}_{t}\}_{t=1}^{T}\). In DDP, both the state-transition dynamics and cost function are approximated to the second order of their Taylor series, while in iLQR -- a \say{simplified} version of DDP, the state-transition dynamics is approximated up to the first order. In a loose sense, DDP is analogy to the Newton's method, while iLQR is analogous to a quasi-Newton's method.

        The main idea of iLQR is to cast a general non-linear trajectory optimisation problem shown in \eqref{eq:TO_problem} to a linear quadratic problem (LQP) in which the state-transition dynamics is linear and the cost function is quadratic. The approximate LQP can then be solved exactly by the linear quadratic regulator (LQR)~\citep{anderson2007optimal}. Subsequently, the newly obtained trajectory is used as the nominal trajectory for the next iteration. This process is repeated until the cost function is converged. The detailed derivation of iLQR can be found in \cref{sec:ilqr_derivation}. Further details of iLQR can be referred to \citep{todorov2005generalized, tassa2012synthesis}. 
        % One of the major contributions of this paper is that we provide the proof of convergence for iLQR adopted from DDP~\citep{yakowitz1984computational} in \cref{sec:ilqr_convergence_proof}, which to the best of our knowledge has not been done before.
        To our best knowledge, there are no previous works that provide a proof on the convergence of iLQR. Therefore, for a complete analysis, we provide the proof of convergence for iLQR adopted from DDP~\citep{yakowitz1984computational} in \cref{sec:ilqr_convergence_proof}.

    \subsection{Meta-learning}
    \label{sec:ML_problem}
        The setting of the meta-learning considered in this paper follows the \emph{task environment}~\citep{baxter2000model}, where tasks are i.i.d. sampled from an unknown distribution \(p(\mathcal{T})\) over a family of tasks. Each task \(\mathcal{T}_{i}\) is associated with two data subsets: training (or support) subset \(\mathcal{S}_{i}^{(s)} = \{(\mathbf{s}_{ij}^{(s)}, y_{ij}^{(s))} \}_{j=1}^{m_{i}^{(s)}}\), where \(\mathbf{s}_{ij}^{(s)}\) denotes a training input and \(y_{ij}^{(s)}\) denotes the corresponding training label, \(i \in \{1, \ldots, M\}\), and validation (or query) subset \(\mathcal{S}_{i}^{(q)}\) which is similarly defined. For training tasks \(\{\mathcal{T}_i\}_{i=1}^{M}\), both data subsets have labels, while for testing tasks \(\mathcal{T}_{M + 1}\), only the data in \(\mathcal{S}_{M + 1}^{(s)}\) is labelled. The aim is to learn a meta-parameter \(\mathbf{x} \in \mathbb{R}^{D}\) shared across all tasks, so that \(\mathbf{x}\) can be efficiently fine-tuned on \(\mathcal{S}_{i}^{(s)}\) to produce a task-specific model that can predict accurately the unlabelled data in \(\mathcal{S}_{i}^{(q)}\). One of the simplest forms of meta-learning is analogous to an extension of hyper-parameter optimisation in single-task learning, where the shared meta-parameter \(\mathbf{x}\) is learnt from many tasks. The objective of meta-learning can be expressed as:
        \begin{equation}
            \min_{\mathbf{x}} \frac{1}{M} \pmb{1}_{M}^{\top} \pmb{\ell}(\mathbf{x}),
            \label{eq:meta_learning_objective}
        \end{equation}
        where \(\pmb{1}_{M}\) is an \(M\)-dimensional vector with all elements equal to 1, and \(\pmb{\ell}(\mathbf{x}) \in \mathbb{R}^{M}\) is a vector containing \(M\) validation losses induced by evaluating the meta-parameter \(\mathbf{x}\) on each data subset \(\mathcal{S}_{i}^{(q)}\) of \(M\) training tasks. Each element of \(\pmb{\ell}(\mathbf{x})\) can be expressed as:
        \begin{equation}
            \pmb{\ell}_{i}(\mathbf{x}) = \frac{1}{m_{i}^{(q)}} \sum_{j=1}^{m_{i}^{(q)}} \ell \left( \mathbf{s}_{ij}^{(q)}, y_{ij}^{(q)}; \phi_{i}(\mathbf{x}) \right), \forall i \in \{1, \ldots, M\},
            \label{eq:validation_loss}
        \end{equation}
        where \(\ell(.)\) is the loss function that is non-negative and twice differentiable, \(\phi(\mathbf{x})\) is the parameter fine-tuned on task \(\mathcal{T}_{i}\):
        \begin{equation}
            \phi_{i}(\mathbf{x}) = \mathbf{x} - \frac{\gamma}{m_{i}^{(s)}} \sum_{k=1}^{m_{i}^{(s)}} \grad_{\mathbf{x}} \left[ \ell \left( \mathbf{s}_{ik}^{(s)}, y_{ik}^{(s)}; \mathbf{x} \right) \right],
            \label{eq:task_adaptation_gd}
        \end{equation}
        and \(\gamma\) is the step size or learning rate for the task adaptation step (also known as inner-loop).

        Note that the gradient-based task adaptation step in \eqref{eq:task_adaptation_gd} is a special case of meta-learning in which \(\mathbf{x}\) is considered as the initialisation of the neural network of interest~\citep{finn2017model}. In metric-based meta-learning~\citep{snell2017prototypical}, the task adaptation step in \eqref{eq:task_adaptation_gd} is slightly different where the class prototypes of training data are embedded into a latent space by the meta-model, and the validation loss is based on the distance between the class prototypes to each data-point in \(\mathcal{S}_{i}^{(q)}\). There are also other extensions of \eqref{eq:meta_learning_objective} using probabilistic approaches~\citep{yoon2018bayesian, ravi2018amortized, nguyen2020uncertainty}. Nevertheless, our approach proposed in \cref{sec:methodology} can be integrated into any of these meta-learning algorithms with a slight modification.

    \subsection{Task-weighting meta-learning}
    \label{sec:task_weighting_meta_learning}
        The minimisation of the average validation loss over \(M\) tasks in \eqref{eq:meta_learning_objective} implicitly implies that those tasks are \remove{evenly distributed} \revise{balanced (similar to class balance in single-task learning)}. This assumption is, however, hardly true in practice, and consequently, makes the trained meta-model perform poorly for testing tasks that are rarely observed \remove{(similar to class imbalance problems in single-task learning)}. To address such issue, a task-weighting factor is introduced to diversify the contribution of each training task, allowing the trained meta-model to generalise better to unseen tasks even if those tasks are rare. The objective of such meta-learning problem can be written as:
        \begin{equation}
        	\begin{split}
        		\mathbf{x}^{*} & = \arg\min_{\mathbf{x}} \mathbf{u}^{\top} \pmb{\ell} (\mathbf{x}) \quad \text{s.t. } \mathbf{u} \in \mathcal{U} \subseteq \mathbb{R}^{M},
        	\end{split}
        	\label{eq:weighted_meta_learning}
        \end{equation}
        where \(\mathbf{u}\) is an \(M\)-dimensional vector that re-weights the influence of \(M\) training tasks, and \(\mathcal{U}\) is the set of feasible weights, i.e., as defined by some weighting criterion.
        
        Note that the task-weighting problem in \eqref{eq:weighted_meta_learning} is carried out at the meta level (often known as \say{outer-loop}). It is, therefore, different from some recent meta-learning methods~\citep{khodak2019adaptive, baik2020meta, flennerhag2020meta, lee2020learning} that design different learning strategies for \(\phi_{i}(\mathbf{x})\) at the task adaptation step (often known as \say{inner-loop}) to estimate the meta-parameters with the same outer-loop objective shown in \eqref{eq:meta_learning_objective}.
        %\gustavo{should we emphasise that "... that design different learning strategies at the \say{inner-loop} that typically estimate the meta-model parameters"?}        \cuong{I do not understand the latter part about estimating the meta-model parameters.}\gustavo{I'm just concerned that the reviewer may not understand inner and outer loops. so I wanted to emphasise what happens in the inner loop.} \cuong{I revised it by emphasising the outer-loop objective function is a uniform weighting.}

        The objective in \eqref{eq:weighted_meta_learning} is more flexible than \eqref{eq:meta_learning_objective}, since it allows one to select different weighting criteria to train the meta-learning model of interest. The most widely-used weighting criterion is \textbf{uniform}: \(\mathbf{u}_{i} = \nicefrac{1}{M}, \forall i \in \{1, \ldots, M\}\), making the objective in \eqref{eq:weighted_meta_learning} resemble the one in \eqref{eq:meta_learning_objective}. Another popular criterion is to select \textbf{difficult} tasks -- tasks that have the largest validation losses -- for training to optimise the performance on the worst-case scenarios~\citep{collins2020task}. However, such difficult tasks may not always be preferred when outliers and noise are present. That leads to another weighting approach which favours the \textbf{most familiar} data-points in single-task learning~\citep{kumar2010self, bengio2009curriculum, wang2017robust} -- often referred as \emph{curriculum learning}. The two latter task-weighting approaches can be considered as the \say{exploration} and \say{exploitation} strategies used in reinforcement learning (RL), respectively. Similar to the exploration and exploitation dilemma in RL, we hypothesise that the optimality for task weighting is formed by a balance between these two approaches. In the following section, we propose a principled approach to automate re-weighting tasks through an optimisation on a sequence of many mini-batches rather than relying on manually-designed criteria as the previous papers.
    \section{Methodology}
\label{sec:methodology}
    \subsection{Task-weighting as a trajectory optimisation}
    \label[subsection]{sec:ml2to}
        In practice, the optimisation in \eqref{eq:weighted_meta_learning} is often carried out using a gradient-based optimiser where the next meta-parameter \(\mathbf{x}^{*}\) is obtained from the current meta-parameter \(\mathbf{x}\) and its corresponding \(\mathbf{u}\) via the function \(f\). Such update can be considered as a state-transition dynamics where the meta-parameter \(\mathbf{x}\) is the state and the weighting vector \(\mathbf{u}\) is the action. Given this observation, we explicitly replace the weighting criterion in \eqref{eq:weighted_meta_learning} by a trajectory optimisation to formulate the task-weighting meta-learning problem as follows:
        \begin{equation}
            \begin{split}
                & \mathbf{x}_{t + 1}^{*} = f(\mathbf{x}_{t}^{*}, \mathbf{u}_{t}^{*}) ~ \forall t \in \{1, \ldots, T\}\\
                & \qquad \text{s.t. } \{ \mathbf{u}_{t}^{*} \}_{t = 1}^{T} = \arg\min_{\{\mathbf{u}\}_{t = 1}^{T}} \sum_{t = 1}^{T} c(\mathbf{x}_{t}, \mathbf{u}_{t})\\
                 & \qquad\hphantom{\text{s.t. } }\qquad \text{s.t. } \mathbf{x}_{t + 1} = f(\mathbf{x}_{t}, \mathbf{u}_{t}) ~ \forall t \in \{1, \ldots, T\}, \text{ and } \mathbf{x}_{1} \text{is given}\\
                & \qquad \hphantom{\text{s.t. } } \mathbf{x}_{1}^{*} = \mathbf{x}_{1},
            \end{split}
            \label{eq:meta_learning_with_trajectory_optimsation}
        \end{equation}
        where \(f(.,.)\) corresponds to the formulation of an optimiser such as stochastic gradient descent (SGD)~\citep{robbins1951stochastic} or Adam~\citep{kingma2014adam}, \(c(.,.)\) is a cost function representing the weighting criteria, \(\mathbf{x}_{1}\) is the initialisation of the meta-learning parameter, and the subscript denotes the time step.

        To solve for an optimal re-weighting vector \(\mathbf{u}\) in the constraint of \eqref{eq:meta_learning_with_trajectory_optimsation}, the cost function needs to be defined. Since our interest is the convergence speed and the generalisation of the learnt meta-model, we define the cost function as an un-discounted sum of uniformly-weighted validation losses of tasks belonging to a sequence of \(T\) mini-batches plus a penalisation on the action \(\mathbf{u}\). For simplicity, the penalty on the action \(\mathbf{u}\) is assumed to follow a Gaussian prior with mean \(\mu_{u}\) and precision \(\beta_{u}\). In particular, the cost function can be expressed as:
        \begin{equation}
            c(\mathbf{x}_{t}, \mathbf{u}_{t}) = \pmb{1}_{M}^{\top} \pmb{\ell}(\mathbf{x}_{t}) + \frac{\beta_{u}}{2} \norm{\mathbf{u}_{t} - \mu_{u} \pmb{1}_{M}}^{2},
            \label{eq:cost_function}
        \end{equation}
        where \(\norm{.}\) denotes the \(\ell_{2}\)-norm.

        Note that the action \(\mathbf{u}_{t}\) is not necessarily normalised to 1. We argue that imposing such constraint might not work well in some cases, for example, a mini-batch containing all familiar tasks, and another one containing all unfamiliar tasks. Our hypothesis is to have small weights for familiar tasks in the former mini-batch, while setting large weights for unfamiliar tasks in the latter mini-batch to diversify the learning. Normalising \(\mathbf{u}_{t}\) to 1 will, however, be undesirable since the contribution of the tasks in both mini-batches would be the same, making the meta-learning model even biased further toward the familiar tasks in the first mini-batch. Hence, we allow the weights to be determined automatically by the optimisation in \eqref{eq:meta_learning_with_trajectory_optimsation} with a Gaussian prior.
        Nevertheless, one can also implement \(\mathbf{u}_{t}\) being normalised to 1 by simply replacing it by \(\mathrm{softmax}(\mathbf{u}_{t})\) in the state-transition dynamics \(f\).

        In general, the constraint in \eqref{eq:meta_learning_with_trajectory_optimsation} cannot be solved exactly, but approximately using iterative methods such as DDP or iLQR. Given the state-transition dynamics \(f\) follows the formulation of a first-order gradient-based optimiser (refer to Eq.~\eqref{eq:dynamics_sgd} in \cref{sec:linearised_dynamics_sgd} and Eq.~\eqref{eq:dynamics_adam} in \cref{sec:linearised_dynamics_adam} for the explicit form of \(f\) using SGD and Adam, respectively), \(f\) consists of the first derivatives of the weighted loss \(\mathbf{u}_{t}^{\top} \pmb{\ell}(\mathbf{x}_{t})\) w.r.t. \(\mathbf{x}\). Hence, applying DDP will result in an intractable solution since DDP requires the second derivatives of \(f\), corresponding to the third derivatives of the weighted loss \(\mathbf{u}_{t}^{\top} \pmb{\ell} (\mathbf{x}_{t})\). In contrast, iLQR needs only the first derivatives of \(f\), which corresponds to the second derivatives of the weighted loss \(\mathbf{u}_{t}^{\top} \pmb{\ell}(\mathbf{x}_{t})\). Although this means that iLQR no longer exhibits the quadratic convergence rate as DDP, in the context of meta-learning, the significant reduction in computation out-weights the speed of convergence for the task weighting vector \(\mathbf{u}\). In this paper, we use iLQR to solve the constraint in \eqref{eq:meta_learning_with_trajectory_optimsation}. The locally-optimal actions obtained is then used to re-weight the tasks in each mini-batch to train the meta-learning model of interest.

        The approximation using Taylor's series on the state-transition dynamics and cost function is shown in \cref{sec:linearisation} and \cref{sec:quadraticisation}, respectively. This approximation leads to the calculation of two Hessian matrices: one for the sum of weighted loss, \(\mathbf{u}_{t}^{\top} \pmb{\ell}(\mathbf{x}_{t})\), in the dynamics, denoted as \(\mathbf{F}_{\mathbf{x}_{t}}\), and the other for the sum of non-weighted loss, \(\pmb{1}^{\top}_{M} \pmb{\ell}(\mathbf{x}_{t})\), in the cost function, denoted as \(\mathbf{C}_{\mathbf{x}_{t}, \mathbf{x}_{t}}\). In addition, while performing recursive backward iLQR, we need to calculate another intermediate matrix of the \emph{cost-to-go} in \eqref{eq:cost_to_go} (please refer to \cref{sec:ilqr_derivation}), denoted as \(\mathbf{V}_{t}\), which has the same size as the Hessian matrix. Naively calculating these Hessian matrices comes at the quadratic complexity \(\order{D^{2}}\) in terms of running time and storage, resulting in an intractable solution for large-scaled models. To address such issue, the two Hessian matrices \(\mathbf{F}_{\mathbf{x}_{t}}\) and \(\mathbf{C}_{\mathbf{x}_{t}, \mathbf{x}_{t}}\) may be approximated by their diagonals which can be efficiently computed using the Hutchinson’s method~\citep{bekas2007estimator}. However, as the size of the model increases, using a few samples from the uniform Rademacher distribution produces noisy estimations of the Hessian diagonals, resulting in a poor approximation~\citep{yao2021adahessian}. Instead of calculating the Hessian diagonals, we use the Gauss-Newton diagonals as replacements. As the Gauss-Newton matrix is known to be a good approximation of the Hessian matrix~\citep{martens2010deep, botev2017practical}, this, therefore, results in a good approximation for the Hessian operator. In addition, Gauss-Newton diagonals can be efficiently calculated using a single backward pass~\citep{dangel2020backpack}. For the matrix \(\mathbf{V}_{t}\), we approximate it by its diagonal matrix. \revise{In fact, matrix \(\mathbf{V}_{t}\) is analogous to the inverse Hessian matrix in Newton's method. Thus, approximating matrix \(\mathbf{V}_{t}\) by its diagonal means performing Newton's method separately for each coordinate, which holds when the diagonal of \(\mathbf{V}_{t}\) is dominant}. We also provide some additional results using full matrix \(\mathbf{V}_{t}\) in \cref{apdx:full_matrix_results}. \revise{In general, we do not observe any significant difference in terms of accuracy evaluated on the validation set between the diagonal approximation and the one with full Gauss-Newton matrix.} Nevertheless, these approximation increases the tractability of our proposed method, allowing to implement the proposed method for very large models, such as deep neural networks.

        The whole procedure of the proposed approach can be described as follows: first, a meta-parameter \(\mathbf{x}_{1}\) is initialised as the initial state, and then, iLQR is employed to solve the constraint in \eqref{eq:meta_learning_with_trajectory_optimsation} to determine a locally-optimal action \(\{ \mathbf{u}_{t}^{*} \}_{t = 1}^{T}\) about an arbitrary-but-feasible trajectory \(\{ (\hat{\mathbf{x}}_{t}, \hat{\mathbf{u}}_{t}) \}_{t=1}^{T}\) with \(\hat{\mathbf{x}}_{1} = \mathbf{x}_{1}\). The obtained weighting vectors \(\{ \mathbf{u}_{t}^{*} \}_{t = 1}^{T}\) are then used to weight tasks in each mini-batch to train the meta-parameter \(\mathbf{x}_{t + 1}^{*}\) in \eqref{eq:meta_learning_with_trajectory_optimsation}. The newly calculated state at the end of the \(T\) time steps, \(\mathbf{x}_{T + 1}^{*}\), is then used as the initial state for the next iteration. This process is repeated until the weighted validation loss \(\mathbf{u}^{\top} \pmb{\ell} (\mathbf{x})\) converges to a local minima. \revise{In the implementation, we observe that this optimisation converges after less than 10 iterations.} The complete algorithm of the proposed task-weighting meta-learning approach is outlined in \cref{algm:task_weighting_ilqr}.

        \begin{algorithm}[t]
            \caption{Task-weighting for meta-learning}
            \label{algm:task_weighting_ilqr}
            \begin{algorithmic}[1]
                \Procedure{train}{}
                    \State define total loss \(J\) in Eq.~\eqref{eq:total_cost}
                    \State define \Call{iLQRbackward}{\hphantom{x}} in \cref{algm:approximate_lqr} (\cref{sec:TO_algorithms})
                    % \Statex
                    \State initialise \(\mathbf{x}_{1}\)
                    \While{\(\mathbf{x}\) is not converged}
                        \State get \(T\) mini-batches, each consists of \(M\) tasks
                        \State generate a random sequence of action \(\{\hat{\mathbf{u}}_{t}\}_{t = 1}^{T}\)
                        \State obtain the corresponding state \(\{ \hat{\mathbf{x}}_{t} \}_{t = 1}^{T}\)
                        % \State Generate random trajectory: \(\{\hat{\mathbf{x}}_{t}, \hat{\mathbf{u}}_{t}\}_{t=1}^{T}\)
                        \While{iLQR cost is not converged}
                            \State \(\{ \mathbf{K}_{t}, \mathbf{k}_{t} \}_{t = 1}^{T}, \theta_{1} \gets\) \Call{iLQRbackward}{$\{ \hat{\mathbf{x}}_{t}, \hat{\mathbf{u}}_{t} \}_{t = 1}^{T}$}
                            \State \(\varepsilon = 2\) %\Comment{Variable for backtracking line search}
                            \Repeat \Comment{Backtracking line search}
                                \State \(\varepsilon \gets \frac{1}{2} \varepsilon\)
                                \For{\(t = 1:T\)} \Comment{iLQR forward pass}
                                    \State \(\mathbf{u}_{t} = \mathbf{K}_{t} \left( \mathbf{x}_{t} - \hat{\mathbf{x}}_{t} \right) + \varepsilon \mathbf{k}_{t} + \hat{\mathbf{u}}_{t}\)
                                    \State \(\mathbf{x}_{t + 1} = f(\mathbf{x}_{t}, \mathbf{u}_{t})\) %\Comment{Non-linear dynamics}
                                \EndFor
                            \Until{\(J(\mathbf{u}_{1:N}) - J(\hat{\mathbf{u}}_{1:N}) \le \frac{1}{2} \varepsilon \theta_{1}\) and \(\mathbf{u}_{ti} \ge 0\)} \label{step:stopping_ilqr}%~\forall t \in \{1, \ldots, T\}, i \in \{1, \ldots, M\}\)}
                            \State \(\{ \hat{\mathbf{x}}_{t} \}_{t = 1}^{T} \gets \{ \mathbf{x}_{t} \}_{t = 1}^{T}\) \Comment{Update nominal state}
                            \State \(\{ \hat{\mathbf{u}}_{t} \}_{t = 1}^{T} \gets \{ \mathbf{u}_{t} \}_{t = 1}^{T}\) %\Comment{Update action}
                        \EndWhile
                        % \Statex
                        \State \(\mathbf{x}_{1} \gets \mathbf{x}_{T}\) \Comment{Update meta-parameter}
                    \EndWhile
                    \State \textbf{return} \(\mathbf{x}_{1}\)
                \EndProcedure
            \end{algorithmic}
        \end{algorithm}

        To simplify the implementation and convergence analysis, we select the nominal actions that coincide with the uniform weighting, meaning that: \(\hat{\mathbf{u}}_{ti} = \nicefrac{1}{M}, \forall t \in \{1, \ldots, T\}, i \in \{1, \ldots, M\}\). In addition, we constrain that all elements of the weighting vector or action \(\mathbf{u}\) are non-negative since each task would either contribute more or less or even not contribute to the learning for \(\mathbf{x}\). This constraint is incorporated into the stopping condition for iLQR shown in step~\ref{step:stopping_ilqr} of \cref{algm:task_weighting_ilqr}. If there is at least one element \(\mathbf{u}_{ti}, t \in \{1, \ldots, T\}, i \in \{1, \ldots, M\}\) being negative, the backtracking line search will iterate one more time with \(\varepsilon\) decaying toward 0, forcing \(\mathbf{u}_{t}\) to stay close to the nominal \(\hat{\mathbf{u}}_{t}\). Thus, in the worst-case, \(\varepsilon\) is reduced to 0, making \(\mathbf{u}_{t}\) coincide with \(\hat{\mathbf{u}}_{t}\), which is the uniform weighting.

        \subsection{Complexity analysis}
        \label{sec:complexity_analysis}
            The downside of TOW is the overhead due to the linearisation and quadraticisation for the state-transition dynamics and cost function, and the calculation to obtain the controller \(\mathbf{K}_{t}\) and \(\mathbf{k}_{t}\) shown in \cref{algm:task_weighting_ilqr}. If \(\order{T_{0}}\) is the time complexity to train a meta-learning method following a uniform weighting strategy, then the time complexity required by TOW will consist of the following:
            \begin{itemize}%[topsep=0pt]
                \item nominal trajectory: \(\order{T_{0}}\)
                \item linearisation and quadraticisation using Gauss-Newton matrices: \(\order{n_{\mathrm{iLQR}} m_{0} \eta D}\)
                \item iLQR backward: \(\order{n_{\mathrm{iLQR}} M D }\)
                \item iLQR forward with back-tracking line search: \(\order{n_{\mathrm{iLQR}} n_{\mathrm{ls}} T_{0}}\),
            \end{itemize}
            where \(n_{\mathrm{iLQR}}\) is the number of iterations in iLQR, \({m_{0} = m_{i}^{(q)}}, i \in \{1, \ldots, M\},\) is the total number of validation samples within a task, \(\eta\) is the number of arithmetic operations in the model of interest, and \(n_{\mathrm{ls}}\) is the number of back-tracking line search. Thus, the final complexity of TOW is: \({\order{(n_{\mathrm{iLQR}} n_{\mathrm{ls}} + 1)T_{0} + n_{\mathrm{iLQR}} (m_{0} \eta + M) D)}}\) comparing to \(\order{T_{0}}\) in the conventional meta-learning.

        \subsection{Convergence analysis}
        \label{sec:convergence_analysis}
            This subsection proves that the training process for MAML using TOW to weight tasks converges to an \(\epsilon_{0}\)-stationary point where \(\epsilon_{0}\) is greater than some positive constant.
            % In other words, we prove an inequality similar to the following:
            % \begin{equation}
            %     \exists \epsilon > 0: \norm{ \grad_{\mathbf{x}} \mathbf{u}_{T_{\mathrm{iter}}} \pmb{\ell} \left( \mathbf{x}_{T_{\mathrm{iter}}} \right) } \le \epsilon.
            % \end{equation}
            % where \(\pmb{\ell}(\mathbf{x})\) is defined in \eqref{eq:validation_loss}.
            Before analysing the convergence of TOW, we state a lemma bounding the norm of the weighting vector (or action) \(\mathbf{u}_{t}\) obtained from iLQR:
            \begin{restatable}[]{lemma}{LemmaBoundedNormU}
                \label{lemma:bound_norm_delta_u}
                If \(\mathbf{u}_{t}\) is a stationary action of a nominal action \(\hat{\mathbf{u}}_{t}\) obtained from iLQR, then:
                \begin{equation*}
                    \exists \delta > 0: \norm{ \mathbf{u}_{t} - \hat{\mathbf{u}}_{t} } \le \delta.
                \end{equation*}
            \end{restatable}
            \begin{proof}
                Please refer to \cref{sec:boundedness_action} for the detailed proof.
            \end{proof}

            To analyse the convergence of a general non-convex function, one typically assumes that the loss function, its first derivative and second derivative are bounded and Lipschitz-continuous as shown in \cref{assumption:loss_boundedness_lipschitz,assumption:loss_smoothness,assumption:hessian_loss_lipschitz}, respectively~\citep{collins2020task, fallah2020convergence}.

            \begin{restatable}[]{assumption}{AssumptionLossBoundednessLipschitz}
                \label{assumption:loss_boundedness_lipschitz}
                The loss function of interest \(\ell\) mentioned in \eqref{eq:validation_loss} is \(B\)-bounded and \(L\)-Lipschitz.
	        \end{restatable}

	        Formally, \cref{assumption:loss_boundedness_lipschitz} means that the loss function \(\ell\) has the following properties:
            \begin{itemize}
                \item Boundedness: \(\exists B > 0: \forall \mathbf{x} \in \mathbb{R}^{D}, \abs{\ell(\mathbf{s}_{ij}, y_{ij}; \mathbf{x})} \le B \),
                \item Lipschitz continuity: \(\exists L > 0: \forall \widetilde{\mathbf{x}}, \overline{\mathbf{x}} \in \mathbb{R}^{D}, \abs{\ell(\mathbf{s}_{ij}, y_{ij}; \widetilde{\mathbf{x}}) - \ell(\mathbf{s}_{ij}, y_{ij}; \overline{\mathbf{x}})} \le L \norm{ \widetilde{\mathbf{x}} - \overline{\mathbf{x}} } \).
            \end{itemize}

            \revise{The boundedness assumption is to bound the second moment of the loss function, while }
            \remove{T}the Lipschitz-continuity assumption of the loss function \(\ell\) \remove{also} implies that the gradient norm of \(\ell\) w.r.t. \(\mathbf{x}\) is bounded above (see \cref{lemma:lipschitz_to_bounded_grad} in \cref{sec:miscellaneous_lemmas}):
            \begin{equation}
                \norm{ \grad_{\mathbf{x}} \ell (\mathbf{s}, y; \mathbf{x}) } \le L.
                \label{eq:bounded_grad}
            \end{equation}

            The bounded gradient norm in \eqref{eq:bounded_grad} also implies that the variance of gradient of the loss function w.r.t. training samples is bounded as shown in \cref{lemma:bounded_variance_grad}.
            \begin{restatable}[]{lemma}{LemmaBoundedVarianceGrad}
                \label{lemma:bounded_variance_grad}
                If \cref{assumption:loss_boundedness_lipschitz} holds, then the variance of gradient of the loss function is \(\sigma^{2}\)-bounded:
                \begin{equation*}
                    \exists \sigma > 0: \forall \mathbf{x} \in \mathbb{R}^{D}, \mathbb{E}_{(\mathbf{s}_{ij}, y) \sim \mathcal{D}_{i}} \left[ \norm{ \grad_{\mathbf{x}} \ell(\mathbf{s}_{ij}, y_{ij}; \mathbf{x}) - \mathbb{E}_{(\mathbf{s}_{ij}, y) \sim \mathcal{D}_{i}} \left[ \grad_{\mathbf{x}} \ell(\mathbf{s}_{ij}, y_{ij}; \mathbf{x}) \right] }^{2} \right] \le \sigma^{2}.
                \end{equation*}
            \end{restatable}
            \begin{proof}
                Please refer to \cref{sec:boundedness_variance_grad} for the detailed proof.
            \end{proof}

            The result in \cref{lemma:bounded_variance_grad} also leads to the boundedness of the variance of the weighted validation loss as shown in \cref{lemma:bounded_variance_weighted_loss}.
            \begin{restatable}[]{lemma}{LemmaBoundedVarianceWeightedLoss}
                \label[lemma]{lemma:bounded_variance_weighted_loss}
                Given the result in \cref{lemma:bounded_variance_grad}, the variance of \(\grad_{\mathbf{x}} \mathbf{u}_{t}^{\top} \pmb{\ell}(\mathbf{x}_{t})\) is  bounded above by \(\widetilde{\sigma}^{2} = \sigma^{2} \left(\delta + M^{-0.5} \right)^{2}\).
	        \end{restatable}
	        \begin{proof}
	            Please refer to \cref{sec:boundedness_variance_loss} for the detailed proof.
	        \end{proof}
	        
            \begin{restatable}[]{assumption}{AssumptionLossSmoothness}
            \label{assumption:loss_smoothness}
                % The loss function \(\ell (\mathbf{s}, y; \mathbf{x})\) is \(S\)-smooth.
                The gradient of the loss function \(\ell (\mathbf{s}, y; \mathbf{x})\) w.r.t. \(\mathbf{x}\) is \(S\)-Lipschitz.
            \end{restatable}
            
            \cref{assumption:loss_smoothness} means that:
            \begin{equation*}
                \exists S > 0: \forall \widetilde{\mathbf{x}}, \overline{\mathbf{x}} \in \mathbb{R}^{D}, \lVert \grad_{\mathbf{x}} \ell(\mathbf{s}_{ij}, y_{ij}; \widetilde{\mathbf{x}}) - \grad_{\mathbf{x}} \ell(\mathbf{s}_{ij}, y_{ij}; \overline{\mathbf{x}}) \lVert \le S \lVert \widetilde{\mathbf{x}} - \overline{\mathbf{x}} \rVert.
            \end{equation*}
    
            % \cref{assumption:loss_smoothness} also leads to the followings:
            % \begin{align}
            %     -S \mathbf{I} \preceq \grad_{\mathbf{x}}^{2} \ell (\mathbf{s}, y; \mathbf{x}) \preceq S \mathbf{I},
            %     \label{eq:hessian_boundedness}
            % \end{align}
            % where \(\mathbf{I}\) is the identity matrix. \cuong{need to explain the notation}
	        
            \begin{restatable}[]{assumption}{AssumptionHessianLossLipschitz}
                \label{assumption:hessian_loss_lipschitz}
                The Hessian matrix \(\grad_{\mathbf{x}}^{2} \ell(\mathbf{s}, y; \mathbf{x}) \) is \(\rho\)-Lipschitz.
            \end{restatable}

            \cref{assumption:hessian_loss_lipschitz} implies that:
            \begin{equation*}
                \exists \rho > 0: \forall \widetilde{\mathbf{x}}, \overline{\mathbf{x}} \in \mathbb{R}^{D}, \norm{ \grad_{\mathbf{x}}^{2} \ell(\mathbf{s}_{ij}, y_{ij}; \widetilde{\mathbf{x}}) - \grad_{\mathbf{x}}^{2} \ell(\mathbf{s}_{ij}, y_{ij}; \overline{\mathbf{x}}) } \le \rho \norm{ \widetilde{\mathbf{x}} - \overline{\mathbf{x}} }.
            \end{equation*}

            These assumptions are used to bound the gradient of the \say{true} validation loss of task \(\mathcal{T}_{i}\), which is defined as follows:

            % \begin{restatable}[True validation loss]{definition}{TrueValidationLoss}
            %     % The true validation loss of task \(\mathcal{T}_{i}\) is defined as
            %     \[\bar{\pmb{\ell}}_{i} (\mathbf{x}) = \mathbb{E}_{\mathcal{D}_{i}^{(q)}} \left[ \ell \left( \mathbf{s}_{ij}^{(q)}, y_{ij}^{(q)}; \phi(\mathbf{x}) \right) \right],\]
            %     where \(\mathbb{E}_{\mathcal{D}_{i}^{(q)}}\) indicates the expectation over all data pairs \(\{\mathbf{s}_{ij}^{(q)}, y_{ij}^{(q)}\}_{j = 1}^{+\infty}\) sampled from the true probability distribution \(\mathcal{D}_{i}^{(q)}\).
            %     \label{definition:true_validation_loss}
            % \end{restatable}
            \begin{equation}
                \bar{\pmb{\ell}}_{i} (\mathbf{x}) = \mathbb{E}_{\mathcal{D}_{i}^{(q)}} \left[ \ell \left( \mathbf{s}_{ij}^{(q)}, y_{ij}^{(q)}; \phi(\mathbf{x}) \right) \right],
                \label{eq:true_validation_loss}
            \end{equation}
                where \(\mathbb{E}_{\mathcal{D}_{i}^{(q)}}\) indicates the expectation over all data pairs \(\{(\mathbf{s}_{ij}^{(q)}, y_{ij}^{(q)})\}_{j = 1}^{+\infty}\) sampled from the true (validation) probability distribution \(\mathcal{D}_{i}^{(q)}\).

            \begin{restatable}[]{lemma}{LemmaSmoothness}
                \label[lemma]{lemma:smoothness}
                If the conditions in \cref{assumption:loss_boundedness_lipschitz,assumption:loss_smoothness,assumption:hessian_loss_lipschitz} hold, then the gradient of the true validation loss \(\bar{\pmb{\ell}}_{i}(\mathbf{x})\) defined in Eq.~\eqref{eq:true_validation_loss} is \(\widetilde{S}\)-Lipschitz, where: \(\widetilde{S} = S(1 + \gamma S)^{2} + \gamma \rho L\).
	        \end{restatable}
	        \begin{proof}
	            Please refer to \cref{sec:smoothness_loss} for the detailed proof.
	        \end{proof}

         %    In addition, we assume that the variance of the loss function \(\ell\) evaluated on different data points is bounded.
	        % \begin{restatable}[]{assumption}{AssumptionBoundedVariance}
	        %     The variance of the gradient \(\grad_{\mathbf{x}} \ell\) is \(\sigma^{2}\)-bounded.
         %        \label{assumption:bounded_variance}
	        % \end{restatable}
	        
	        % \cref{assumption:bounded_variance} implies that:
         %    \begin{equation*}
         %        \exists \sigma > 0: \forall \mathbf{x} \in \mathbb{R}^{D}, \mathbb{E}_{(\mathbf{s}_{ij}, y) \sim \mathcal{D}_{i}} \left[ \norm{ \grad_{\mathbf{x}} \ell(\mathbf{s}_{ij}, y_{ij}; \mathbf{x}) - \mathbb{E}_{(\mathbf{s}_{ij}, y) \sim \mathcal{D}_{i}} \left[ \grad_{\mathbf{x}} \ell(\mathbf{s}_{ij}, y_{ij}; \mathbf{x}) \right] }^{2} \right] \le \sigma^{2}.
         %    \end{equation*}

	        % Such assumption leads to the boundedness of the variance of the weighted validation loss as shown in \cref{lemma:bounded_variance_weighted_loss}.
	        % \begin{restatable}[]{lemma}{LemmaBoundedVarianceWeightedLoss}
         %        \label[lemma]{lemma:bounded_variance_weighted_loss}
         %        If Assumption~\ref{assumption:bounded_variance} holds, then the variance of \(\grad_{\mathbf{x}} \mathbf{u}_{t}^{\top} \pmb{\ell}(\mathbf{x}_{t})\) is  bounded above by \(\widetilde{\sigma}^{2} = \sigma^{2} \left(\delta + M^{-0.5} \right)^{2}\).
	        % \end{restatable}
	        % \begin{proof}
	        %     Please refer to \cref{sec:boundedness_smoothness_loss} for the detailed proof.
	        % \end{proof}

            Given the above assumptions and lemmas, the convergence of TOW can be shown in Theorem~\ref{theorem:convergence_tow}. \revise{We also provide some examples of loss functions and analyse whether they satisfy \cref{assumption:loss_boundedness_lipschitz,assumption:loss_smoothness,assumption:hessian_loss_lipschitz} in \cref{sec:examples_loss}.} %Further details on the proof is referred to \appendixname~\ref{sec:tow_convergence}.

            \begin{restatable}[]{theorem}{TheoremConvergenceTOW}
                \label{theorem:convergence_tow}
                If \cref{assumption:loss_boundedness_lipschitz,assumption:loss_smoothness,assumption:hessian_loss_lipschitz} hold, the learning rate \(\alpha < \nicefrac{2}{\widetilde{S} \left( \delta \sqrt{M} + 1 \right)}\), and \(\mathbf{z}\) is randomly sampled from \(\{\mathbf{x}_{t}\}_{t=1}^{T_{\mathrm{iter}}}\) returned by Algorithm~\ref{algm:task_weighting_ilqr}, then:
                \begin{equation*}
                    \mathbb{E}_{\mathbf{z} \sim \{\mathbf{x}_{t}\}_{t=1}^{T_{\mathrm{iter}}}} \left[ \mathbb{E}_{\mathcal{D}_{1:M}^{(q) \, t}} \left[ \left\Vert \grad_{\mathbf{z}} \mathbf{u}_{t}^{\top} \bar{\pmb{\ell}}_{1:M} \left( \mathbf{z} \right) \right\Vert^{2} \right] \right] \le \epsilon_{0} + \frac{\kappa}{T_{\mathrm{iter}}},
                \end{equation*}
                where:
                \begin{align}
                    \epsilon_{0} & = \frac{4\delta B \sqrt{M} + \alpha^{2} \widetilde{\sigma}^{2} \widetilde{S} \left( \delta \sqrt{M} + 1 \right)}{\alpha \left[2 - \alpha \widetilde{S} \left( \delta \sqrt{M} + 1 \right) \right]} > 0 \label{eq:epsilon0}\\
                    \kappa & = \frac{2 \mathbf{u}_{1}^{\top} \bar{\pmb{\ell}}_{1:M}\left( \mathbf{x}_{1} \right)}{\alpha \left[2 - \alpha \widetilde{S} \left( \delta \sqrt{M} + 1 \right) \right]},
                \end{align}
                with \(T_{\mathrm{iter}}\) as the number of gradient-update for the meta-parameter (or the number of mini-batches of tasks used), and \(\mathbb{E}_{\mathcal{D}_{1:M}^{(q) \, t}}\) as the expectation taken over all data sampled from \(t\) mini-batches \(\{\mathcal{D}_{i}^{(q)}\}_{i=1}^{t}\), each \(\mathcal{D}_{i}^{(q)}\) has \(M\) tasks.
            \end{restatable}
            \begin{proof}
                Please refer to \cref{sec:tow_convergence_proof} for the detailed proof.
            \end{proof}
            
            \cref{theorem:convergence_tow} shows that the expectation of squared gradient norm of the weighted validation loss is upper-bounded by a monotonically reducing function w.r.t. the number of iterations \(T_{\mathrm{iter}}\). This implies that \cref{algm:task_weighting_ilqr} converges in expectation to an \(\epsilon_{0}\)-stationary point.
            %Note that the learning rate \(\alpha\) and \(\delta\) can be selected and tuned by increasing \(\beta_{u}\) to reduce further the bound on the gradient norm.
            
            \begin{remark}
                The result in \cref{theorem:convergence_tow} agrees with some previous works on task-weighting for meta-learning, e.g. \cite[Ineq. (75)]{collins2020task} where the gradient norm is bounded above by some positive constant. The tightness of the bound in \cref{theorem:convergence_tow} mostly depends on how small the value of \(\epsilon_{0}\) is. In fact, we can observe that \(\lim_{\delta \to 0} \epsilon_{0} = 0\). Thus, to ensure that \(\epsilon_{0}\) is small, \(\delta\) needs to be small. We can make \(\delta\) small by imposing a strong prior of \(\mathbf{u}\) by setting a large value of \(\beta_{u}\) in \cref{eq:cost_function}, e.g. \(\beta_{u} = 10\) in our experiments presented in \cref{sec:experiment}. In addition, we integrate the backtracking line search in \cref{algm:task_weighting_ilqr} to force \(\mathbf{u}\) to stay close to the uniform weighting \(\hat{\mathbf{u}}\), making \(\delta\) very small. Another factor contributes to the small value of \(\epsilon_{0}\) is the inverse of the number of tasks in a mini-batch, \(\nicefrac{1}{M}\), as seen in~\eqref{eq:epsilon0}. In practice, \(M\) cannot be too large and often in the range of 5 to 10. And since we can impose constraints to make \(\delta\) tiny, we can guarantee to obtain a small \(\epsilon_{0}\) to make the bound in \cref{theorem:convergence_tow} tight.
            \end{remark}
    \section{Related work}
\label{sec:related_work}

    Our work directly relates to re-weighting tasks in meta-learning. One notable recent work is TR-MAML~\citep{collins2020task} which places higher weights on tasks with larger validation losses to optimise performance for worst-case scenarios. However, when the number of training tasks is very large, e.g. there will be \(\binom{1000}{5} \approx 8.25 \times 10^{12}\) 5-way classification tasks formed from 1000 characters in Omniglot dataset~\citep{lake2015human}, learning weight for each training task is intractable. TR-MAML circumvents such issue by clustering tasks into a small number of clusters based on some ad-hoc intuition and learn the weight for each cluster. This, however, reduces the practicability of TR-MAML. Another work, \(\alpha\)-MAML~\citep{cai2020weighted}, provides an upper-bound on the distance between the weighted risk evaluated on training tasks to the expected risk on testing tasks. The re-weight factors can then be obtained to minimise that upper-bound, reducing the variance between training and testing tasks. In reinforcement learning (RL), MWL-MAML~\citep{xu2021meta} is recently proposed to employ meta-learning to learn the local optimal re-weight factor of each trajectory using a few gradient descent steps. The downside of MWL-MAML is the need of validation trajectories (or validation tasks in meta-learning) that are representative enough to learn those weights. Furthermore, TR-MAML, \(\alpha\)-MAML and MWL-MAML rely on a single mini-batch of tasks to determine the weights without considering the effect of sequence of mini-batches when training a meta-model, potentially rendering sub-optimal solutions. In contrast, our proposed method does not need to cluster tasks nor require additional set of validation tasks. In addition, our proposed method automates the calculation of task-weighting through an optimisation over a sequence of mini-batches, allowing to obtain better local-optimal solutions outside of a single mini-batch of tasks. \revise{There are also other studies about task balancing, such as Learn to Balance (L2B)~\citep{lee2020learning}. However, L2B introduces additional parameters in the task adaptation step (inner-loop), while our method explicitly introduces a weighting vector at the meta-parameter update step (outer-loop).}

    Our work is also similar to task-weighting in multi-task learning~\citep{chen2018gradnorm, sener2018multi, guo2018dynamic, liu2021towards} where the goal is to obtain an optimal re-weighting vector \(\mathbf{u}\) for all tasks. Such modelling can, therefore, work well with a small number of tasks, but potentially fall short when the number of tasks is very large, e.g. in the magnitude of \(10^{12}\) training tasks for 5-way Omniglot classification, due to the poor scalability of the computational and storage complexities of that modelling. In comparison, our proposed approach does not explicitly learn the weighting vector for all training tasks, but determines the weighting vector for tasks in current and some following mini-batches via a trajectory optimisation technique. In a loose sense, the multi-task learning approaches can be considered as an analogy to a \say{batch} learning setting w.r.t. the weighting vector \(\mathbf{u}\), while ours is analogous to an \say{online} learning setting which can scale well to the number of training tasks. \revise{At the time of writing this paper, we were not aware of a concurrent work -- Auto-Lambda~\citep{liu2022autolambda} -- that is designed to use meta-learning to learn how to weight tasks in the multi-task setting. Auto-Lambda is similar to TR-MAML, which is designed for a fixed number of tasks in the multi-task learning, while being intractable when the number of tasks is large. Furthermore, Auto-Lambda is also similar to MWL-MAML since Auto-Lambda employs validation subsets of training tasks to meta-learn the weighting of tasks.}
    % and sample-weighting in single-task setting~\citep{zhao2015stochastic, wang2017robust, ren2018learning, katharopoulos2018not}. 

    This paper is motivated from the observation of large variation in terms of prediction performance made by meta-learning algorithms on various testing tasks~\citep[\figureautorefname~1]{dhillon2019baseline}, implying that the trained meta-model may be biased toward certain training tasks. Such observation may be rooted in task relatedness or task similarity which is a growing research topic in the field of transfer learning. Existing works include task-clustering using k-nearest neighbours~\citep{thrun1996discovering} or using convex optimisation~\citep{jacob2009clustered}, learning task relationship through task covariance matrices~\citep{zhang2012convex}, or theoretical guarantees to learn similarity between tasks~\citep{shui2019a}. Recently, a large-scale empirical study, known as Taskonomy~\citep{zamir2018taskonomy}, investigated the relationship between 26 computer vision tasks. Another promising direction to quantify task similarity is to employ task representation, notably Task2Vec~\citep{achille2019task2vec}, which is based on Fisher Information matrix to embed tasks into a latent space. One commonality among those studies is that learning from certain training tasks may be beneficial to generalise to unseen tasks. This suggests the design of a mechanism to re-weight the contribution of each training task to improve the performance of the meta-model of interest.

    Furthermore, our work is  related to finite-horizon discrete-time trajectory optimisation or open-loop optimal control which has been well studied in the field of control and robotics. The objective is to minimise a cost function that depends on the states and actions in many consecutive time steps given the state-transition dynamics. Exact solution can be obtained for the simplest problem where the cost is quadratic and the dynamics is linear using linear quadratic regulator~\citep{anderson2007optimal}. For a general non-linear problem, approximate solutions can be found via iterative approaches, such as differential dynamic programming~\citep{jacobson1970differential, murray1984differential, yakowitz1984computational} and iterative LQR (iLQR)~\citep{todorov2005generalized, tassa2012synthesis}.
    % Our work is relatively connected to example weighting which has been thoroughly investigated in the literature. The most popular approach is to select the most informative or hard examples which have large losses~\citep{loshchilov2016online} or gradient norms~\citep{zhao2015stochastic,katharopoulos2018not} to prioritise for learning. However, relying solely on informative or hard examples may be sub-optimal when there is the presence of noise or outliers in the data generation process. In self-pace learning~\citep{kumar2010self} and curriculum learning~\citep{bengio2009curriculum}, it is beneficial to start with easier examples. In \citep{wang2017robust}, the weights of examples are incorporated into a Bayesian model as latent variables to produce a robust weighting solution. In addition, one can employ a validation set as in meta-learning to learn those weights~\citep{ren2018learning}. Although this automates the design of example weighting, the trade-off is the need of a representative validation set of examples.
    \section{Experiments}
\label{sec:experiment}

    \subsection{N-way k-shot classification}
    \label{sec:nway_kshot_classification}
        In this section, we empirically compare the performance of the proposed trajectory optimisation task weighting (TOW) approach with three baselines: one with uniform weighting, denoted as \emph{uniform}, one with higher weights on difficult tasks (or tasks with higher losses), denoted as \emph{exploration}, and the other one with higher weights on easier tasks (or tasks with lower losses), denoted as \emph{exploitation}. The experiments are based on \(n\)-way \(k\)-shot classification setting used in few-shot learning with tasks formed from Omniglot~\citep{lake2015human} and mini-ImageNet~\citep{vinyals2016matching} -- the two most widely used datasets to evaluate the performance of meta-learning algorithms.

        \begin{figure*}[t!]
            \centering
            \begin{subfigure}{0.4 \linewidth}
                \centering
                \begin{tikzpicture}
                    \pgfplotstableread[col sep=comma, header=true]{results/maml_omniglot_CNN_val_accuracy.csv} \myTable
    
                    \begin{axis}[
                        height = 0.5\linewidth,
                        width = 0.67\linewidth,
                        xlabel={\textnumero~of iterations (\(\times\)1,000)},
                        xlabel style={font=\small},
                        xticklabel style = {font=\small},
                        ylabel={Accuracy (\%)},
                        ylabel style={font=\small, yshift=-0.5em},
                        yticklabel style = {font=\small},
                        ymin=80,
                        ymax=98,
                        legend entries={uniform, exploration, exploitation, TOW},
                        legend style={draw=none, font=\footnotesize},
                        legend image post style={scale=0.5},
                        legend cell align={left},
                        legend pos=south east,
                        % restrict x to domain=0:300,
                        scale only axis
                    ]
                        % \addplot[mark=none, MidnightBlue, very thick] table[x expr=0.05 * (\coordindex + 1), y={Testing_accuracy_uniform}]{\myTable};
                        % \addplot[mark=none, BurntOrange, very thick] table[x expr=0.05 * (\coordindex + 1), y={Testing_accuracy_exploration}]{\myTable};
                        % \addplot[mark=none, ForestGreen, very thick] table[x expr=0.05 * (\coordindex + 1), y={Testing_accuracy_exploitation}]{\myTable};
                        % \addplot[mark=none, BrickRed, solid, very thick] table[x expr=0.05 * (\coordindex + 1), y={Testing_accuracy_tow}]{\myTable};
                        \addplot[mark=none, MidnightBlue, very thick] table[x expr=0.05 * (\coordindex + 1), y={Testing_accuracy_uniform}]{\myTable};
                        \addplot[mark=none, BurntOrange, very thick] table[x expr=0.05 * (\coordindex + 1), y={Testing_accuracy_exploration}]{\myTable};
                        \addplot[mark=none, ForestGreen, very thick] table[x expr=0.05 * (\coordindex + 1), y={Testing_accuracy_exploitation}]{\myTable};
                        \addplot[mark=none, BrickRed, solid, very thick] table[x expr=0.05 * (\coordindex + 1), y={Testing_accuracy_tow}]{\myTable};
                    \end{axis}
                \end{tikzpicture}
                \caption{MAML on Omniglot}
                \label{fig:maml_omniglot_val_accuracy}
            \end{subfigure}
            % \hspace{3em}
            % \hfill
            \begin{subfigure}{0.4 \linewidth}
                \centering
                \begin{tikzpicture}
                    \pgfplotstableread[col sep=comma, header=true]{results/protonet_omniglot_CNN_val_accuracy.csv} \myTable
    
                    \begin{axis}[
                        height = 0.5\linewidth,
                        width = 0.67\linewidth,
                        xlabel={\textnumero~of iterations (\(\times\)1,000)},
                        xlabel style={font=\small},
                        xticklabel style = {font=\small},
                        ylabel={Accuracy (\%)},
                        ylabel style={font=\small, yshift=-0.5em},
                        yticklabel style = {font=\small},
                        ymin=90,
                        ymax=98.5,
                        scale only axis
                    ]
                        \addplot[mark=none, MidnightBlue, very thick] table[x expr=0.05 * (\coordindex + 1), y={Testing_accuracy_uniform}]{\myTable};
                        \addplot[mark=none, BurntOrange, very thick] table[x expr=0.05 * (\coordindex + 1), y={Testing_accuracy_exploration}]{\myTable};
                        \addplot[mark=none, ForestGreen, very thick] table[x expr=0.05 * (\coordindex + 1), y={Testing_accuracy_exploitation}]{\myTable};
                        \addplot[mark=none, BrickRed, solid, very thick] table[x expr=0.05 * (\coordindex + 1), y={Testing_accuracy_tow}]{\myTable};
                    \end{axis}
                \end{tikzpicture}
                \caption{Protonet on Omniglot}
                \label{fig:protonet_omniglot_val_accuracy}
            \end{subfigure}
    
            \vspace{1.5em}
            \centering
            \hspace{-4em}
            \begin{subfigure}[t]{0.4 \linewidth}
                \centering
                \begin{tikzpicture}
                    \pgfplotstableread[col sep=comma, header=true]{results/maml_miniImageNet_CNN_val_accuracy.csv} \myTable
    
                    \begin{axis}[
                        height = 0.5\linewidth,
                        width = 0.67\linewidth,
                        xlabel={\textnumero~of iterations (\(\times\)1,000)},
                        xlabel style={font=\small},
                        xticklabel style = {font=\small},
                        restrict x to domain=0:43,
                        ylabel={Accuracy (\%)},
                        ylabel style={font=\small, yshift=-0.5em},
                        yticklabel style = {font=\small},
                        ymin=35,
                        ymax=45,
                        % restrict x to domain=0:300,
                        scale only axis
                    ]
                        \addplot[mark=none, MidnightBlue, very thick] table[x expr=0.1 * 2 * (\coordindex + 1), y={Testing_accuracy_uniform}]{\myTable};
                        \addplot[mark=none, BurntOrange, very thick] table[x expr=0.1 * 2 * (\coordindex + 1), y={Testing_accuracy_exploration}]{\myTable};
                        \addplot[mark=none, ForestGreen, very thick] table[x expr=0.1 * 2 * (\coordindex + 1), y={Testing_accuracy_exploitation}]{\myTable};
                        \addplot[mark=none, BrickRed, solid, very thick] table[x expr=0.1 * 2 * (\coordindex + 1), y={Testing_accuracy_tow}]{\myTable};
                    \end{axis}
                \end{tikzpicture}
                \caption{MAML on mini-ImageNet}
                \label{fig:maml_miniimagenet_val_accuracy}
            \end{subfigure}
            \hspace{-5.25em}
            % \hfill
            \begin{subfigure}[t]{0.4 \linewidth}
                \centering
                \begin{tikzpicture}
                    \pgfplotstableread[col sep=comma, header=true]{results/protonet_miniImageNet_CNN_val_accuracy.csv} \myTable
    
                    \begin{axis}[
                        height = 0.5\linewidth,
                        width = 0.67\linewidth,
                        xlabel={\textnumero~of iterations (\(\times\)1,000)},
                        xlabel style={font=\small},
                        xticklabel style = {font=\small},
                        % xmin=2,
                        % restrict x to domain=0:15,
                        % xtick={15, 30, 45},
                        % ylabel={Accuracy (\%)},
                        % ylabel style={font=\small, yshift=-0.5em},
                        yticklabel style = {font=\small},
                        ymin=35,
                        ymax=45,
                        yticklabels={,,},
                        scale only axis
                    ]
                        \addplot[mark=none, MidnightBlue, very thick] table[x expr=0.1 * 2 * (\coordindex + 1), y={Testing_accuracy_uniform}]{\myTable};
                        \addplot[mark=none, BurntOrange, very thick] table[x expr=0.1 * 2 * (\coordindex + 1), y={Testing_accuracy_exploration}]{\myTable};
                        \addplot[mark=none, ForestGreen, very thick] table[x expr=0.1 * 2 * (\coordindex + 1), y={Testing_accuracy_exploitation}]{\myTable};
                        \addplot[mark=none, BrickRed, solid, very thick] table[x expr=0.1 * 2 * (\coordindex + 1), y={Testing_accuracy_tow}]{\myTable};
                    \end{axis}
                \end{tikzpicture}
                \caption{Protonet on mini-ImageNet}
                \label{fig:protonet_miniimagenet_val_accuracy}
            \end{subfigure}
            % \vspace{1.5em}
            \hspace{-6.25em}
            \begin{subfigure}[t]{0.4 \linewidth}
                \centering
                \begin{tikzpicture}
                    \pgfplotstableread[col sep=comma, header=true]{results/maml_miniImageNet_ResNet10_val_accuracy.csv} \myTable
    
                    \begin{axis}[
                        height = 0.5\linewidth,
                        width = 0.67\linewidth,
                        xlabel={\textnumero~of iterations (\(\times\)1,000)},
                        xlabel style={font=\small},
                        xticklabel style = {font=\small},
                        % xmin=2,
                        % restrict x to domain=0:15,
                        % xtick={15, 30, 45},
                        % ylabel={Accuracy (\%)},
                        % ylabel style={font=\small, yshift=-0.5em},
                        yticklabel style = {font=\small},
                        ymin=35,
                        ymax=45,
                        yticklabels={,,},
                        scale only axis
                    ]
                        \addplot[mark=none, MidnightBlue, very thick] table[x expr=0.1 * 2 * (\coordindex + 1), y={Testing_accuracy_uniform}]{\myTable};
                        \addplot[mark=none, BurntOrange, very thick] table[x expr=0.1 * 2 * (\coordindex + 1), y={Testing_accuracy_exploration}]{\myTable};
                        \addplot[mark=none, ForestGreen, very thick] table[x expr=0.1 * 2 * (\coordindex + 1), y={Testing_accuracy_exploitation}]{\myTable};
                        \addplot[mark=none, BrickRed, solid, very thick] table[x expr=0.1 * 2 * (\coordindex + 1), y={Testing_accuracy_tow}]{\myTable};
                    \end{axis}
                \end{tikzpicture}
                \caption{MAML with Resnet-10}
                \label{fig:maml_resnet10_miniImageNet_val_accuracy}
            \end{subfigure}
            \caption{Validation accuracy exponential moving average (with smoothing factor 0.1) of different task-weighting strategies evaluated on: \protect\subref{fig:maml_omniglot_val_accuracy} and \subref{fig:protonet_omniglot_val_accuracy}
            Omniglot, and \subref{fig:maml_miniimagenet_val_accuracy}, \subref{fig:protonet_miniimagenet_val_accuracy} and \subref{fig:maml_resnet10_miniImageNet_val_accuracy} mini-ImageNet.
            % The column plots show testing accuracy on: \subref{fig:omniglot_testing_accuracy} Omniglot and \subref{fig:miniimagenet_testing_accuracy} mini-ImageNet.
            }
            \label{fig:val_test_accuracy}
        \end{figure*}
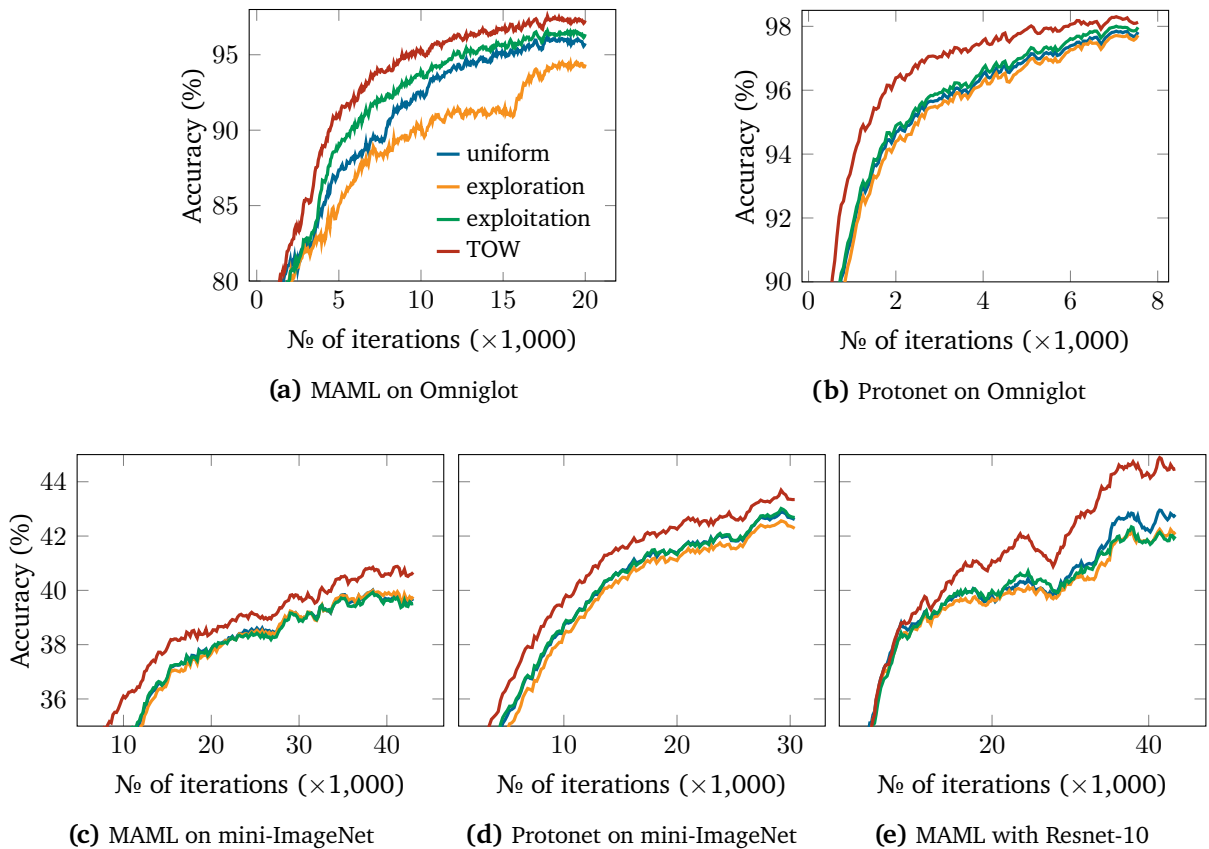
    
        Naively implementing the two baselines, \emph{exploration} and \emph{exploitation}, will easily lead to trivial solutions where only the task with largest or smallest loss within a mini-batch is selected. Thus, only one task in each mini-batch is used for learning, and consequently, making the learning noisy and unstable. We, therefore, introduce a prior, denoted as \(p(\mathbf{u})\), as a regularisation to prevent many tasks within the same mini-batch from being discarded. The objective to determine the weights for these two baselines can be written as follows:
        \begin{equation}
            \mathbf{u}^{*} = \begin{dcases}
            \arg\min_{\mathbf{u}} -\mathbf{u}^{\top} \pmb{\ell}(\mathbf{x}) - \ln p(\mathbf{u}) & \text{for \emph{exploration}}\\
            \arg\min_{\mathbf{u}} \mathbf{u}^{\top} \pmb{\ell}(\mathbf{x}) - \ln p(\mathbf{u}) & \text{for \emph{exploitation}}.
            \end{dcases}
            \label{eq:baseline_exploration}
        \end{equation}
        
        In general, the prior \(p(\mathbf{u})\) can be any distribution that has support in \((0, +\infty)\) such as Beta, Gamma or Cauchy distribution. For simplicity, \(p(\mathbf{u})\) is selected as a Dirichlet distribution with a concentration \(\kappa > 1\) to constrain the weight vector within a probability simplex. One can then use a non-linear optimisation solver to solve \eqref{eq:baseline_exploration} to obtain an optimal \(\mathbf{u}^{*}\) for one of the two baselines. In the implementation, we use Sequential Least SQuares Programming (SLSQP) to obtain \(\mathbf{u}^{*}\). Note that the definition of the \emph{exploration} baseline above resembles TR-MAML~\citep{collins2020task}, but is applicable for common few-shot learning benchmarks where the number of tasks is large. Similarly, the \emph{exploitation} is an analogy to robust Bayesian data re-weighting~\citep{wang2017robust} or \emph{curriculum learning} in single-task learning.
    
        For Omniglot dataset, there are a total of 1,623 different handwritten characters from 50 different alphabets. Each of characters was drawn online via Amazon's Mechanical Turk by 20 different people~\citep{lake2015human}. Conventionally, 1,000 randomly-sampled characters are used for training and the remaining is used to testing. Such train-test split might, however, be inadequate since the characters in one alphabet can be present in both training and testing sets, and learning a character in that alphabet might help to classify easily from that same alphabet. To make the classification more challenging, we follow the original train-test split~\citep{lake2015human} by using 30 alphabets for training and 20 alphabets for testing. We also utilise the hierarchy structure of alphabet-character to form finer-grained classification tasks to make the classification more difficult than the random train-test split. Furthermore, we follow the convention from previous work to re-size all those gray images to 28-by-28 pixel\textsuperscript{2} to have a fair evaluation.

        \begin{table}[t]
            \centering
            \caption{The classification accuracy averaged on 1,000 random testing tasks generated from Omniglot and 600 tasks from mini-ImageNet with 95 percent confident interval; the bold numbers denote their statistically differences from the ones in the same column.}
            \label{tab:5way_1shot}
            \begin{tabular}{l l l l l}
                \toprule
                \multirow{2}{*}{\bfseries } & \multirow{2}{5em}{\bfseries Weighting method} & \multirow{2}{*}{\bfseries Omniglot} & \multicolumn{2}{c}{\bfseries Mini-ImageNet} \\
                \cmidrule{4-5}
                & & & \bfseries 4 layer CNN & \bfseries Resnet-10 \\
                \midrule
                \multirow{5}{*}{MAML} & Uniform & 94.86 \textpm 0.43 & 48.70 \textpm 1.84\tablefootnote{reported in \citep{finn2017model}} & 49.12 \textpm 0.76 \\
                & Exploration & 92.64 \textpm 0.52 & 48.80 \textpm 0.72 & 48.72 \textpm 0.74 \\
                & Exploitation & 95.34 \textpm 0.42 & 49.22 \textpm 0.74 & 48.44 \textpm 0.76 \\
                \rowcolor{Gray!25} \cellcolor{White} & TOW & 95.94 \textpm 0.40 & 51.55 \textpm 0.75 & \textbf{52.32 \textpm 0.80} \\
                \midrule
                \multirow{5}{*}{Protonet} & Uniform & 95.21 \textpm 0.37 & 49.42 \textpm 0.78\tablefootnote{reported in \citep{snell2017prototypical}} & \_ \\
                & Exploration & 94.57 \textpm 0.38 & 48.56 \textpm 0.77 & \_ \\
                & Exploitation & 95.78 \textpm 0.40 & 48.39 \textpm 0.79 & \_ \\
                \rowcolor{Gray!25} \cellcolor{White} & TOW & \textbf{96.84 \textpm 0.37} & \textbf{51.05 \textpm 0.80} & \_ \\
                \bottomrule
            \end{tabular}
        \end{table}
    
        For mini-ImageNet, the dataset consists of 100 classes, each class has 600 colour images sampled from 1,000 classes of the ImageNet dataset~\citep{deng2009imagenet}. We follow the standard train-test split that uses 64 classes for training, 16 classes for validation and 20 for testing~\citep{ravi2017optimization} in our evaluation. To be consistent with previous work, we pre-process all images by resizing them to 84-by-84 pixels\textsuperscript{2} before carrying out any training or testing.
        
        The base model used across the experiments is the 4 CNN module network that is widely used in few-shot image classification~\citep{vinyals2016matching, finn2017model}. Each module of the base network consists of 32 filters with 3-by-3 kernel, followed by a batch normalisation, activated by the Rectified Linear Unit (ReLU) and pooled by a 2-by-2 max-pooling layer. The output of the last CNN module is flattened before performing classification. Two common meta-learning algorithms considered in this section include MAML~\citep{finn2017model} and Prototypical Networks~\citep{snell2017prototypical}. In MAML, the flattened features are passed to a linear fully-connected layer to classify, while in Prototypical Networks, the classification is based on Euclidean distances to the prototypes of each class.
    
        For all experiments, the learning rate \(\gamma\) of task adaptation (also known as inner-loop) shown in Eq.~\eqref{eq:task_adaptation_gd} is 0.1 for Omniglot and 0.01 for mini-ImageNet with 5 gradient updates. The learning rate for the meta-parameters, \(\alpha\), is set at \(10^{-4}\) for all the setting. The mini-batch size is \(M = 10\) tasks for Omniglot and \(M = 5\) tasks for mini-ImageNet. For the Dirichlet concentration of the prior in the \emph{exploration} and \emph{exploitation} baselines, we try three values of \(\kappa \in \{0.2, 1.2, 5\}\), and found that a too small value of \(\kappa\) leads to noisy learning since only the easiest or hardest task is selected, while too large value of \(\kappa\) makes both the baselines identical to uniform weighting. Hence, we select \(\kappa = 1.2\) that balances between these two strategies. Note that \(\kappa = 1\) results in a random prior, leading to a trivial solution. For the trajectory optimiser iLQR, the state-transition dynamics \(f\) follows the formula of Adam optimiser since Adam provides a less noisy training as in SGD. The nominal trajectory is, as mentioned in~\cref{sec:ml2to}, selected with uniform actions: \(\hat{\mathbf{u}}_{tj} = \nicefrac{1}{M}, \forall j \in \{1, \ldots, M\}, t \in \{1, \ldots, T\}\). The number of iterations used in iLQR is 2 \revise{to speed up the training, although higher number of iterations can be used to achieve better performance by trading off the running time. We also provide an ablation study with two different numbers of iterations in iLQR in \cref{sec:ablation_study}, where the larger number of iterations in iLQR slightly improves the prediction accuracy on the validation set of mini-ImageNet}. The number of time steps (or number of mini-batches) is \(T = 10\) for Omniglot and 5 for mini-ImageNet. The parameters of the prior on the action \(\mathbf{u}_{t}\) are \(\mu_{u} = \nicefrac{1}{M}\) and \(\beta_{u} = 10\). As we do not observe any major difference between different configuration of \(M\) and \(T\) used in this experiment, we report the result for the case \(M = 10\) and \(T = 5\). Each experiment is carried out on a single NVIDIA Tesla V100 GPU with 32 GB memory following the configuration of NVIDIA DGX-1.

        \begin{table}[t]
            \centering
            \caption{Running time (in GPU-hour) of different task-weighting methods based on MAML.}
            \label{tab:running_time}
            \begin{tabular}{l c c c}
                \toprule
                & \multirow{2}{*}{\bfseries Omniglot} & \multicolumn{2}{c}{\bfseries mini-ImageNet}\\
                \cmidrule{3-4}
                & & \bfseries CNN & \bfseries Resnet-10\\
                \midrule
                Exploration & 1.55 & 5.24 & 7.18\\
                Exploitation & 1.55 & 5.24 & 7.18\\
                Uniform & 1.35 & 5.03 & 7.16\\
                \rowcolor{Gray!25} TOW & 7.50 & 38.12 & 67.78\\
                \bottomrule
            \end{tabular}
        \end{table}
    
        \cref{fig:maml_omniglot_val_accuracy,fig:protonet_omniglot_val_accuracy,fig:maml_miniimagenet_val_accuracy,fig:protonet_miniimagenet_val_accuracy} plot the testing accuracy evaluated on 100 validation tasks drawn from Omniglot and mini-ImageNet following the 5-way 1-shot setting. The validation accuracy curves along the training process show that TOW can achieve higher performance comparing to the three baselines on various datasets and meta-learning methods. We also carry out an experiment using Resnet-10~\citep{he2016deep} on mini-ImageNet to demonstrate the scalability of TOW. The results on Resnet-10 in \cref{fig:maml_resnet10_miniImageNet_val_accuracy} shows a a similar observation that TOW out-performs other task-weighting methods. We note that the validation accuracy curves of Resnet-10 fluctuates due to our injected dropout to regularise the network from overfitting since it is known that larger networks, such as Resnet-10 or Resnet-18, severely overfit in the few-shot setting~\citep{nguyen2020uncertainty}. For the evaluation on testing sets, we follow the standard setting in few-shot learning by measuring the prediction accuracy on 1,000 and 600 testing tasks formed from Omniglot and mini-ImageNet, respectively~\citep{vinyals2016matching, finn2017model}. The results in \cref{tab:5way_1shot} show that TOW can be at least 2 percent more accurate than the best baseline among Uniform, Exploration, and Exploitation. Note that there a difference between the results shown in \cref{fig:val_test_accuracy} and \cref{tab:5way_1shot} due to their differences in terms of (i) the tasks form: one from validation set, while the other from testing set, and (ii) the number of tasks evaluated. Despite the promising results, the downside of TOW is the overhead caused by approximating the cost and state-transition dynamics over \(T\) mini-batches of tasks to determine the locally-optimal \(\{\mathbf{u}_{t}^{*}\}_{t = 1}^{T}\). As shown in \cref{tab:running_time}, TOW is about 7 to 9 times slower than the three baselines. We also provide a visualisation of the weights \(\mathbf{u}_{t}\) in \cref{apdx:weight_visualisation}.

    \subsection{Any-shot classification}
    \label{sec:any_shot_classification}
        We also follow the \emph{realistic task distribution}~\citep{lee2020learning} to evaluate further the performance of TOW. The new setting is mostly similar to \(N\)-way \(k\)-shot, except \(k\) is not fixed and might be different for each class within a task. Specifically, with a probability of 0.5, the number of shots for each class is sampled from a uniform distribution: \(k \sim \mathrm{Uniform}(1, 50)\) to simulate class imbalance. With the other 0.5 probability, the same number of shots \(k \sim \mathrm{Uniform}(1, 50)\) is used for all classes within that task. The number of validation (or query) samples is kept at 15 samples per class.
        
        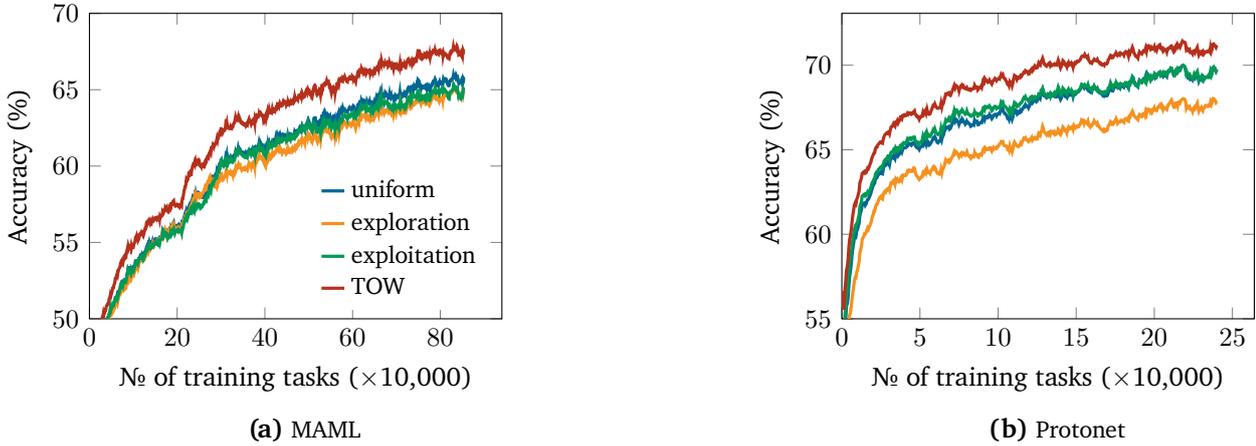
\begin{figure}[t]
            \begin{subfigure}[t]{0.45 \linewidth}
                \begin{tikzpicture}
                    \pgfplotstableread[col sep=comma, header=true]{results/maml_any_shot_val_accuracy.csv} \myTable
    
                    \begin{axis}[
                        height = 0.5\linewidth,
                        width = 0.67\linewidth,
                        xlabel={\textnumero~of training tasks (\(\times\)10,000)},
                        xlabel style={font=\small, yshift=0em},
                        xticklabel style = {font=\small},
                        xmin=0,
                        ylabel={Accuracy (\%)},
                        ylabel style={font=\small, yshift=0em},
                        yticklabel style = {font=\small},
                        ymin=50,
                        ymax=70,
                        % restrict x to domain=0:300,
                        legend entries={uniform, exploration, exploitation, TOW},
                        legend style={draw=none, font=\footnotesize},
                        legend image post style={scale=0.5},
                        legend cell align={left},
                        legend pos=south east,
                        scale only axis
                    ]
                        \addplot[mark=none, MidnightBlue, very thick] table[x expr=0.1 * (\coordindex + 1), y={Testing_accuracy_uniform}]{\myTable};
                        \addplot[mark=none, BurntOrange, very thick] table[x expr=0.1 * (\coordindex + 1), y={Testing_accuracy_exploration}]{\myTable};
                        \addplot[mark=none, ForestGreen, very thick] table[x expr=0.1 * (\coordindex + 1), y={Testing_accuracy_exploitation}]{\myTable};
                        \addplot[mark=none, BrickRed, solid, very thick] table[x expr=0.1 * (\coordindex + 1), y={Testing_accuracy_tow}]{\myTable};
                    \end{axis}
                \end{tikzpicture}
                \caption{MAML}
            \end{subfigure}
            \hfill
            \begin{subfigure}[t]{0.45 \linewidth}
                \begin{tikzpicture}
                    \pgfplotstableread[col sep=comma, header=true]{results/protonet_any_shot_val_accuracy.csv} \myTable
    
                    \begin{axis}[
                        height = 0.5\linewidth,
                        width = 0.67\linewidth,
                        xlabel={\textnumero~of training tasks (\(\times\)10,000)},
                        xlabel style={font=\small, yshift=0em},
                        xticklabel style = {font=\small},
                        xmin=0,
                        ylabel={Accuracy (\%)},
                        ylabel style={font=\small, yshift=0em},
                        yticklabel style = {font=\small},
                        ymin=55,
                        restrict x to domain=0:24,
                        scale only axis
                    ]
                        \addplot[mark=none, MidnightBlue, very thick] table[x expr=0.05 * (\coordindex + 1), y={Testing_accuracy_uniform}]{\myTable};
                        \addplot[mark=none, BurntOrange, very thick] table[x expr=0.05 * (\coordindex + 1), y={Testing_accuracy_exploration}]{\myTable};
                        \addplot[mark=none, ForestGreen, very thick] table[x expr=0.05 * (\coordindex + 1), y={Testing_accuracy_exploitation}]{\myTable};
                        \addplot[mark=none, BrickRed, solid, very thick] table[x expr=0.05 * (\coordindex + 1), y={Testing_accuracy_tow}]{\myTable};
                    \end{axis}
                \end{tikzpicture}
                \caption{Protonet}
            \end{subfigure}
            \caption{Exponential moving average with smoothing factor 0.1 of the prediction accuracy evaluated on validation tasks formed from the any-shot setting of mini-ImageNet dataset mentioned in \cref{sec:any_shot_classification} where the base model is a 4-module CNN.}
            \label{fig:any_shot_training}
        \end{figure}

        Similar to the experiments carried out in \cref{sec:nway_kshot_classification}, TOW demonstrates a higher performance compared to the three baselines: exploitation, exploration and uniform along the training process, as shown in \cref{fig:any_shot_training}. In general, TOW can achieve the state-of-the-art results when evaluating on 3,000 testing tasks formed from mini-ImageNet, as shown in \cref{tab:any_shot_results}, compared to common meta-learning methods.
        %, except for the Bayesian TAML~\citep{lee2020learning}, which is designed for this setting. 
        To further evaluate TOW, we follow the same setting and use the models trained on mini-ImageNet to test on 50 classes of bird images split from CUB dataset. The results in the last column of \cref{tab:any_shot_results} show that TOW can also work well on out-of-distribution tasks formed from CUB compared to most of the methods in the literature. The main reason that explains the worse performance of TOW, compared with Bayesian TAML, is that the meta-learning based methods used by TOW are MAML and Protonet, which have a smaller number of meta-parameters to model tasks than Bayesian TAML.
        % In theory, TOW can be placed on top of Bayesian TAML to obtain an even higher performance \gustavo{The reviewer can ask "why not then running this experiment with Bayesian TAML?" If you're prepared to run that experiment, we can leave this comment, so reviewers can easily pick that up. If you're not ready, then I recommend remove this comment.}.

        \begin{table}[t]
            \centering
            \caption{Prediction results on any-shot classification evaluated on 3,000 testing tasks with 95 percent confident interval; the bold numbers denote the results that are statistically significant. The results of previous methods are reported in \citep{lee2020learning}.}
            \label{tab:any_shot_results}
            \begin{tabular}{l l l}
                \toprule
                \bfseries Training set & \multicolumn{2}{c}{\bfseries mini-ImageNet} \\
                \cmidrule{2-3}
                \bfseries Testing set & \bfseries mini-ImageNet & \bfseries CUB \\
                \midrule
                MAML~\citep{finn2017model} & 66.64 \textpm 0.22 & 65.77 \textpm 0.24 \\
                Meta-SGD~\citep{li2017meta} & 69.95 \textpm 0.20 & 65.94 \textpm 0.22 \\
                MT-net~\citep{lee2018gradient} & 67.63 \textpm 0.23 & 66.09 \textpm 0.23 \\
                ABML~\citep{ravi2018amortized} & 56.91 \textpm 0.19 & 57.88 \textpm 0.20 \\
                Protonet~\citep{snell2017prototypical} & 69.11 \textpm 0.19 & 60.80 \textpm 0.19 \\
                Proto-MAML~\citep{triantafillou2020meta} & 68.96 \textpm 0.18 & 61.77 \textpm 0.19 \\
                Bayesian TAML~\citep{lee2020learning} & 71.46 \textpm 0.19 & \textbf{71.71 \textpm 0.21} \\
                \rowcolor{Gray!25} TOW-MAML & 70.02 \textpm 0.24 & 68.34 \textpm{0.25} \\
                \rowcolor{Gray!25} TOW-Protonet & \textbf{72.12 \textpm 0.21} & 64.79 \textpm 0.25 \\
                \bottomrule
            \end{tabular}
        \end{table}
    \section{Discussion and conclusion}
\label{sec:discussion}
    We propose a principled approach based on trajectory optimisation to mitigate the issue of non-uniform distribution of training tasks in meta-learning. The idea is to model the training process in meta-learning by trajectory optimisation with state as meta-parameter and action as the weights of training tasks. The local optimal weights obtained from iLQR -- a trajectory optimiser are then used to re-weight tasks to train the meta-parameter of interest. We demonstrate that the proposed approach converges with less number of training tasks and has a final prediction accuracy that out-performs some common hand-crafted task-weighting baselines.

    Our proposed method also has some limitations that could be addressed in future work. TOW relies on iLQR which is not ideal for large-scale systems with high dimensional state space such as deep neural networks. Despite the approximation of Hessian matrices to use only diagonals as mentioned in \cref{sec:ml2to}, the linearisation of the state-transition dynamics and quadraticisation of the cost function are still time-consuming, and consequently, reduce TOW's efficiency. Future work might find a faster approximation to optimise the running time for TOW as well as evaluate on large-scaled datasets, such as meta-dataset~\citep{triantafillou2020meta}.
    % the second derivative matrix of the value function, denoted as \(\mathbf{V}_{t}\) in \appendixname~\ref{sec:ilqr_derivation}, is still a dense matrix that has the same size as the Hessian matrices. Consequently, a large amount of memory is required to store such large matrix, making the proposed method applicable only for small models, such as the 4 CNN module network used in \sectionautorefname~\ref{sec:experiment}. This scalability issue has been an activate research in trajectory optimisation. Existing solutions rely on projection to reduce the dimensionality~\citep{huang2020balanced}. The tradeoff is to solve large-scale Lyapunov or algebraic Ricatti equation~\citep{benner2008numerical}, which is infeasible in the context of neural networks.

    Furthermore, our method is local in nature due to the Taylor's series approximation about a nominal trajectory used in iLQR. One way to improve further is to define a \say{global} or \say{stationary} policy \(\pi_{\theta}(\mathbf{x}_{t}, \mathbf{u}_{t})\), which is similar to Guided Policy Search~\citep{levine2013guided, levine2014learning}. This policy can then be trained on multiple local optimal trajectories obtained from iLQR. While this approach may offer a superior generalisation for the policy, scalability is an issue since the policy needs to process the high-dimensional state \(\mathbf{x}_{t}\). As a result, a very large model may be required to implement such policy.
    
    % % Acknowledgements should go at the end, before appendices and references
    
    % % \acks{Available upon acceptance of the paper.}
    
    % % Manual newpage inserted to improve layout of sample file - not
    % % needed in general before appendices/bibliography.

    \vskip 0.2in
%    \bibliography{references}
	\printbibliography

    \newpage
    \appendix

\onecolumn
\section[Convergence analysis for TOW]{Convergence analysis for trajectory optimisation based task weighting meta-learning}
\label{sec:tow_convergence}
    \subsection{Notations}
        The following notations are used throughout the paper:
        \begin{itemize}
            \item \(\norm{\mathbf{x}}\) is the L2 norm of a vector \(\mathbf{x} \in \mathbb{R}^{D}\), e.g. \(\sqrt{\mathbf{x}^{\top} \mathbf{x}}\)
            \item \(\norm{\mathbf{A}}\) is the matrix norm of a matrix \(\mathbf{A} \in \mathbb{R}^{M \times D}\) induced by a vector norm:
            \begin{equation*}
                \norm{\mathbf{A}} = \sup \frac{\norm{\mathbf{A} \mathbf{x}}}{\norm{\mathbf{x}}}, \forall \mathbf{x} \in \mathbf{R}^{D}: \norm{\mathbf{x}} \neq 0.
            \end{equation*}
        \end{itemize}

        Given the vector and matrix norm, two common inequalities used in this section are:
        \begin{itemize}
            \item Triangle inequality: \(\norm{ \mathbf{A} + \mathbf{B} } \le \norm{\mathbf{A}} + \norm{\mathbf{B}}, \forall \mathbf{A}, \mathbf{B} \in \mathbb{R}^{M \times D}\)
            \item Sub-multiplication: \(\norm{\mathbf{A} \mathbf{x}} \le \norm{\mathbf{A}} \norm{\mathbf{x}}\).
        \end{itemize}

    \subsection{Auxiliary lemmas}
    \label{sec:auxiliary_lemmas}
        \subsubsection{Boundedness of action (or weighting vector) \texorpdfstring{\(\mathbf{u}\)}{u}}
        \label{sec:boundedness_action}
            \LemmaBoundedNormU*
            \begin{proof}
                According to the procedure of iLQR shown in \eqref{eq:action_backtracking_linesearch}:
                \begin{equation}
                    \mathbf{u}_{t} - \hat{\mathbf{u}}_{t} = \mathbf{K}_{t} \left( \mathbf{x}_{t} - \hat{\mathbf{x}}_{t} \right) + \varepsilon \mathbf{k}_{t}.
                \end{equation}
                Since the matrix \(\mathbf{K}\), vectors \(\mathbf{k}\), \(\mathbf{x}\) and \(\hat{\mathbf{x}}\) are well-defined, the norm of \(\mathbf{u} - \hat{\mathbf{u}}\) is also well-defined.
    
                In addition, \(\delta\) is implicitly related to the Gaussian prior \(\mathcal{N}(\mathbf{u}_{t}; \mu_{u} \pmb{1}, \nicefrac{1}{\beta_{u}} \mathbf{I}_{M})\). A larger value of \(\beta_{u}\) in (\ref{eq:cost_function}) would result in a smaller value for \(\delta\).
            \end{proof}
    
            \begin{corollary}
                \label[corollary]{corollary:bound_norm_u}
                If \(\mathbf{u}_{t_{1}}\) and \(\mathbf{u}_{t_{2}}\) are stationary actions of two nominal actions \(\hat{\mathbf{u}}_{t_{1}} = \hat{\mathbf{u}}_{t_{2}} = \hat{\mathbf{u}}_{\mathrm{uniform}} = \nicefrac{1}{M}~\pmb{1}\) obtained from iLQR at time steps \(t_{1}\) and \(t_{2}\), respectively, then the followings hold:
                \begin{subequations}
                    \begin{empheq}{align}
                        & \left\Vert \mathbf{u}_{t_{1}} - \mathbf{u}_{t_{2}} \right\Vert \le 2\delta & \text{(bounded L2 norm of difference between 2 actions)}\\
                        & \lVert \mathbf{u}_{t_{1}} \rVert \le \delta + \frac{1}{\sqrt{M}} & \text{(bounded L2 norm)}\\
                        & \pmb{1}^{\top} \mathbf{u}_{t_{1}} \le \delta \sqrt{M} + 1. & \text{(bounded L1 norm)}
                    \end{empheq}
                \end{subequations}
            \end{corollary}

            \begin{proof}
                The first inequality can be proved by simply applying triangle inequality on the \(\ell_{2}\) norm and employing the result in \cref{lemma:bound_norm_delta_u}:
                \begin{equation}
                \begin{aligned}[b]
                    \norm{ \mathbf{u}_{t_{1}} - \mathbf{u}_{t_{2}} } & = \norm{ (\mathbf{u}_{t_{1}} - \hat{\mathbf{u}}_{\mathrm{uniform}}) + (\hat{\mathbf{u}}_{\mathrm{uniform}} - \mathbf{u}_{t_{2}}) } \\
                    & \le \norm{ \mathbf{u}_{t_{1}} - \hat{\mathbf{u}}_{\mathrm{uniform}} } + \norm{ \mathbf{u}_{t_{2}} - \hat{\mathbf{u}}_{\mathrm{uniform}} } \\
                    & \le \norm{ \mathbf{u}_{t_{1}} - \hat{\mathbf{u}}_{t_{1}} } + \norm{ \mathbf{u}_{t_{2}} - \hat{\mathbf{u}}_{t_{2}} } \\
                    & \le 2\delta.
                \end{aligned}
                \end{equation}
    
                The second inequality can similarly be proved using triangle inequality:
                \begin{equation}
                    \begin{aligned}[b]
                        \norm{ \mathbf{u}_{t_{1}} } = \norm{ (\mathbf{u}_{t_{1}} - \hat{\mathbf{u}}_{\mathrm{uniform}}) + \hat{\mathbf{u}}_{\mathrm{uniform}} } & \le \norm{ \mathbf{u}_{t_{1}} - \hat{\mathbf{u}}_{\mathrm{uniform}} } + \norm{\hat{\mathbf{u}}_{\mathrm{uniform}}} \le \delta + 1.
                    \end{aligned}
                \end{equation}
                
                The last inequality can be proved using Cauchy-Schwarz inequality:
                \begin{equation}
                    \begin{aligned}[b]
                        \pmb{1}^{\top} \mathbf{u}_{t_{1}} \le \abs{ \pmb{1}^{\top} \mathbf{u}_{t_{1}} } & \le \sqrt{M} \norm{ \mathbf{u}_{t_{1}} } \quad \text{(Cauchy-Schwarz inequality)}\\
                        & \le \delta \sqrt{M} + 1.
                    \end{aligned}
                \end{equation}
            \end{proof}

        \subsubsection{Boundedness of variance of gradient of the loss function \texorpdfstring{\(\ell\)}{l}}
        \label{sec:boundedness_variance_grad}
            \LemmaBoundedVarianceGrad*
            \begin{proof}
                The term inside the expectation in the left-hand side can be written as:
                \begin{equation}
                    \begin{aligned}[b]
                        % \exists \sigma > 0: \forall \mathbf{x} \in \mathbb{R}^{D}, \mathbb{E}_{(\mathbf{s}_{ij}, y) \sim \mathcal{D}_{i}} \left[ \norm{ \grad_{\mathbf{x}} \ell(\mathbf{s}_{ij}, y_{ij}; \mathbf{x}) - \mathbb{E}_{(\mathbf{s}_{ij}, y) \sim \mathcal{D}_{i}} \left[ \grad_{\mathbf{x}} \ell(\mathbf{s}_{ij}, y_{ij}; \mathbf{x}) \right] }^{2} \right] \le \sigma^{2}.
                        & \norm{ \grad_{\mathbf{x}} \ell(\mathbf{s}_{ij}, y_{ij}; \mathbf{x}) - \mathbb{E}_{(\mathbf{s}_{ij}, y) \sim \mathcal{D}_{i}} \left[ \grad_{\mathbf{x}} \ell(\mathbf{s}_{ij}, y_{ij}; \mathbf{x}) \right] } \\
                        & \le \norm{ \grad_{\mathbf{x}} \ell(\mathbf{s}_{ij}, y_{ij}; \mathbf{x})} + \norm{ \mathbb{E}_{(\mathbf{s}_{ij}, y) \sim \mathcal{D}_{i}} \left[ \grad_{\mathbf{x}} \ell(\mathbf{s}_{ij}, y_{ij}; \mathbf{x}) \right] } \quad \text{(triangle inequality)} \\
                        & \le \norm{ \grad_{\mathbf{x}} \ell(\mathbf{s}_{ij}, y_{ij}; \mathbf{x}) } + \mathbb{E}_{(\mathbf{s}_{ij}, y) \sim \mathcal{D}_{i}} \left[ \norm{ \grad_{\mathbf{x}} \ell(\mathbf{s}_{ij}, y_{ij}; \mathbf{x}) } \right] \quad \text{(Jensen's inequality)} \\
                        & \le 2L \quad \text{(Boundedness of gradient in \eqref{eq:bounded_grad})}.
                    \end{aligned}
                \end{equation}
                Thus, selecting \(\sigma = 2L\) completes the proof.
            \end{proof}

        \subsubsection{Boundedness of variance of the weighted loss}
        \label{sec:boundedness_variance_loss}
            \LemmaBoundedVarianceWeightedLoss*
            \begin{proof}
                We use the following well-known inequality for variance as a part of the proof.
    
                If \(X_{i}, \forall i \in \{1, \ldots, n\}\) are random variables with finite variance: \(\mathrm{Var}\left( X_{i} \right) < +\infty\), then:
                \begin{equation*}
                    \mathrm{Var} \left( \sum_{i=1}^{n} X_{i} \right) \le n \sum_{i=1}^{n} \mathrm{Var} \left( X_{i} \right).
                \end{equation*} 
    
                Note that \(\pmb{\ell}_{i} (\mathbf{x}_{t})\) in \eqref{eq:validation_loss} is the empirical expected values of loss \(\ell\) evaluated on task \(i\)-th.
                Hence, applying the above inequality for variance gives:
                \begin{equation}
                    \begin{aligned}[b]
                        \mathrm{Var} \left( \grad_{\mathbf{x}} \pmb{\ell}_{i} (\mathbf{x}_{t}) \right) & = \mathrm{Var} \left( \frac{1}{m_{q}}\sum_{j=1}^{m_{q}} \grad_{\mathbf{x}} \ell \left( \mathbf{s}_{ij}^{(q)}, y_{ij}^{(q)}; \mathbf{x} - \frac{\gamma}{m_{s}} \sum_{k=1}^{m_{s}} \grad_{\mathbf{x}} \left[ \ell \left( \mathbf{s}_{ik}^{(s)}, y_{ik}^{(s)}; \mathbf{x}_{t} \right) \right] \right) \right) \\
                        & \le \frac{1}{m_{q}} \sum_{j=1}^{m_{q}} \mathrm{Var} \left( \grad_{\mathbf{x}} \ell \right) \\
                        & \le \sigma^{2}.
                    \end{aligned}
                \end{equation}
    
                Hence, the variance of \(\mathbf{u}_{ti} \grad_{\mathbf{x}} \bar{\pmb{\ell}}(\mathbf{x}_{t})\) is bounded by \(\mathbf{u}_{ti}^{2} \sigma^{2}\). This leads to:
                \begin{equation}
                    \begin{aligned}[b]
                        \mathrm{Var} \left( \grad_{\mathbf{x}} \mathbf{u}_{t}^{\top} \bar{\pmb{\ell}} (\mathbf{x}_{t}) \right) & = \mathrm{Var} \left( \sum_{i=1}^{M} \mathbf{u}_{ti} \grad_{\mathbf{x}} \pmb{\ell}_{i} (\mathbf{x}_{t}) \right) \\
                        & = \sum_{i = 1}^{M} \mathbf{u}_{ti}^{2} \mathrm{Var} \left( \grad_{\mathbf{x}} \pmb{\ell}_{i}(\mathbf{x}_{t}) \right) \\
                        & \le \sigma^{2} \sum_{i=1}^{M} \mathbf{u}_{ti}^{2} = \sigma^{2} \lVert \mathbf{u}_{t} \rVert^{2}\\
                        & \le \sigma^{2} \left( \delta + M^{-0.5} \right)^{2} \quad \text{(\cref{corollary:bound_norm_u})}.
                    \end{aligned}
                \end{equation}
            \end{proof}

        \subsubsection{Smoothness of validation loss}
        \label{sec:smoothness_loss}
            In this \subsubsectionautorefname, we prove \cref{lemma:smoothness} about the smoothness of validation loss. To make the \subsubsectionautorefname~self-contained, we re-state the definition of the task-specific parameter \(\phi_{i}(\mathbf{x})\) and the true validation loss \(\bar{\pmb{\ell}}_{i} (\mathbf{x})\) as follows:

            \begin{align}
                \phi_{i}(\mathbf{x}) & = \mathbf{x} - \frac{\gamma}{m_{i}^{(s)}} \sum_{k=1}^{m_{i}^{(s)}} \nabla_{\mathbf{x}} \left[ \ell \left( \mathbf{s}_{ik}^{(s)}, y_{ik}^{(s)}; \mathbf{x} \right) \right] \tag{\ref{eq:task_adaptation_gd}}\\
                \bar{\pmb{\ell}}_{i} (\mathbf{x}) & = \mathbb{E}_{\left( \mathbf{s}_{ij}^{(q)}, y_{ij}^{(q)} \right) \sim \mathcal{D}_{i}^{(q)}} \left[ \ell \left( \mathbf{s}_{ij}^{(q)}, y_{ij}^{(q)}; \phi(\mathbf{x}) \right) \right]. \tag{\ref{eq:true_validation_loss}}
            \end{align}

            \cref{lemma:smoothness} and its proof are shown as follows:

            \LemmaSmoothness*

            \begin{proof}
                Before starting the proof, we abuse the notation of gradient of the loss function at a point \(\mathbf{x} = \mathbf{v}\) as follows:
                \begin{equation}
                    \grad_{\mathbf{x}} \ell (\mathbf{s}, y; \mathbf{v}) = \grad_{\mathbf{x}} \ell (\mathbf{s}, y; \mathbf{x}) \eval*{}_{\mathbf{x} = \mathbf{v}.}
                \end{equation}

                Given the definition of the true validation loss in Eq.~\eqref{eq:true_validation_loss}, its gradient w.r.t. \(\mathbf{x}\) can be calculated using chain rule as follows:
                \begin{equation}
                    \begin{aligned}[b]
                        \grad_{\mathbf{x}} \bar{\pmb{\ell}}_{i} (\mathbf{x}) & = \grad_{\mathbf{x}} \mathbb{E}_{\left( \mathbf{s}_{ij}^{(q)}, y_{ij}^{(q)} \right) \sim \mathcal{D}_{i}^{(q)}} \left[ \ell \left( \mathbf{s}_{ij}^{(q)}, y_{ij}^{(q)}; \phi_{i}(\mathbf{x}) \right) \right] \\
                        & = \mathbb{E}_{\left( \mathbf{s}_{ij}^{(q)}, y_{ij}^{(q)} \right) \sim \mathcal{D}_{i}^{(q)}} \left[ \grad_{\mathbf{x}} \phi_{i}(\mathbf{x}) \times \grad_{\mathbf{x}} \ell \left( \mathbf{s}_{ij}^{(q)}, y_{ij}^{(q)}; \phi_{i}(\mathbf{x}) \right) \right].
                    \end{aligned}
                \end{equation}
                And since \(\phi_{i}(\mathbf{x})\) defined in Eq.~\eqref{eq:task_adaptation_gd} does not depends on validation (or query) samples, we can, therefore, rewrite the above gradient as:
                \begin{equation}
                    \begin{split}
                        \grad_{\mathbf{x}} \bar{\pmb{\ell}}_{i} (\mathbf{x}) & = \grad_{\mathbf{x}} \phi_{i}(\mathbf{x}) \times \mathbb{E}_{\left( \mathbf{s}_{ij}^{(q)}, y_{ij}^{(q)} \right) \sim \mathcal{D}_{i}^{(q)}} \left[ \grad_{\mathbf{x}} \ell \left( \mathbf{s}_{ij}^{(q)}, y_{ij}^{(q)}; \phi_{i}(\mathbf{x}) \right) \right]\\
                        & = \left\{ \mathbf{I} - \gamma  \mathbb{E}_{\left(\mathbf{s}_{ij}^{(s)}, y_{ij}^{(s)} \right) \sim \mathcal{S}_{i}^{(s)}} \left[ \grad_{\mathbf{x}}^{2} \ell \left( \mathbf{s}_{ik}^{(s)}, y_{ik}^{(s)} ;\mathbf{x} \right) \right] \right\} \\
                        & \quad \times \mathbb{E}_{\left( \mathbf{s}_{ij}^{(q)}, y_{ij}^{(q)} \right) \sim \mathcal{D}_{i}^{(q)}} \left[ \grad_{\mathbf{x}} \ell \left( \mathbf{s}_{ij}^{(q)}, y_{ij}^{(q)}; \phi_{i}(\mathbf{x}) \right) \right].
                    \end{split}
                \end{equation}
                Note that we abuse the notation and use \(\mathbb{E}_{\left(\mathbf{s}_{ij}^{(s)}, y_{ij}^{(s)} \right) \sim \mathcal{S}_{i}^{(s)}}\) to indicate the average evaluated on all data points in set \(\mathcal{S}_{i}\).
    
                In the following, we omit sample \((\mathbf{s}, y)\) from the expectation to simplify the notations. In particular, the above gradient can be re-written as:
                \begin{equation}
                    \begin{aligned}[b]
                        \grad_{\mathbf{x}} \bar{\pmb{\ell}}_{i} (\mathbf{x}) & = \left\{ \mathbf{I} - \gamma  \mathbb{E}_{\mathcal{S}_{i}^{(s)}} \left[ \grad_{\mathbf{x}}^{2} \ell \left( \mathbf{s}_{ik}^{(s)}, y_{ik}^{(s)}; \mathbf{x} \right) \right] \right\} \mathbb{E}_{ \mathcal{D}_{i}^{(q)}} \left[ \grad_{\mathbf{x}} \ell \left( \mathbf{s}_{ij}^{(q)}, y_{ij}^{(q)}; \phi_{i}(\mathbf{x}) \right)
                        \right].
                    \end{aligned}
                \end{equation}
    
                Thus, we can calculate the difference of the gradient evaluated on the same task \(\mathcal{T}_{i}\) but with two different meta-parameters:
                \begin{equation}
                    \begin{aligned}[b]
                        & \left\Vert \grad_{\mathbf{x}} \bar{\pmb{\ell}}_{i} \left( \bar{\mathbf{x}} \right) - \grad_{\mathbf{x}} \bar{\pmb{\ell}}_{i} \left( \widetilde{\mathbf{x}} \right) \right\Vert\\
                        & = \left\Vert \left\{ \mathbf{I} - \gamma  \mathbb{E}_{\mathcal{S}_{i}^{(s)}} \left[ \grad_{\mathbf{x}}^{2} \ell \left( \mathbf{s}_{ik}^{(s)}, y_{ik}^{(s)}; \bar{\mathbf{x}} \right)\right] \right\} \mathbb{E}_{ \mathcal{D}_{i}^{(q)}} \left[ \grad_{\mathbf{x}} \ell \left( \mathbf{s}_{ik}^{(q)}, y_{ik}^{(q)}; \phi_{i}(\bar{\mathbf{x}}) \right) \right] \right.\\
                        & \quad \left. - \left\{ \mathbf{I} - \gamma  \mathbb{E}_{\mathcal{S}_{i}^{(s)}} \left[ \grad_{\mathbf{x}}^{2} \ell \left( \mathbf{s}_{ik}^{(s)}, y_{ik}^{(s)}; \widetilde{\mathbf{x}} \right) \right] \right\} \mathbb{E}_{ \mathcal{D}_{i}^{(q)}} \left[ \grad_{\mathbf{x}} \ell \left( \mathbf{s}_{ik}^{(q)}, y_{ik}^{(q)}; \phi_{i}(\widetilde{\mathbf{x}}) \right) \right] \right\Vert\\
                        & = \left\Vert \left\{ \mathbf{I} - \gamma  \mathbb{E}_{\mathcal{S}_{i}^{(s)}} \left[ \grad_{\mathbf{x}}^{2} \ell \left( \mathbf{s}_{ik}^{(s)}, y_{ik}^{(s)}; \bar{\mathbf{x}} \right)\right] \right\} \mathbb{E}_{ \mathcal{D}_{i}^{(q)}} \left[ \grad_{\mathbf{x}} \ell \left( \mathbf{s}_{ik}^{(q)}, y_{ik}^{(q)}; \phi_{i}(\bar{\mathbf{x}}) \right) \right] \right.\\
                        & \quad \left. - \left\{ \mathbf{I} - \gamma  \mathbb{E}_{\mathcal{S}_{i}^{(s)}} \left[ \grad_{\mathbf{x}}^{2} \ell \left( \mathbf{s}_{ik}^{(s)}, y_{ik}^{(s)}; \bar{\mathbf{x}} \right)\right] + \gamma  \mathbb{E}_{\mathcal{S}_{i}^{(s)}} \left[ \grad_{\mathbf{x}}^{2} \ell \left( \mathbf{s}_{ik}^{(s)}, y_{ik}^{(s)}; \bar{\mathbf{x}} \right)\right] \right. \right. \\
                        & \qquad \left. - \gamma  \mathbb{E}_{\mathcal{S}_{i}^{(s)}} \left[ \grad_{\mathbf{x}}^{2} \ell \left( \mathbf{s}_{ik}^{(s)}, y_{ik}^{(s)}; \widetilde{\mathbf{x}} \right) \right] \right\} \times \left. \mathbb{E}_{ \mathcal{D}_{i}^{(q)}} \left[ \grad_{\mathbf{x}} \ell \left( \mathbf{s}_{ik}^{(q)}, y_{ik}^{(q)}; \phi_{i}(\widetilde{\mathbf{x}}) \right) \right] \right\Vert\\
                        & = \left\Vert \left\{ \mathbf{I} - \gamma  \mathbb{E}_{\mathcal{S}_{i}^{(s)}} \left[ \grad_{\mathbf{x}}^{2} \ell \left( \mathbf{s}_{ik}^{(s)}, y_{ik}^{(s)}; \bar{\mathbf{x}} \right) \right] \right\} \mathbb{E}_{ \mathcal{D}_{i}^{(q)}} \left[ \grad_{\mathbf{x}} \ell \left( \mathbf{s}_{ik}^{(q)}, y_{ik}^{(q)}; \phi_{i}(\bar{\mathbf{x}}) \right) - \grad_{\mathbf{x}} \ell \left( \mathbf{s}_{ik}^{(q)}, y_{ik}^{(q)}; \phi_{i}(\widetilde{\mathbf{x}}) \right) \right] \right.\\
                        & \qquad\quad \left. - \gamma \mathbb{E}_{\mathcal{S}_{i}^{(s)}} \left[ \grad_{\mathbf{x}}^{2} \ell \left( \mathbf{s}_{ik}^{(s)}, y_{ik}^{(s)}; \bar{\mathbf{x}} \right) - \grad_{\mathbf{x}}^{2} \ell \left( \mathbf{s}_{ik}^{(s)}, y_{ik}^{(s)}; \widetilde{\mathbf{x}} \right) \right] \mathbb{E}_{ \mathcal{D}_{i}^{(q)}} \left[ \grad_{\mathbf{x}} \ell \left( \mathbf{s}_{ik}^{(q)}, y_{ik}^{(q)}; \phi_{i}(\widetilde{\mathbf{x}}) \right)
                        \right] \right\Vert.
                    \end{aligned}
                \end{equation}
    
                Applying the triangle inequality gives:
                \begin{equation}
                    \begin{aligned}[b]
                        & \left\Vert \grad_{\mathbf{x}} \bar{\pmb{\ell}}_{i} \left( \bar{\mathbf{x}} \right) - \grad_{\mathbf{x}} \bar{\pmb{\ell}}_{i} \left( \widetilde{\mathbf{x}} \right) \right\Vert \\
                        & \le \left\Vert \left\{ \mathbf{I} - \gamma  \mathbb{E}_{\mathcal{S}_{i}^{(s)}} \left[ \grad_{\mathbf{x}}^{2} \ell \left( \mathbf{s}_{ik}^{(s)}, y_{ik}^{(s)}; \bar{\mathbf{x}} \right) \right] \right\} \mathbb{E}_{ \mathcal{D}_{i}^{(q)}} \left[ \grad_{\mathbf{x}} \ell \left( \mathbf{s}_{ik}^{(q)}, y_{ik}^{(q)}; \phi_{i}(\bar{\mathbf{x}}) \right) - \grad_{\mathbf{x}} \ell \left( \mathbf{s}_{ik}^{(q)}, y_{ik}^{(q)}; \phi_{i}(\widetilde{\mathbf{x}}) \right) \right] \right\Vert \\
                        & \quad + \left\Vert \gamma \mathbb{E}_{\mathcal{S}_{i}^{(s)}} \left[ \grad_{\mathbf{x}}^{2} \ell \left( \mathbf{s}_{ik}^{(s)}, y_{ik}^{(s)}; \bar{\mathbf{x}} \right) - \grad_{\mathbf{x}}^{2} \ell \left( \mathbf{s}_{ik}^{(s)}, y_{ik}^{(s)}; \widetilde{\mathbf{x}} \right) \right] \mathbb{E}_{ \mathcal{D}_{i}^{(q)}} \left[ \grad_{\mathbf{x}} \ell \left( \mathbf{s}_{ik}^{(q)}, y_{ik}^{(q)}; \phi_{i}(\widetilde{\mathbf{x}}) \right)
                        \right] \right\Vert.
                    \end{aligned}
                    \label{eq:upper_bound_loss_smoothness}
                \end{equation}
    
                Next, we upper-bound the two terms in the right-hand side of Ineq.~\eqref{eq:upper_bound_loss_smoothness}. The first term can be upper-bounded as:
                \begin{equation}
                    \begin{aligned}[b]
                        \text{First term} & = \left\Vert \left\{ \mathbf{I} - \gamma  \mathbb{E}_{\mathcal{S}_{i}^{(s)}} \left[ \grad_{\mathbf{x}}^{2} \ell \left( \mathbf{s}_{ik}^{(s)}, y_{ik}^{(s)}; \bar{\mathbf{x}} \right) \right] \right\} \right.\\
                        & \quad \left. \times \mathbb{E}_{ \mathcal{D}_{i}^{(q)}} \left[ \grad_{\mathbf{x}} \ell \left( \mathbf{s}_{ik}^{(q)}, y_{ik}^{(q)}; \phi_{i}(\bar{\mathbf{x}}) \right) - \grad_{\mathbf{x}} \ell \left( \mathbf{s}_{ik}^{(q)}, y_{ik}^{(q)}; \phi_{i}(\widetilde{\mathbf{x}}) \right) \right] \right\Vert\\
                        & \le \left\Vert \mathbf{I} - \gamma  \mathbb{E}_{\mathcal{S}_{i}^{(s)}} \left[ \grad_{\mathbf{x}}^{2} \ell \left( \mathbf{s}_{ik}^{(s)}, y_{ik}^{(s)}; \bar{\mathbf{x}} \right) \right] \right\Vert \\
                        & \quad \times \left\Vert \mathbb{E}_{ \mathcal{D}_{i}^{(q)}} \left[ \grad_{\mathbf{x}} \ell \left( \mathbf{s}_{ik}^{(q)}, y_{ik}^{(q)}; \phi_{i}(\bar{\mathbf{x}}) \right) - \grad_{\mathbf{x}} \ell \left( \mathbf{s}_{ik}^{(q)}, y_{ik}^{(q)}; \phi_{i}(\widetilde{\mathbf{x}}) \right) \right] \right\Vert
                    \end{aligned}
                \end{equation}
    
                Applying Jensen's inequality on the L2 norm of the expectation in the right-hand side of the above inequality to bring the expectation outside of the L2 norm, then employing the smoothness of \(\ell\) in \cref{assumption:loss_smoothness} to obtain the following:
                \begin{equation}
                    \begin{aligned}[b]
                        \text{First term} & \le \left\Vert \mathbf{I} - \gamma  \mathbb{E}_{\mathcal{S}_{i}^{(s)}} \left[ \grad_{\mathbf{x}}^{2} \ell \left( \mathbf{s}_{ik}^{(s)}, y_{ik}^{(s)}; \bar{\mathbf{x}} \right) \right] \right\Vert \\
                        & \quad \times \mathbb{E}_{ \mathcal{D}_{i}^{(q)}} \left\Vert \grad_{\mathbf{x}} \ell \left( \mathbf{s}_{ik}^{(q)}, y_{ik}^{(q)}; \phi_{i}(\bar{\mathbf{x}}) \right) - \grad_{\mathbf{x}} \ell \left( \mathbf{s}_{ik}^{(q)}, y_{ik}^{(q)}; \phi_{i}(\widetilde{\mathbf{x}}) \right) \right\Vert\\
                        & \le \left\Vert \mathbf{I} - \gamma  \mathbb{E}_{\mathcal{S}_{i}^{(s)}} \left[ \grad_{\mathbf{x}}^{2} \ell \left( \mathbf{s}_{ik}^{(s)}, y_{ik}^{(s)}; \bar{\mathbf{x}} \right) \right] \right\Vert \times S \left\Vert \phi_{i}(\bar{\mathbf{x}}) - \phi_{i}(\widetilde{\mathbf{x}}) \right\Vert .
                    \end{aligned}
                    \label{eq:first_term_1}
                \end{equation}

                Given the definition of \(\phi(\mathbf{x})\) in Eq.~\eqref{eq:task_adaptation_gd}, we can obtain the following:
                \begin{equation}
                    \begin{split}
                        \left\Vert \phi_{i}(\bar{\mathbf{x}}) - \phi_{i}(\widetilde{\mathbf{x}}) \right\Vert & = \left\Vert \left( \bar{\mathbf{x}} - \widetilde{\mathbf{x}} \right) - \gamma \mathbb{E}_{\mathcal{S}_{i}^{(s)}} \left[ \grad_{\mathbf{x}} \ell \left( \mathbf{s}_{ik}^{(s)}, y_{ij}^{(s)}; \bar{\mathbf{x}} \right) - \grad_{\mathbf{x}} \ell \left( \mathbf{s}_{ik}^{(s)}, y_{ij}^{(s)}; \widetilde{\mathbf{x}} \right) \right] \right\Vert\\
                        & \le \left\Vert \bar{\mathbf{x}} - \widetilde{\mathbf{x}} \right\Vert + \gamma \left\Vert \mathbb{E}_{\mathcal{S}_{i}^{(s)}} \left[ \grad_{\mathbf{x}} \ell \left( \mathbf{s}_{ik}^{(s)}, y_{ij}^{(s)}; \bar{\mathbf{x}} \right) - \grad_{\mathbf{x}} \ell \left( \mathbf{s}_{ik}^{(s)}, y_{ij}^{(s)}; \widetilde{\mathbf{x}} \right) \right] \right\Vert \\
                        & \qquad \text{(triangle inequality)}\\
                        & \le \left\Vert \bar{\mathbf{x}} - \widetilde{\mathbf{x}} \right\Vert + \gamma \mathbb{E}_{\mathcal{S}_{i}^{(s)}} \left\Vert \grad_{\mathbf{x}} \ell \left( \mathbf{s}_{ik}^{(s)}, y_{ij}^{(s)}; \bar{\mathbf{x}} \right) - \grad_{\mathbf{x}} \ell \left( \mathbf{s}_{ik}^{(s)}, y_{ij}^{(s)}; \widetilde{\mathbf{x}} \right) \right\Vert \\
                        & \qquad \text{(Jensen's inequality)}\\
                        & \le \left\Vert \bar{\mathbf{x}} - \widetilde{\mathbf{x}} \right\Vert + \gamma S \left\Vert \bar{\mathbf{x}} - \widetilde{\mathbf{x}} \right\Vert ~\text{(\cref{assumption:loss_smoothness})}
                    \end{split}
                \end{equation}

                Thus, one can upper-bound further Ineq.~\eqref{eq:first_term_1} as follows:
                \begin{equation}
                    \begin{aligned}[b]
                        \text{First term} & \le S (1 + \gamma S) \left\Vert \mathbf{I} - \gamma  \mathbb{E}_{\mathcal{S}_{i}^{(s)}} \left[ \grad_{\mathbf{x}}^{2} \ell \left( \mathbf{s}_{ik}^{(s)}, y_{ik}^{(s)}; \bar{\mathbf{x}} \right) \right] \right\Vert \times \left\Vert \bar{\mathbf{x}} - \widetilde{\mathbf{x}} \right\Vert .
                    \end{aligned}
                \end{equation}
    
                If \(\{\lambda_{d}\}_{d=1}^{D}\) are the eigenvalues of the Hessian matrix \(\mathbb{E}_{\mathcal{S}_{i}^{(s)}} \left[ \grad_{\mathbf{x}}^{2} \ell \left( \mathbf{s}_{ik}^{(s)}, y_{ik}^{(s)}; \bar{\mathbf{x}} \right) \right]\), then due to \cref{lemma:eigen_value_shift}, the eigenvalues of \(\mathbf{I} - \gamma  \mathbb{E}_{\mathcal{S}_{i}^{(s)}} \left[ \grad_{\mathbf{x}}^{2} \ell \left( \mathbf{s}_{ik}^{(s)}, y_{ik}^{(s)}; \bar{\mathbf{x}} \right) \right]\) are \(\{1 - \gamma \lambda_{d}\}_{d = 1}^{D}\). In addition, since \(\mathbf{I} - \gamma  \mathbb{E}_{\mathcal{S}_{i}^{(s)}} \left[ \grad_{\mathbf{x}}^{2} \ell \left( \mathbf{s}_{ik}^{(s)}, y_{ik}^{(s)}; \bar{\mathbf{x}} \right) \right]\) is symmetric and positive semi-definite, its norm equals to the largest eigenvalue (refer to \cref{lemma:norm_max_eigenvalue}):
                \begin{equation}
                    \begin{aligned}
                        \left\Vert \mathbf{I} - \gamma  \mathbb{E}_{\mathcal{S}_{i}^{(s)}} \left[ \grad_{\mathbf{x}}^{2} \ell \left( \mathbf{s}_{ik}^{(s)}, y_{ik}^{(s)}; \bar{\mathbf{x}} \right) \right] \right\Vert & = \max_{d} \left\vert 1 - \gamma \lambda_{d} \right\vert \le 1 + \gamma \max_{d} | \lambda_{d} |.
                    \end{aligned}
                \end{equation}
                
                According to \cref{assumption:loss_smoothness}, one can imply that the eigenvalues of the Hessian matrix \(\mathbb{E}_{\mathcal{S}_{i}^{(s)}} \left[ \grad_{\mathbf{x}}^{2} \ell \left( \mathbf{s}_{ik}^{(s)}, y_{ik}^{(s)}; \bar{\mathbf{x}} \right) \right] \) are smaller than \(S\)~\citep[section 3.2, page 266]{bubeck2015convex}. This results in:
                \begin{equation}
                    \left\Vert \mathbf{I} - \gamma  \mathbb{E}_{\mathcal{S}_{i}^{(s)}} \left[ \grad_{\mathbf{x}}^{2} \ell \left( \mathbf{s}_{ik}^{(s)}, y_{ik}^{(s)}; \bar{\mathbf{x}} \right) \right] \right\Vert \le 1 + \gamma S.
                \end{equation}
    
                Therefore, the first term on the right-hand side of \eqref{eq:upper_bound_loss_smoothness} is upper-bounded by:
                \begin{equation}
                    \text{First term} \le S \left( 1 + \gamma S \right)^{2} \left\Vert \bar{\mathbf{x}} - \widetilde{\mathbf{x}} \right\Vert.
                    \label{eq:first_term_upper_bound}
                \end{equation}
    
                The second term in the right-hand side of \eqref{eq:upper_bound_loss_smoothness} can be upper-bounded as:
                \begin{equation}
                    \begin{aligned}[b]
                        \text{Second term} & = \left\Vert \gamma \mathbb{E}_{\mathcal{S}_{i}^{(s)}} \left[ \grad_{\mathbf{x}}^{2} \ell \left( \mathbf{s}_{ik}^{(s)}, y_{ik}^{(s)}; \bar{\mathbf{x}} \right) - \grad_{\mathbf{x}}^{2} \ell \left( \mathbf{s}_{ik}^{(s)}, y_{ik}^{(s)}; \widetilde{\mathbf{x}} \right) \right] \right.\\
                        & \quad \left. \times \mathbb{E}_{ \mathcal{D}_{i}^{(q)}} \left[ \grad_{\mathbf{x}} \ell \left( \mathbf{s}_{ik}^{(q)}, y_{ik}^{(q)}; \phi_{i}(\widetilde{\mathbf{x}}) \right)
                        \right] \right\Vert \\
                        & \le \norm{ \gamma \mathbb{E}_{\mathcal{S}_{i}^{(s)}} \left[ \grad_{\mathbf{x}}^{2} \ell \left( \mathbf{s}_{ik}^{(s)}, y_{ik}^{(s)}; \bar{\mathbf{x}} \right) - \grad_{\mathbf{x}}^{2} \ell \left( \mathbf{s}_{ik}^{(s)}, y_{ik}^{(s)}; \widetilde{\mathbf{x}} \right) \right] } \\
                        & \quad \times \norm{ \mathbb{E}_{ \mathcal{D}_{i}^{(q)}} \left[ \grad_{\mathbf{x}} \ell \left( \mathbf{s}_{ik}^{(q)}, y_{ik}^{(q)}; \phi_{i}(\widetilde{\mathbf{x}}) \right)
                        \right] } \\
                        & \le \gamma \mathbb{E}_{\mathcal{S}_{i}^{(s)}} \norm{ \grad_{\mathbf{x}}^{2} \ell \left( \mathbf{s}_{ik}^{(s)}, y_{ik}^{(s)}; \bar{\mathbf{x}} \right) - \grad_{\mathbf{x}}^{2} \ell \left( \mathbf{s}_{ik}^{(s)}, y_{ik}^{(s)}; \widetilde{\mathbf{x}} \right) } \\
                        & \quad \times \mathbb{E}_{ \mathcal{D}_{i}^{(q)}} \left\Vert \grad_{\mathbf{x}} \ell \left( \mathbf{s}_{ik}^{(q)}, y_{ik}^{(q)}; \phi_{i}(\widetilde{\mathbf{x}}) \right)
                        \right\Vert \quad \text{(Jensen's inequality)}\\
                        & \le \gamma \rho L \left\Vert \bar{\mathbf{x}} - \widetilde{\mathbf{x}} \right\Vert ~\text{(Eq.~\eqref{eq:bounded_grad} and \cref{assumption:hessian_loss_lipschitz})}.
                    \end{aligned}
                    \label{eq:second_term_upper_bound}
                \end{equation}

                Combining the results in \eqref{eq:upper_bound_loss_smoothness}, \eqref{eq:first_term_upper_bound} and \eqref{eq:second_term_upper_bound} gives:
                \begin{equation}
                    \left\Vert \grad_{\mathbf{x}} \bar{\pmb{\ell}}_{i} \left( \bar{\mathbf{x}} \right) - \grad_{\mathbf{x}} \bar{\pmb{\ell}}_{i} \left( \widetilde{\mathbf{x}} \right) \right\Vert \le \left[S (1 + \gamma S)^{2} + \gamma \rho L \right] \left\Vert \bar{\mathbf{x}} - \widetilde{\mathbf{x}} \right\Vert.
                \end{equation}
                This completes the proof.
            \end{proof}

    \subsection{Convergence of TOW}
    \label{sec:tow_convergence_proof}
        \TheoremConvergenceTOW*
        \begin{proof}
            According to \cref{eq:true_validation_loss}, \(\bar{\pmb{\ell}}_{i}(\mathbf{x}) \in \mathbb{R}\) is the expected validation loss of task \(i\)-th using meta-parameter \(\mathbf{x}\). In addition, the notation \(\mathcal{D}_{1:M}^{(q)}\) indicates the data probability that generates query data pairs for \(M\) tasks in a mini-batch.

            For convenience, we also denote \(\bar{\pmb{\ell}}(x) \in \mathbb{R}^{M}\) is the vector consisting of expected validation losses on \(M\) tasks in a mini-batch of tasks:
            \begin{equation}
                \bar{\pmb{\ell}}(x) = \begin{bmatrix}
                    \bar{\pmb{\ell}}_{1} (x) & \bar{\pmb{\ell}}_{2} (x) & \ldots & \bar{\pmb{\ell}}_{M} (x)
                \end{bmatrix}^{\top}.
            \end{equation}
            
            From \cref{lemma:smoothness}, the gradient of the validation loss \(\bar{\pmb{\ell}}_{i} \left(\mathbf{x} \right)\) is \(\widetilde{S}\)-Lipschitz continuous. Hence, applying Taylor's theorem gives:
            \begin{equation}
                \bar{\pmb{\ell}}_{i} \left( \mathbf{x}_{t + 1} \right) \le \bar{\pmb{\ell}}_{i} \left( \mathbf{x}_{t} \right) + \grad_{\mathbf{x}}^{\top} \bar{\pmb{\ell}}_{i} \left( \mathbf{x}_{t} \right) \left( \mathbf{x}_{t + 1} - \mathbf{x}_{t} \right) + \frac{\widetilde{S}}{2} \lVert \mathbf{x}_{t + 1} - \mathbf{x}_{t} \rVert^{2}.
            \end{equation}

            Note that \(\mathbf{u}_{ti}\) is constrained to be non-negative as mentioned in \cref{sec:ml2to} or in step~\ref{step:stopping_ilqr} of Algorithm~\ref{algm:task_weighting_ilqr}. Hence, one can multiply both sides by \(\mathbf{u}_{ti} \ge 0, t \in \{1, \ldots, T\}, i \in \{1, \ldots, M\}\) and sum side-by-side to obtain:
            \begin{equation}
                \mathbf{u}_{t}^{\top} \bar{\pmb{\ell}} \left( \mathbf{x}_{t + 1} \right) \le \mathbf{u}_{t}^{\top} \bar{\pmb{\ell}} \left( \mathbf{x}_{t} \right) + \grad_{\mathbf{x}}^{\top} \left[ \mathbf{u}_{t} \bar{\pmb{\ell}} \left( \mathbf{x}_{t} \right) \right] \left( \mathbf{x}_{t + 1} - \mathbf{x}_{t} \right) + \frac{\widetilde{S}}{2} \lVert \mathbf{x}_{t + 1} - \mathbf{x}_{t} \rVert^{2} \pmb{1}^{\top}_{M} \mathbf{u}_{t}.
            \end{equation}

            By conditioning on \(\mathbf{x}_{t}\) and taking the expectation over all data sampled from \(\{ \mathcal{D}_{i}^{(q)} \}_{i=1}^{M}\), which is used to calculate \(\mathbf{x}_{t + 1}\), we can obtain the following:
            \begin{equation}
                \begin{aligned}[b]
                    & \mathbb{E}_{\mathcal{D}_{1:M}^{(q)}} \left[ \mathbf{u}_{t}^{\top} \bar{\pmb{\ell}} \left( \mathbf{x}_{t + 1} \right) | \mathbf{x}_{t} \right] \\
                    & \le \mathbb{E}_{\mathcal{D}_{1:M}^{(q)}} \left[ \mathbf{u}_{t}^{\top} \bar{\pmb{\ell}}\left( \mathbf{x}_{t} \right) + \grad_{\mathbf{x}}^{\top} \left[ \mathbf{u}_{t} \bar{\pmb{\ell}} \left( \mathbf{x}_{t} \right) \right] \left( \mathbf{x}_{t + 1} - \mathbf{x}_{t} \right) \vphantom{\frac{1}{1}} + \left. \frac{\widetilde{S}}{2} \lVert \mathbf{x}_{t + 1} - \mathbf{x}_{t} \rVert^{2} \pmb{1}^{\top}_{M} \mathbf{u}_{t} ~ \right| ~ \mathbf{x}_{t} \right].
                \end{aligned}
            \end{equation}

            Note that for simplicity, we assume that the state-transition dynamics is SGD:
            \begin{equation}
                \mathbf{x}_{t + 1} - \mathbf{x}_{t} = - \alpha \grad_{\mathbf{x}} \left[ \mathbf{u}_{t}^{\top} \pmb{\ell} \left( \mathbf{x}_{t} \right) \right].
            \end{equation}

            Thus, one can simplify further to obtain:
            \begin{equation}
                \begin{aligned}[b]
                    \mathbb{E}_{\mathcal{D}_{1:M}^{(q)}} \left[ \mathbf{u}_{t}^{\top} \bar{\pmb{\ell}} \left( \mathbf{x}_{t + 1} \right) | \mathbf{x}_{t} \right] & \le \mathbb{E}_{\mathcal{D}_{1:M}^{(q)}} \left[ \mathbf{u}_{t}^{\top} \bar{\pmb{\ell}}\left( \mathbf{x}_{t} \right) - \alpha \grad_{\mathbf{x}}^{\top} \left[ \mathbf{u}_{t} \bar{\pmb{\ell}} \left( \mathbf{x}_{t} \right) \right] \grad_{\mathbf{x}} \left[ \mathbf{u}_{t}^{\top} \pmb{\ell} \left( \mathbf{x}_{t} \right) \right] \vphantom{\frac{1}{1}} \right.\\
                    & \qquad \left.\left. + \frac{\alpha^{2} \widetilde{S}}{2} \lVert \grad_{\mathbf{x}} \left[ \mathbf{u}_{t}^{\top} \pmb{\ell} \left( \mathbf{x}_{t} \right) \right] \rVert^{2} \pmb{1}^{\top}_{M} \mathbf{u}_{t} ~ \right| ~ \mathbf{x}_{t} \right].
                \end{aligned}
            \end{equation}
            Note that one can use \cref{eq:true_validation_loss} to imply that: 
            \begin{equation*}
                \grad_{\mathbf{x}} \left[ \mathbf{u}_{t} \bar{\pmb{\ell}} \left( \mathbf{x}_{t} \right) \right] = \mathbb{E}_{\mathcal{S}_{1:M}^{(q)} \sim \mathcal{D}_{1:M}^{(q) \, m_{1:M}^{(q)}}} \left[ \grad_{\mathbf{x}} \left[ \mathbf{u}_{t}^{\top} \pmb{\ell} \left( \mathbf{x}_{t} \right) \right] \right],
            \end{equation*}
            which also implies:
            \begin{equation}
                \grad_{\mathbf{x}} \left[ \mathbf{u}_{t} \bar{\pmb{\ell}} \left( \mathbf{x}_{t} \right) \right] = \mathbb{E}_{\mathcal{D}_{1:M}^{(q)}} \left[ \grad_{\mathbf{x}} \left[ \mathbf{u}_{t}^{\top} \pmb{\ell} \left( \mathbf{x}_{t} \right) \right] \right].
            \end{equation}
            Thus, the above inequality can be written as:
            \begin{equation}
                \begin{aligned}[b]
                    \mathbb{E}_{\mathcal{D}_{1:M}^{(q)}} \left[ \mathbf{u}_{t}^{\top} \bar{\pmb{\ell}} \left( \mathbf{x}_{t + 1} \right) | \mathbf{x}_{t} \right] & \le \mathbf{u}_{t}^{\top} \bar{\pmb{\ell}}\left( \mathbf{x}_{t} \right) - \alpha \left\Vert \grad_{\mathbf{x}} \left[ \mathbf{u}_{t} \bar{\pmb{\ell}} \left( \mathbf{x}_{t} \right) \right] \right\Vert^{2} \\
                    & \quad + \frac{\alpha^{2} \widetilde{S}}{2} \left[ \mathrm{Var} \left[\grad_{\mathbf{x}} \left[ \mathbf{u}_{t}^{\top} \pmb{\ell} \left( \mathbf{x}_{t} \right) \right] \right] + \left\Vert \grad_{\mathbf{x}} \left[ \mathbf{u}_{t} \bar{\pmb{\ell}} \left( \mathbf{x}_{t} \right) \right] \right\Vert^{2} \right] \pmb{1}^{\top}_{M} \mathbf{u}_{t}\\
                    & \qquad \qquad \text{(since \(\mathbb{E}[X^{2}] = \mathrm{Var}[X] + (\mathbb{E}[X])^{2}\) )}\\
                    & \le \mathbf{u}_{t}^{\top} \bar{\pmb{\ell}}\left( \mathbf{x}_{t} \right) - \alpha \left( 1 - \frac{\alpha \widetilde{S}}{2} \pmb{1}^{\top}_{M} \mathbf{u}_{t} \right) \left\Vert \grad_{\mathbf{x}} \left[ \mathbf{u}_{t} \bar{\pmb{\ell}} \left( \mathbf{x}_{t} \right) \right] \right\Vert^{2} \\
                    & \quad + \frac{\alpha^{2} \widetilde{S}}{2} \mathrm{Var} \left[\grad_{\mathbf{x}} \left[ \mathbf{u}_{t}^{\top} \pmb{\ell} \left( \mathbf{x}_{t} \right) \right] \right]\pmb{1}^{\top}_{M} \mathbf{u}_{t}.
                \end{aligned}
            \end{equation}
            
            Since:
            \begin{equation*}
                \begin{dcases}
                    \pmb{1}^{\top}_{M} \mathbf{u}_{t} & \le \delta \sqrt{M} + 1 \quad \text{(\cref{corollary:bound_norm_u} in \appendixname~\ref{sec:boundedness_action})}\\
                    \mathrm{Var} \left[\grad_{\mathbf{x}} \left[ \mathbf{u}_{t}^{\top} \pmb{\ell} \left( \mathbf{x}_{t} \right) \right] \right] & \le \widetilde{\sigma}^{2} \qquad \qquad\text{(\cref{lemma:bounded_variance_weighted_loss})}
                \end{dcases}
            \end{equation*}
            then:
            \begin{equation}
                \begin{aligned}[b]
                    \mathbb{E}_{\mathcal{D}_{1:M}^{(q)}} \left[ \mathbf{u}_{t}^{\top} \bar{\pmb{\ell}} \left( \mathbf{x}_{t + 1} \right) | \mathbf{x}_{t} \right] & \le \mathbf{u}_{t}^{\top} \bar{\pmb{\ell}}\left( \mathbf{x}_{t} \right) - \alpha \left[ 1 - \frac{\alpha \widetilde{S}}{2} \left( \delta \sqrt{M} + 1 \right) \right] \left\Vert \grad_{\mathbf{x}} \left[ \mathbf{u}_{t} \bar{\pmb{\ell}} \left( \mathbf{x}_{t} \right) \right] \right\Vert^{2} \\
                    & \quad + \frac{\alpha^{2} \widetilde{\sigma}^{2} \widetilde{S}}{2} \left( \delta \sqrt{M} + 1 \right).
                \end{aligned}
            \end{equation}

            Re-arranging the gradient norm of the weighted validation loss to the left-hand side gives:
            \begin{equation}
                \begin{aligned}[b]
                    \alpha \left[ 1 - \frac{\alpha \widetilde{S}}{2} \left( \delta \sqrt{M} + 1 \right) \right] \left\Vert \grad_{\mathbf{x}} \left[ \mathbf{u}_{t} \bar{\pmb{\ell}} \left( \mathbf{x}_{t} \right) \right] \right\Vert^{2} & \le \mathbf{u}_{t}^{\top} \bar{\pmb{\ell}}\left( \mathbf{x}_{t} \right) - \mathbb{E}_{\mathcal{D}_{1:M}^{(q)}} \left[ \mathbf{u}_{t}^{\top} \bar{\pmb{\ell}} \left( \mathbf{x}_{t + 1} \right) | \mathbf{x}_{t} \right] \\
                    & \quad + \frac{\alpha^{2} \widetilde{\sigma}^{2} \widetilde{S}}{2} \left( \delta \sqrt{M} + 1\right).
                \end{aligned}
                \label{eq:proof_70}
            \end{equation}

            The right-hand-side, excluding the constant term at the end, can be written as:
            \begin{equation}
                \begin{aligned}[b]
                    \mathbf{u}_{t}^{\top} \bar{\pmb{\ell}}\left( \mathbf{x}_{t} \right) - \mathbb{E}_{\mathcal{D}_{1:M}^{(q)}} \left[ \mathbf{u}_{t}^{\top} \bar{\pmb{\ell}} \left( \mathbf{x}_{t + 1} \right) | \mathbf{x}_{t} \right] & = \left[ \mathbf{u}_{t}^{\top} \bar{\pmb{\ell}}\left( \mathbf{x}_{t} \right) - \mathbb{E}_{\mathcal{D}_{1:M}^{(q)}} \left[ \mathbf{u}_{t + 1}^{\top} \bar{\pmb{\ell}}\left( \mathbf{x}_{t + 1} \right) | \mathbf{x}_{t} \right] \right]\\
                    & \quad + \left[ \mathbb{E}_{\mathcal{D}_{1:M}^{(q)}} \left[ \mathbf{u}_{t + 1}^{\top} \bar{\pmb{\ell}}\left( \mathbf{x}_{t + 1} \right) | \mathbf{x}_{t} \right] - \mathbb{E}_{\mathcal{D}_{1:M}^{(q)}} \left[ \mathbf{u}_{t}^{\top} \bar{\pmb{\ell}} \left( \mathbf{x}_{t + 1} \right) | \mathbf{x}_{t} \right] \right]
                \end{aligned}
            \end{equation}

            The second part in the right-hand side of the above expression can be upper-bounded as:
            \begin{equation}
                \begin{aligned}[b]
                    & \mathbb{E}_{\mathcal{D}_{1:M}^{(q)}} \left[ \mathbf{u}_{t + 1}^{\top} \bar{\pmb{\ell}}\left( \mathbf{x}_{t + 1} \right) | \mathbf{x}_{t} \right] - \mathbb{E}_{\mathcal{D}_{1:M}^{(q)}} \left[ \mathbf{u}_{t}^{\top} \bar{\pmb{\ell}} \left( \mathbf{x}_{t + 1} \right) | \mathbf{x}_{t} \right]\\
                    & = \mathbb{E}_{\mathcal{D}_{1:M}^{(q)}} \left[ \left( \mathbf{u}_{t + 1} - \mathbf{u}_{t} \right)^{\top} \bar{\pmb{\ell}} \left( \mathbf{x}_{t + 1} \right) | \mathbf{x}_{t} \right]\\
                    & \le \mathbb{E}_{\mathcal{D}_{1:M}^{(q)}} \left[ \left\Vert \mathbf{u}_{t + 1} - \mathbf{u}_{t} \right\Vert \left. \sqrt{\sum_{i = 1}^{M} \bar{\pmb{\ell}}_{i}^{2} \left( \mathbf{x}_{t + 1} \right)} ~ \right| \mathbf{x}_{t} \right] \quad \text{(Cauchy-Schwarz inequality)}\\
                    & \le B \sqrt{M} ~ \mathbb{E}_{\mathcal{D}_{1:M}^{(q)}} \left[ \left\Vert \mathbf{u}_{t + 1} - \mathbf{u}_{t} \right\Vert | \mathbf{x}_{t} \right] \quad \text{(\(\ell\) is \(B\)-bounded, and hence, \(\bar{\pmb{\ell}}_{i}\) is \(B\)-bounded)}\\
                    & \le 2\delta B \sqrt{M} \quad \text{(\cref{corollary:bound_norm_u})}.
                \end{aligned}
            \end{equation}
    
            Hence, one can upper-bound further the right-hand side of \eqref{eq:proof_70}, resulting in:
            \begin{equation}
                \begin{aligned}[b]
                    \alpha \left[ 1 - \frac{\alpha \widetilde{S}}{2} \left( \delta \sqrt{M} + 1 \right) \right] \left\Vert \grad_{\mathbf{x}} \left[ \mathbf{u}_{t} \bar{\pmb{\ell}} \left( \mathbf{x}_{t} \right) \right] \right\Vert^{2} & \le \mathbf{u}_{t}^{\top} \bar{\pmb{\ell}}\left( \mathbf{x}_{t} \right) - \mathbb{E}_{\mathcal{D}_{1:M}^{(q)}} \left[ \mathbf{u}_{t + 1}^{\top} \bar{\pmb{\ell}}\left( \mathbf{x}_{t + 1} \right) | \mathbf{x}_{t} \right] \\
                    & \quad + 2\delta B \sqrt{M} + \frac{\alpha^{2} \widetilde{\sigma}^{2} \widetilde{S}}{2} \left( \delta \sqrt{M} + 1 \right).
                \end{aligned}
                \label{eq:equation_86}
            \end{equation}

            Note that the recursive gradient update is:
            \begin{equation*}
                \begin{aligned}[b]
                    \mathbf{x}_{t + 1} & = \mathbf{x}_{t} - \alpha \grad_{\mathbf{x}} \left[ \mathbf{u}_{t }^{\top} \pmb{\ell} \left( \mathbf{x}_{t} \right) \right]\\
                    \mathbf{x}_{t} & = \mathbf{x}_{t - 1} - \alpha \grad_{\mathbf{x}} \left[ \mathbf{u}_{t - 1}^{\top} \pmb{\ell} \left( \mathbf{x}_{t - 1} \right) \right]\\
                    & \vdots\\
                    \mathbf{x}_{2} & = \mathbf{x}_{1} - \alpha \grad_{\mathbf{x}} \left[ \mathbf{u}_{1}^{\top} \pmb{\ell} \left( \mathbf{x}_{1} \right) \right].
                \end{aligned}
            \end{equation*}

            We take the expectation over all the mini-batches used at time step \(t, t - 1, \ldots, 1\) to remove the conditioning in \eqref{eq:equation_86}. We denote this expectation as \(\mathbb{E}_{\mathcal{D}_{1:M}^{(q)\, t}}\), where the superscript \(t\) indicates the power, meaning that the expectation is carried out over \(t\) mini-batches that are used to calculate the state \(\mathbf{x}_{t + 1}\) from \(\mathbf{x}_{1}\). This results in:
            \begin{equation}
                \begin{aligned}[b]
                    & \alpha \left[ 1 - \frac{\alpha \widetilde{S}}{2} \left( \delta \sqrt{M} + 1 \right) \right] \mathbb{E}_{\mathcal{D}_{1:M}^{(q) \, t}} \left[ \left\Vert \grad_{\mathbf{x}} \left[ \mathbf{u}_{t} \bar{\pmb{\ell}} \left( \mathbf{x}_{t} \right) \right] \right\Vert^{2} \right] \\
                    & \le \mathbb{E}_{\mathcal{D}_{1:M}^{(q) \, t}} \left[ \mathbf{u}_{t}^{\top} \bar{\pmb{\ell}}\left( \mathbf{x}_{t} \right) \right] - \mathbb{E}_{\mathcal{D}_{1:M}^{(q) \, t}} \left[ \left. \mathbf{u}_{t + 1}^{\top} \bar{\pmb{\ell}}\left( \mathbf{x}_{t + 1} \right) \right| \mathbf{x}_{1} \right]  + 2\delta B \sqrt{M} + \frac{\alpha^{2} \widetilde{\sigma}^{2} \widetilde{S}}{2} \left( \delta \sqrt{M} + 1 \right).
                \end{aligned}
                \label{eq:proof_74}
            \end{equation}

            Summing \eqref{eq:proof_74} from \(t = 1\) to \(T_{\mathrm{iter}}\) gives:
            \begin{equation}
                \begin{aligned}[b]
                    & \alpha \left[ 1 - \frac{\alpha \widetilde{S}}{2} \left( \delta \sqrt{M} + 1 \right) \right] \sum_{t = 1}^{T_{\mathrm{iter}}} \mathbb{E}_{\mathcal{D}_{1:M}^{(q) \, t}} \left[ \norm{ \grad_{\mathbf{x}} \left[ \mathbf{u}_{t} \bar{\pmb{\ell}} \left( \mathbf{x}_{t} \right) \right] }^{2} \right] \\
                    & \le  \mathbf{u}_{1}^{\top} \bar{\pmb{\ell}}\left( \mathbf{x}_{1} \right) - \mathbb{E}_{\mathcal{D}_{1:M}^{(q) \, T_{\mathrm{iter}}}} \left[ \mathbf{u}_{T_{\mathrm{iter}} + 1}^{\top} \bar{\pmb{\ell}}\left( \mathbf{x}_{T_{\mathrm{iter}} + 1} \right) \right] + T_{\mathrm{iter}} \left[ 2\delta B \sqrt{M} \vphantom{\frac{1}{1}} + \frac{\alpha^{2} \widetilde{\sigma}^{2} \widetilde{S}}{2} \left( \delta \sqrt{M} + 1 \right) \right].
                \end{aligned}
            \end{equation}

            Note that \(\mathbf{u}_{ti} \ge 0 ~ \forall t \in \{1, \ldots, T_{\mathrm{iter}}, i \in \{1, \ldots, M\}\), and \(\bar{\pmb{\ell}}_{i} \ge 0\) due to the non-negativity of the loss function \(\ell\) assumed in \eqref{eq:validation_loss}. Hence, we can upper-bound further the right-hand side to:
            \begin{equation}
                \begin{aligned}[b]
                    & \alpha \left[ 1 - \frac{\alpha \widetilde{S}}{2} \left( \delta \sqrt{M} + 1 \right) \right] \sum_{t=1}^{T_{\mathrm{iter}}} \mathbb{E}_{\mathcal{D}_{1:M}^{(q) \, t}} \left[ \norm{ \grad_{\mathbf{x}} \left[ \mathbf{u}_{t} \bar{\pmb{\ell}} \left( \mathbf{x}_{t} \right) \right] }^{2} \right] \\
                    & \le  \mathbf{u}_{1}^{\top} \bar{\pmb{\ell}}\left( \mathbf{x}_{1} \right) + T_{\mathrm{iter}} \left[ 2\delta B \sqrt{M} \vphantom{\frac{1}{1}} + \frac{\alpha^{2} \widetilde{\sigma}^{2} \widetilde{S}}{2} \left( \delta \sqrt{M} + 1 \right) \right].
                \end{aligned}
            \end{equation}

            If the learning rate \(\alpha\) is selected such that: \(1 - \nicefrac{\alpha \widetilde{S}}{2} \left( \delta \sqrt{M} + 1 \right) > 0 \), then dividing both sides by \(\alpha T_{\mathrm{iter}} \left[ 1 - \nicefrac{\alpha \widetilde{S}}{2} \times \left( \delta \sqrt{M} + 1 \right) \right] \) gives:
            \begin{equation}
                \begin{aligned}[b]
                    \frac{1}{T_{\mathrm{iter}}} \sum_{t=1}^{T_{\mathrm{iter}}} \mathbb{E}_{\mathcal{D}_{1:M}^{(q) \, t}} \left[ \norm{ \grad_{\mathbf{x}} \left[ \mathbf{u}_{t} \bar{\pmb{\ell}} \left( \mathbf{x}_{t} \right) \right] }^{2} \right] & \le \frac{2 \mathbf{u}_{1}^{\top} \bar{\pmb{\ell}}\left( \mathbf{x}_{1} \right) + T_{\mathrm{iter}} \left[ 4 \delta B \sqrt{M} + \alpha^{2} \widetilde{\sigma}^{2} \widetilde{S} \left( \delta \sqrt{M} + 1 \right) \right]}{\alpha T_{\mathrm{iter}} \left[ 2 - \alpha \widetilde{S} \left( \delta \sqrt{M} + 1 \right) \right]}.
                \end{aligned}
                \label{eq:proof_77}
            \end{equation}

            Next, we use a similar trick in the SVRG paper~\citep{johnson2013accelerating} to make the left-hand side term useful. The idea is to output some randomly chosen \(\mathbf{x}_{t}\) rather than outputing \(\mathbf{x}_{T_{\mathrm{iter}}}\). For simplicity, let \(\mathbf{z} = \mathbf{x}_{t}\) with a uniform probability for \(t \in \{1, \ldots, T_{\mathrm{iter}}\}\). The expectation of the gradient norm in this case can be expressed as:
            \begin{equation}
                \mathbb{E}_{\mathbf{z} \sim \{ \mathbf{x}_{t} \}_{t = 1}^{T_{\mathrm{iter}}} } \left[ \mathbb{E}_{\mathcal{D}_{1:M}^{(q) \, t}} \left[ \left\Vert \grad_{\mathbf{x}} \left[ \mathbf{u}_{t} \bar{\pmb{\ell}} \left( \mathbf{x}_{t} \right) \right] \right\Vert^{2} \right] \right] = \frac{1}{T_{\mathrm{iter}}} \sum_{t=1}^{T_{\mathrm{iter}}} \mathbb{E}_{\mathcal{D}_{1:M}^{(q) \, t}} \left[ \left\Vert \grad_{\mathbf{x}} \left[ \mathbf{u}_{t} \bar{\pmb{\ell}} \left( \mathbf{x}_{t} \right) \right] \right\Vert^{2} \right].
            \end{equation}
            This combining with \eqref{eq:proof_77} leads to:
            \begin{equation}
                \begin{aligned}[b]
                    \mathbb{E}_{\mathbf{z}  \sim \{ \mathbf{x}_{t} \}_{t = 1}^{T_{\mathrm{iter}}} } \left[ \mathbb{E}_{\mathcal{D}_{1:M}^{(q) \, t}} \left[ \left\Vert \grad_{\mathbf{z}} \left[ \mathbf{u}_{t} \bar{\pmb{\ell}} \left( \mathbf{z} \right) \right] \right\Vert^{2} \right] \right] & \le \frac{2 \mathbf{u}_{1}^{\top} \bar{\pmb{\ell}}\left( \mathbf{x}_{1} \right) + T_{\mathrm{iter}} \left[ 4 \delta B \sqrt{M} + \alpha^{2} \widetilde{\sigma}^{2} \widetilde{S} \left( \delta \sqrt{M} + 1 \right) \right]}{\alpha T_{\mathrm{iter}} \left[ 2 - \alpha \widetilde{S} \left( \delta \sqrt{M} + 1 \right) \right]}.
                \end{aligned}
            \end{equation}
        \end{proof}

    \subsection{Miscellaneous lemmas}
    \label{sec:miscellaneous_lemmas}
        % \begin{lemma}
        %     If \(X_{i}, \forall i \in \{1, \ldots, n\}\) are random variables with finite variance: \(\mathrm{Var}\left( X_{i} \right) < +\infty\), then:
        %     \begin{equation*}
        %         \mathrm{Var} \left( \sum_{i=1}^{n} X_{i} \right) \le n \sum_{i=1}^{n} \mathrm{Var} \left( X_{i} \right).
        %     \end{equation*}
        %     \label[lemma]{lemma:variance_sum_inequality}
        % \end{lemma}
        % \begin{proof}
        %     Applying Cauchy-Schwarz inequality to the random quantities \(X_{i} - \mathbb{E} \left[ X_{i} \right] \) yields that:
        %     \begin{equation}
        %         \left( \sum_{i = 1}^{n} X_{i} - \mathbb{E} \left[ X_{i} \right] \right)^{2} \le n \sum_{i = 1}^{n} \left( X_{i} - \mathbb{E} \left[ X_{i} \right] \right)^{2}.
        %     \end{equation}
        %     Taking the expectation over all \(X_{i}\) both sides gives:
        %     \begin{equation}
        %         \mathbb{E} \left[ \left( \sum_{i = 1}^{n} X_{i} - \mathbb{E} \left[ X_{i} \right] \right)^{2} \right] \le n \mathbb{E} \left[ \sum_{i = 1}^{n} \left( X_{i} - \mathbb{E} \left[ X_{i} \right] \right)^{2} \right].
        %     \end{equation}
        %     Using the definition of variance to the above inequality completes the proof.
        % \end{proof}

        \begin{lemma}
            \label[lemma]{lemma:eigen_value_shift}
            If \(\lambda\) is an eigenvalue of matrix \(\mathbf{A}\), then \(\lambda - 1\) is an eigenvalue of matrix \(\mathbf{A} - \mathbf{I}\), where \(\mathbf{I}\) is the identity matrix.
        \end{lemma}
        \begin{proof}
            According to the definition of eigenvalue, \(\lambda\) is an eigenvalue of \(\mathbf{A}\) if:
            \begin{equation*}
                \det \left( \mathbf{A} - \lambda \mathbf{I} \right) = 0.
            \end{equation*}
            Rewriting the above equation gives:
            \begin{equation}
                \det \left( \left(\mathbf{A} - \mathbf{I} \right) - \left(\lambda - 1 \right) \mathbf{I} \right) = 0.
            \end{equation}
            Hence, \(\lambda - 1\) is an eigenvalue of \(\mathbf{A} - \mathbf{I}\).
        \end{proof}

        \begin{lemma}[Adpated from \url{https://math.stackexchange.com/a/4303207/274798}] \ \\
            \label[lemma]{lemma:lipschitz_to_bounded_grad}
            If a function \(f: \mathbb{R}^{n} \to \mathbb{R}^{m}\), where \(n, m \in \mathbb{N}\), is differentiable and \(L\)-Lipschitz, then its gradient norm is bounded by \(L\).
        \end{lemma}
        \begin{proof}
            According to the definition of vector or matrix norm:
            \begin{equation}
                \begin{split}
                    \norm{\grad f(x)} & = \sup_{v} \frac{\norm{\grad f(x) \, v}}{\norm{v}}, \forall x, v \in \mathbb{R}^{n}: \norm{v} \neq 0\\
                    & = \sup_{v} \lim_{\lambda \to 0} \frac{\abs{\lambda} \norm{\grad f(x) \, v}}{\abs{\lambda} \norm{v}}, ~\lambda \in \mathbb{R}, \lambda \neq 0\\
                    & = \sup_{v} \lim_{\lambda \to 0} \frac{ \norm{\grad f(x) \, (\lambda v)}}{\norm{\lambda v}}.
                \end{split}
            \end{equation}
            Applying the triangle inequality gives:
            \begin{equation}
                \begin{split}
                    \norm{\grad f(x)} & \le \sup_{v} \lim_{\lambda \to 0} \frac{\norm{f(x + \lambda v) - f(x) - \grad f(x) \, (\lambda v)}}{\norm{\lambda v}} + \frac{\norm{f(x + \lambda v) - f(x)}}{\norm{\lambda v}}\\
                    & \le \sup_{v} \lim_{\lambda \to 0} \frac{\norm{f(x + \lambda v) - f(x) - \grad f(x) \, (\lambda v)}}{\norm{\lambda v}} + \frac{L \norm{x + \lambda v - x}}{\norm{\lambda v}} ~\text{(\(f\) is \(L\)-Lipschitz)}\\
                    & \le L.
                \end{split}
            \end{equation}
            The first term equals to 0 since it corresponds to the definition of Fr\'{e}chet derivative.
        \end{proof}

        \begin{lemma}[Adapted from \url{https://math.stackexchange.com/a/2193914/274798}]\ \\
            If \(A \in \mathbb{R}^{n \times n}\) is a positive definite symmetric matrix, then its largest eigenvalue is
            \begin{equation}
                \lambda_{\mathrm{max}} = \max \qty{\frac{\norm{A \, x}}{\norm{x}}: x \neq 0}.
            \end{equation}
            \label[lemma]{lemma:norm_max_eigenvalue}
        \end{lemma}
        \begin{proof}
            According to the spectral theorem, if \(A\) is a real and symmetric matrix, then there exists an orthonormal basis consisting of eigenvectors of \(A\), denoted as \(v_{1}, v_{2}, \ldots, v_{n}\). Without loss of generality, let's assume that the corresponding eigenvalues are: \(\lambda_{1}, \lambda_{2}, \ldots, \lambda_{n}\).

            For any non-zero vector \(x\), we can represent it in this orthonormal basis as:
            \begin{equation}
                x = \sum_{i = 1}^{n} c_{i} v_{i}, ~c_{i} \in \mathbb{R}.
            \end{equation}

            Then, the function of interest can be calculated as follows:
            \begin{equation}
                \begin{split}
                    \frac{\norm{A \, x}}{\norm{x}} & = \frac{\norm{\sum_{i = 1}^{n} c_{i} \lambda_{i} v_{i}}}{\sqrt{\sum_{i = 1}^{n} c_{i}^{2}}}\\
                    & =\sqrt{\frac{\sum_{i = 1}^{n} c_{i}^{2} \lambda_{i}^{2}}{\sum_{i = 1}^{n} c_{i}^{2}}} \quad \text{(\(\{v_{i}\}_{i = 1}^{n}\) is an orthonormal basis)}\\
                    & \le \max_{i \in \{1, \ldots, n\}} \lambda_{i}.
                \end{split}
            \end{equation}
            The equality occurs when the vector \(x\) equals to the corresponding eigenvector of the eigenvalue \(\max_{i \in \{1, \ldots, n\}} \lambda_{i} \).
        \end{proof}
\newpage
\section{Additional results when calculating with full matrix \texorpdfstring{\(\mathbf{V}_{t}\)}{Vt}}
\label{apdx:full_matrix_results}
    In this \appendixname, we provide additional results of prediction accuracy on tasks formed from the evaluation sets of Omniglot and mini-ImageNet. These results are based on the naive implementation of iLQR where the auxiliary Hessian matrix \(\mathbf{V}_{t}\) of the \emph{cost-to-go} in \eqref{eq:cost_to_go} is exact without any approximation. We note that due to the quadratic complexity of running time, we cannot train the one on mini-ImageNet until convergence. The running time using the full matrix \(\mathbf{V}_{t}\) is 160 GPU-hour for Omniglot and 184 GPU-hour for mini-ImageNet.
    
    \begin{figure}[ht]
        \centering
        \begin{subfigure}[b]{0.45 \linewidth}
            \centering
            \includegraphics[width = \linewidth]{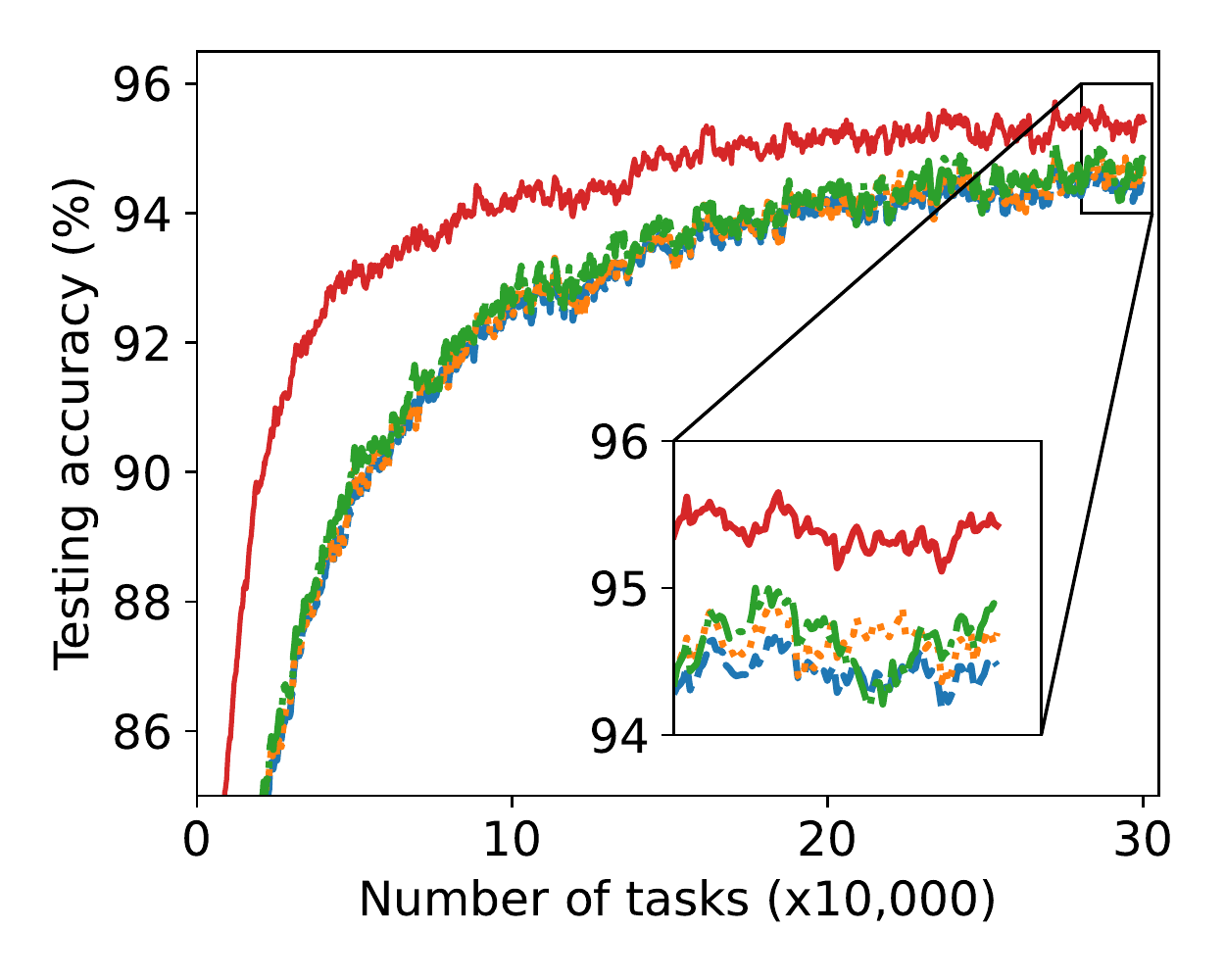}
            \caption{MAML on Omniglot}
        \end{subfigure}
        \hspace{3em}
        % \begin{subfigure}[b]{0.45 \linewidth}
        %     \centering
        %     \includegraphics[width = \linewidth]{img/protonet_omniglot_test_accuracy_square.pdf}
        %     \caption{Protonet on Omniglot}
        % \end{subfigure}
        \begin{subfigure}[b]{0.45 \linewidth}
            \centering
            \includegraphics[width = \linewidth]{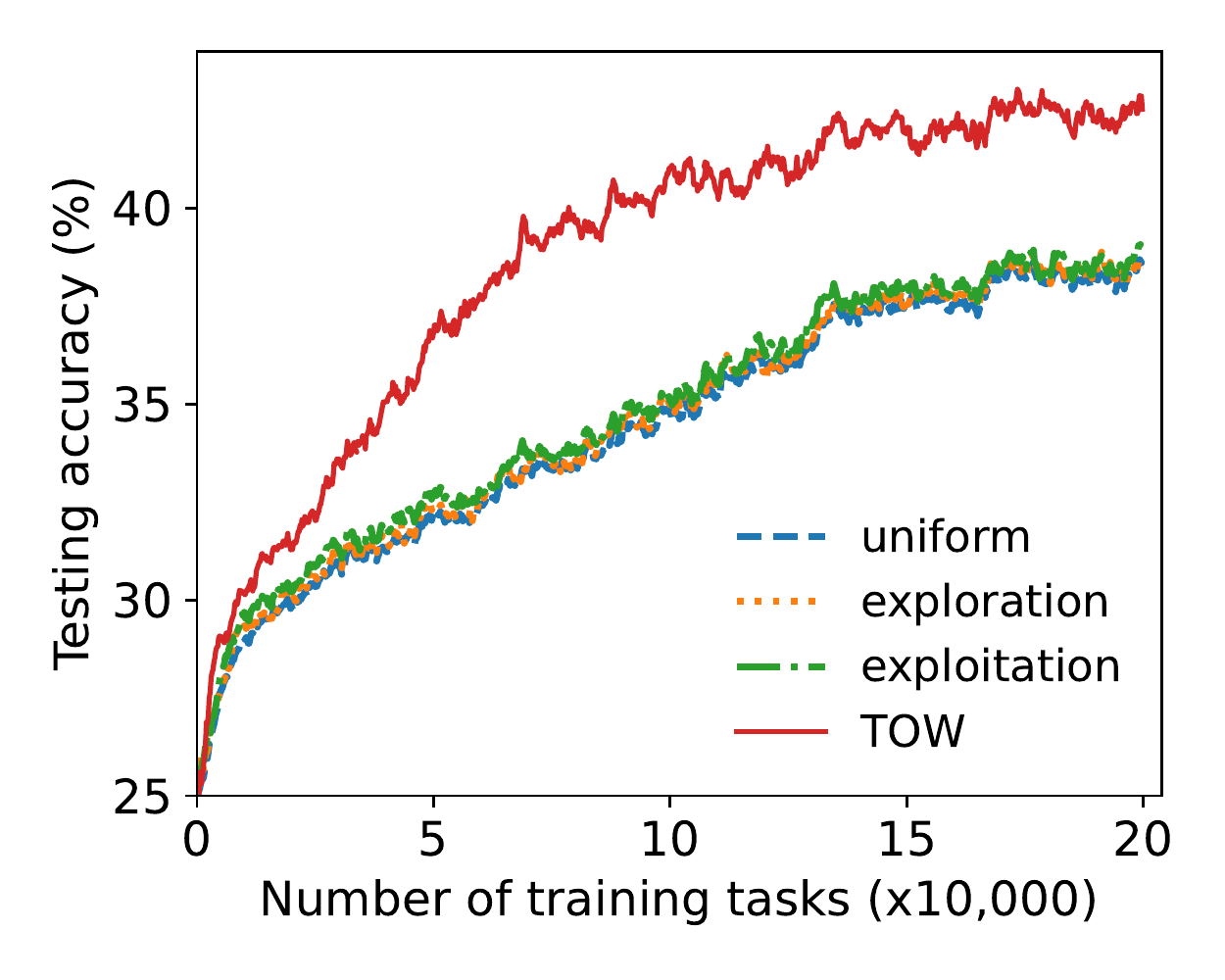}
            \caption{MAML on mini-ImageNet}
        \end{subfigure}
        % \hspace{3em}
        % \begin{subfigure}[b]{0.45 \linewidth}
        %     \centering
        %     \includegraphics[width = \linewidth, draft]{img/maml_miniImageNet_validation_accuracy_full.pdf}
        %     \caption{Protonet on mini-ImageNet}
        % \end{subfigure}
        \caption{Additional results of prediction accuracy on 100 random validation tasks using MAML when the matrix \(\mathbf{V}_{t}\) is exact without any approximation.}
        \label{fig:additional_results_full_matrix_V}
    \end{figure}
\newpage
\section{Visualisation of weight values}
\label{apdx:weight_visualisation}
    To further understand how the weight \(\mathbf{u}_{t}\) varies along the training process, we conduct a study to monitor the weights \(\mathbf{u}_{t}\) of a set of pre-defined Omniglot tasks. In the first setting, we fix \(T = 5\) mini-batches of tasks, each consisting of \(M = 5\) tasks formed from the \emph{same} alphabet. In the second case, the configuration is similar, but each mini-batch contains \(M = 5\) tasks formed from 5 \emph{different} alphabets. For each case, we also run with tasks drawn from training and testing sets. Our desire is to observe how the weight changes and whether there is a difference between training and testing tasks. \revise{We plot the weight for a single task belonging to the \say{controlled} 5-task mini-batch of interest in \cref{fig:weight_visualisation}. The results} show the variation of the weights of the tasks of interest with a common trend among all the tasks considered, in which the weights are large at the beginning and gradually reduce when training progresses. This is, indeed, expected since most tasks are unfamiliar to the model at the early state, and gradually becomes more familiar. In addition, we observe that the weights for testing tasks are slightly larger than the ones for training tasks.
    
    \begin{figure}[ht]
        \centering
        \begin{tikzpicture}
            \pgfplotstableread[col sep=comma, header=false]{results/omniglot_weight_train_same_alphabet.csv} \tableTrainSame
            \pgfplotstableread[col sep=comma, header=false]{results/omniglot_weight_train_different_alphabet.csv} \tableTrainDifferent
            \pgfplotstableread[col sep=comma, header=false]{results/omniglot_weight_test_same_alphabet.csv} \tableTestSame
            \pgfplotstableread[col sep=comma, header=false]{results/omniglot_weight_test_different_alphabet.csv} \tableTestDifferent

            \begin{axis}[
                height = 0.28125\linewidth,
                width = 0.375\linewidth,
                xlabel={\textnumero~of meta-updates (\(\times\)10,000)},
                xlabel style={font=\small},
                xticklabel style = {font=\small},
                ylabel={Weight value},
                ylabel style={font=\small, yshift=-0.5em},
                yticklabel style = {font=\small},
                % ymin=80,
                % ymax=98,
                % % restrict x to domain=0:300,
                scale only axis,
                legend entries={train-same, train-different, test-same, test-different},
                legend style={draw=none, font=\footnotesize, legend columns=2, /tikz/every even column/.append style={column sep=0.25em}},
                legend image post style={scale=1},
                legend cell align={left},
                legend pos=south east
            ]
                \addplot[mark=none, MidnightBlue, very thick] table[x expr=0.5 * (\coordindex + 1), y expr=\thisrowno{0}]{\tableTrainSame};
                \addplot[mark=none, BurntOrange, very thick, dashed] table[x expr=0.5 * (\coordindex + 1), y expr=\thisrowno{0}]{\tableTrainDifferent};
                \addplot[mark=none, ForestGreen, very thick, dotted] table[x expr=0.5 * (\coordindex + 1), y expr=\thisrowno{1}]{\tableTestSame};
                \addplot[mark=none, BrickRed, very thick, dashdotted] table[x expr=0.5 * (\coordindex + 1), y expr=\thisrowno{0}]{\tableTestDifferent};
            \end{axis}
        \end{tikzpicture}
        \caption{Visualisation of the weight values where tasks are drawn from the same Omniglot alphabet (either training or testing set); the notation \emph{same} means that all tasks in a mini-batch are formed from one alphabet, while \emph{different} indicates the mini-batch consists of tasks formed from different alphabets.}
        \label{fig:weight_visualisation}
    \end{figure}
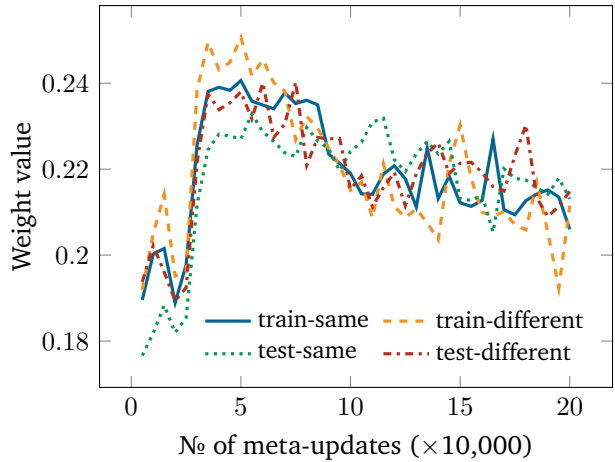

    We also conduct the same ablation study for mini-ImageNet tasks. Since in mini-ImageNet, we do not have any information regarding to the hierarchical structure of classes, we carry out the experiment on training and testing tasks only. One could, of course, utilise the word-net structure to categorise classes into \say{alphabet} as Omniglot, but for the sake of simplicity, we proceed without including such information. \figureautorefname~\ref{fig:weight_visualisation_miniImageNet} shows the weight values of some tasks at each training checkpoint, which also has a similar but noisier trend as the ones in Omniglot.
    
    \begin{figure}[ht]
        \centering
        \begin{tikzpicture}
            \pgfplotstableread[col sep=comma, header=false]{results/miniImageNet_weight_train.csv} \tableTrain
            \pgfplotstableread[col sep=comma, header=false]{results/miniImageNet_weight_test.csv} \tableTest

            \begin{axis}[
                height = 0.28125\linewidth,
                width = 0.375\linewidth,
                xlabel={\textnumero~of meta-updates (\(\times\)10,000)},
                xlabel style={font=\small},
                xticklabel style = {font=\small},
                ylabel={Weight value},
                ylabel style={font=\small, yshift=-0.5em},
                yticklabel style = {font=\small},
                % ymin=80,
                % ymax=98,
                % restrict x to domain=0:30,
                scale only axis,
                legend entries={train, test},
                legend style={draw=none, font=\footnotesize, legend columns=1, /tikz/every even column/.append style={column sep=0.25em}},
                legend image post style={scale=1},
                legend cell align={left},
                legend pos=south west
            ]
                \addplot[mark=none, BrickRed, very thick] table[x expr=(\coordindex + 1), y expr=\thisrowno{2}]{\tableTrain};
                \addplot[mark=none, ForestGreen, very thick, dashed] table[x expr=(\coordindex + 1), y expr=\thisrowno{1}]{\tableTest};
            \end{axis}
        \end{tikzpicture}
        \caption{Visualisation of the weight values associated with mini-ImageNet tasks.}
        \label{fig:weight_visualisation_miniImageNet}
    \end{figure}
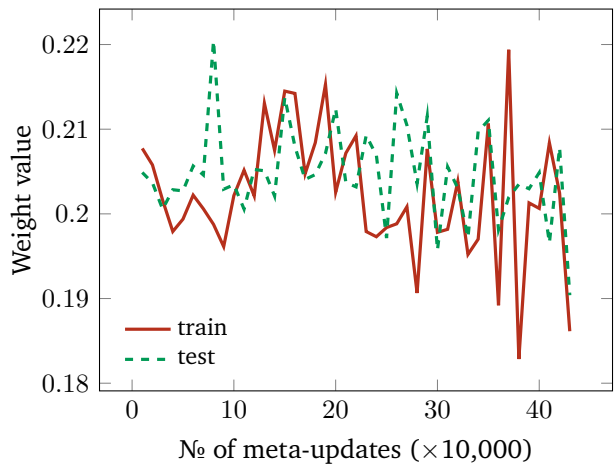
\newpage
\section{Ablation study}
\label{sec:ablation_study}
    To understand more about the effect of the number of iterations used in iLQR, we carry out an ablation study by analysing the validation accuracy on two different settings: one with 2 iterations and the other with 10 iterations. The ablation study is carried out on mini-ImageNet dataset using the same 4 convolutional module neural network with the same configuration parameters presented in \cref{sec:nway_kshot_classification}. The qualitative result between the two settings in \cref{fig:ablation_study_varying_number_iterations_iLQR} agrees with our intuition where the larger the number of iterations in iLQR, the more accurate weighting and hence, the larger validation accuracy. However, the trade-off is the significant increasing in terms of training time: approximately 6 times slower to train than the one with 2 iterations.

    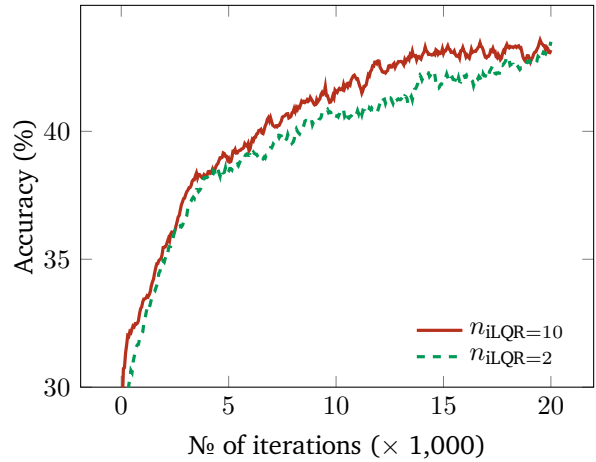
\begin{figure}[th]
            \centering
            % \begin{subfigure}{0.5 \linewidth}
                \centering
                \begin{tikzpicture}
                    \pgfplotstableread[col sep=comma, header=true]{results/maml_miniImageNet_CNN_val_accuracy.csv} \myTable
                    \pgfplotstableread[col sep=comma, header=true]{results/maml_miniImageNet_CNN_val_accuracy_10.csv} \secondTable
    
                    \begin{axis}[
                        height = 0.28125\linewidth,
                        width = 0.375\linewidth,
                        xlabel={\textnumero~of iterations (\(\times\) 1,000)},
                        xlabel style={font=\small},
                        xticklabel style = {font=\small},
                        ylabel={Accuracy (\%)},
                        ylabel style={font=\small, yshift=-0.5em},
                        yticklabel style = {font=\small},
                        ymin=30,
                        legend entries={\(n_{\text{iLQR} = 10}\), \(n_{\text{iLQR} = 2}\)},
                        legend style={draw=none, font=\small},
                        legend image post style={scale=1},
                        legend cell align={left},
                        legend pos=south east,
                        restrict x to domain=0:20,
                        scale only axis
                    ]
                        \addplot[mark=none, BrickRed, solid, very thick] table[x expr=0.05 * (\coordindex + 1), y={Testing_accuracy_tow}]{\secondTable};
                        \addplot[mark=none, ForestGreen, very thick, dashed] table[x expr=0.05 * (\coordindex + 1), y={Testing_accuracy_tow}]{\myTable};
                    \end{axis}
                \end{tikzpicture}
                \caption{Ablation study with two different numbers of iterations in iLQR is plotted with exponential moving average (smoothing factor is 0.1).}
                \label{fig:ablation_study_varying_number_iterations_iLQR}
            % \end{subfigure}
    \end{figure}
\newpage
\section{Derivation of iLQR}
\label{sec:ilqr_derivation}
    This section presents the derivation for one iteration of iLQR adopted from the original iLQR papers~\citep{todorov2005generalized,tassa2012synthesis}. One difference is at the backtracking line-search where we employ a similar acceptance criterion as in DDP.

    For brevity, the finite-horizon discrete-time optimal control problem in \eqref{eq:TO_problem} is restated:
    \begin{equation}
        \begin{aligned}[b]
            \mathbf{x}_{t + 1} & = \arg\min_{\mathbf{x}_{t}} \mathbf{u}_{t}^{\top} \pmb{\ell}(\mathbf{x}_{t}) & \text{s.t. } \mathbf{u}_{t} & = \text{weighting criterion}.
        \end{aligned}
        \tag{\ref{eq:meta_learning_objective}}
    \end{equation}
    Let \(Q_{t}\) be the \emph{cost-to-go} defined as:
    \begin{equation}
        Q_{t}(\mathbf{x}_{t:T}, \mathbf{u}_{t:T}) = \sum_{j=t}^{T} c(\mathbf{x}_{j}, \mathbf{u}_{j}),
        \label{eq:cost_to_go}
    \end{equation}
    and \(V_{t}\) be the \emph{value} function at time step \(t\) which is the cost-to-go given the local optimal action sequence:
    \begin{equation}
        V_{t}(\mathbf{x}_{t}) = \min_{\mathbf{u}_{t:T}} Q_{t}( \mathbf{x}_{t:T}, \mathbf{u}_{t:T}).
        \label{eq:value_function}
    \end{equation}

    According to the Dynamic Programming Principle, the objective function which is the minimisation over an entire sequence of actions can be reduced to a sequence of minimisation over a single action, proceeding backwards in time:
    \begin{equation}
        V_{t}(\mathbf{x}_{t}) = \min_{\mathbf{u}_{t}} \left[ c(\mathbf{x}_{t}, \mathbf{u}_{t}) + V\left( \mathbf{x}_{t + 1} \right) \right] = \min_{\mathbf{u}_{t}} \left[ c(\mathbf{x}_{t}, \mathbf{u}_{t}) + V\left( f(\mathbf{x}_{t}, \mathbf{u}_{t}) \right) \right].
        \label{eq:dynamic_programming}
    \end{equation}
    
    By applying the Taylor's series to the first order for the state-transition dynamics \(f\) and to the second order for the cost function \(c\) about a nominal trajectory \(\{\hat{\mathbf{x}}_{t}, \hat{\mathbf{u}}_{t}\}\), the term inside the minimisation in \eqref{eq:dynamic_programming} can be approximated to:
    \begin{equation}
        c(\mathbf{x}_{t}, \mathbf{u}_{t}) + V\left( f(\mathbf{x}_{t}, \mathbf{u}_{t}) \right) \approx \frac{1}{2} \begin{bmatrix}
            1\\
            \delta\mathbf{x}_{t}\\
            \delta\mathbf{u}_{t}
        \end{bmatrix}^{\top} \begin{bmatrix}
            0 & \mathbf{q}_{\mathbf{x}_{t}}^{\top} & \mathbf{q}_{\mathbf{u}_{t}}^{\top}\\
            \mathbf{q}_{\mathbf{x}_{t}} & \mathbf{Q}_{\mathbf{x}_{t}, \mathbf{x}_{t}} & \mathbf{Q}_{\mathbf{x}_{t}, \mathbf{u}_{t}}\\
            \mathbf{q}_{\mathbf{u}_{t}} & \mathbf{Q}_{\mathbf{u}_{t}, \mathbf{x}_{t}} & \mathbf{Q}_{\mathbf{u}_{t}, \mathbf{u}_{t}}
        \end{bmatrix} \begin{bmatrix}
            1\\
            \delta\mathbf{x}_{t}\\
            \delta\mathbf{u}_{t}
        \end{bmatrix},
        \label{eq:value_function_approximation}
    \end{equation}
    where \(\delta \mathbf{x}_{t} = \mathbf{x}_{t} - \hat{\mathbf{x}}_{t}\), \(\delta \mathbf{u}_{t} = \mathbf{u}_{t} - \hat{\mathbf{u}}_{t}\), and
    \begin{equation}
        \begin{aligned}[b]
            \mathbf{q}_{\mathbf{x}_{t}} & = \mathbf{c}_{\mathbf{x}_{t}} + \mathbf{F}_{\mathbf{x}_{t}}^{\top} \mathbf{v}_{t + 1} & \mathbf{q}_{\mathbf{u}_{t}} & = \mathbf{c}_{\mathbf{u}_{t}} + \mathbf{F}_{\mathbf{u}_{t}}^{\top} \mathbf{v}_{t + 1} \\
            \mathbf{Q}_{\mathbf{x}_{t}, \mathbf{x}_{t}} & = \mathbf{C}_{\mathbf{x}_{t}, \mathbf{x}_{t}} + \mathbf{F}_{\mathbf{x}_{t}}^{\top} \mathbf{V}_{t + 1} \mathbf{F}_{\mathbf{x}_{t}} & \mathbf{Q}_{\mathbf{u}_{t}, \mathbf{u}_{t}} & = \mathbf{C}_{\mathbf{u}_{t}, \mathbf{u}_{t}} + \mathbf{F}_{\mathbf{u}_{t}}^{\top} \mathbf{V}_{t + 1} \mathbf{F}_{\mathbf{u}_{t}} \\
            \mathbf{Q}_{\mathbf{x}_{t}, \mathbf{u}_{t}} & = \mathbf{Q}_{\mathbf{u}_{t}, \mathbf{x}_{t}}^{\top} = \mathbf{C}_{\mathbf{x}_{t}, \mathbf{u}_{t}} + \mathbf{F}_{\mathbf{x}_{t}}^{\top} \mathbf{V}_{t + 1} \mathbf{F}_{\mathbf{u}_{t}},
        \end{aligned}
    \end{equation}
    with \(\mathbf{F}_{\mathbf{x}_{t}}\) and \(\mathbf{F}_{\mathbf{u}_{t}}\) are the first derivatives of the state-transtition dynamics \(f\) w.r.t. the variable in the subscripts, \(\mathbf{c}_{(.)}\) and \(\mathbf{C}_{(.)}\) are the first and second derivatives of the cost function \(c\) w.r.t. the variables in the subscripts, and \(\mathbf{v}_{t + 1}\) and \(\mathbf{V}_{t + 1}\) are the first and second derivatives of the value function \(V\) w.r.t. the state \(\mathbf{x}_{t + 1}\).

    Minimising the quadratic problem in \eqref{eq:value_function_approximation} w.r.t. \(\delta\mathbf{u}_{t}\) results in:
    \begin{equation}
        \delta \mathbf{u}_{t}^{*} = \mathbf{K}_{t} \delta\mathbf{x}_{t} + \mathbf{k}_{t},
        \label{eq:action_optimal_t}
    \end{equation}
    where:
    \begin{equation}
        \begin{aligned}[b]
            \mathbf{K}_{t} = -\mathbf{Q}_{\mathbf{u}_{t}, \mathbf{u}_{t}}^{-1} \mathbf{Q}_{\mathbf{u}_{t}, \mathbf{x}_{t}} \qquad \mathbf{k}_{t} = -\mathbf{Q}_{\mathbf{u}_{t}, \mathbf{u}_{t}}^{-1} \mathbf{q}_{\mathbf{u}_{t}},
        \end{aligned}
    \end{equation}
    Substituting the action \(\delta\mathbf{u}_{t}\) obtained in \eqref{eq:action_optimal_t} into \eqref{eq:value_function_approximation} gives a quadratic model:

    \begin{equation}
        V_{t}(\mathbf{x}_{t}) \approx \frac{1}{2} \delta \mathbf{x}_{t}^{\top} \mathbf{V}_{t} \delta \mathbf{x}_{t} + \delta \mathbf{x}_{t}^{\top} \mathbf{v}_{t} + \mathrm{const.},
        \label{eq:value_function_quadratic}
    \end{equation}
    where \(\mathbf{V}_{t}\) and \(\mathbf{v}_{t}\) are the second and first derivatives of the value function \(V_{t}\), which can be expressed as:
    \begin{align}
        \mathbf{V}_{t} & = \mathbf{Q}_{\mathbf{x}_{t}, \mathbf{x}_{t}} + \mathbf{Q}_{\mathbf{x}_{t}, \mathbf{u}_{t}} \mathbf{K}_{t} &
        \mathbf{v}_{t} & = \mathbf{q}_{\mathbf{x}_{t}} + \mathbf{Q}_{\mathbf{x}_{t}, \mathbf{u}_{t}} \mathbf{k}_{t}.
    \end{align}

    Hence, one can recursively calculate the local quadratic model \(\{\mathbf{V}_{t}, \mathbf{v}_{t}\}\) of the value function \(V_{t}\), and the linear controller \(\{\mathbf{K}_{t}, \mathbf{k}_{t}\}\) backward through time. Once it is completed, a new trajectory can be calculated using the forward pass with the general non-linear state-transition dynamics \(f\) as follows:
    \begin{equation}
        \begin{aligned}[b]
            \hat{\mathbf{x}}_{1} & = \mathbf{x}_{1}\\
            \mathbf{u}_{t} & = \mathbf{K}_{t} \left( \mathbf{x}_{t} - \hat{\mathbf{x}}_{t} \right) + \mathbf{k}_{t} + \hat{\mathbf{u}}_{t}\\
            \mathbf{x}_{t + 1} & = f\left( \mathbf{x}_{t}, \mathbf{u}_{t} \right).
        \end{aligned}
        \label{eq:forward_pass}
    \end{equation}

    Although the trajectory obtained in \eqref{eq:forward_pass} converges to a local minimum for the approximate model of the value function \(V_{t}\), it does not guarantee the convergence for general non-linear models such as the one in \eqref{eq:TO_problem}. The reason is that the new trajectory might deviate too far from the nominal trajectory, resulting in a poor Taylor approximation of the true model. To overcome, a backtracking linear-search with parameter \(\varepsilon \in (0, 1]\) is introduced:
    \begin{equation}
        \mathbf{u}_{t} = \mathbf{K}_{t} \left( \mathbf{x}_{t} - \hat{\mathbf{x}}_{t} \right) + \varepsilon \mathbf{k}_{t} + \hat{\mathbf{u}}_{t}.
        \label{eq:action_backtracking_linesearch}
    \end{equation}
    The criterion to accept the trajectory produced in the iteration with backtracking line-search is similar to the one in DDP:
    \begin{equation}
        J(\mathbf{u}_{1:T}) - J(\hat{\mathbf{u}}_{1:T}) < \frac{1}{2} \varepsilon \theta_{1},
    \end{equation}
    where
    \begin{align}
        J(\mathbf{u}_{1:T}) & = \sum_{t=1}^{T} c(\mathbf{x}_{t}, \mathbf{u}_{t}) \label{eq:total_cost}\\
        \theta_{t} & = \theta_{t + 1} - \mathbf{q}_{\mathbf{u}_{t}}^{\top} \mathbf{Q}_{\mathbf{u}_{t}, \mathbf{u}_{t}} \mathbf{q}_{\mathbf{u}_{t}}. \label{eq:theta1}
    \end{align}
    Hence, in the worst case when the new trajectory strays too far from the model's region of validity, then \(\varepsilon \to 0\) and the trajectory is the same as the nominal trajectory. The procedure of one iLQR iteration is outlined in Algorithm~\ref{algm:approximate_lqr} in \appendixname{~\ref{sec:TO_algorithms}}. In addition, the convergence proof for iLQR adapted from DDP~\citep{yakowitz1984computational} is provided in \appendixname~\ref{sec:ilqr_convergence_proof} to complete the analysis.
\newpage
\section{Convergence of iLQR}
\label{sec:ilqr_convergence_proof}
    To prove the convergence of iLQR algorithm, we rely on Theorem 3 in \citep[p. 14]{polak1971computational}, which is re-stated in \cref{thrm:how_to_prove_convergence} in \cref{sec:auxiliary_ilqr}

    \subsection{Auxiliary to prove convergence}
    \label{sec:auxiliary_ilqr}
        \begin{definition}[Algorithm model]
            Given \(a: \mathcal{T} \to \mathcal{T}\) is an algorithm, and \(c: \mathcal{T} \to \mathcal{R}\) is a function used as stopping criterion, the algorithm model can be described as:
            \begin{enumerate}
                \item Compute a \(z_{0} \in \mathcal{T}\).
                \item Set \(i = 0\).
                \item Compute \(a(z_{i})\) \label{algorithm_model:step_i+1}.
                \item Set \(z_{i + 1} = a(z_{i})\).
                \item If \(c(z_{i + 1}) \ge c(z_{i})\), stop;\footnote{A direct test for determining if \(z_{i}\) is desirable may be substituted for the test \(c(z_{i + 1}) \ge c(z_{i})\).} else set \(i = i + 1\) and go to step~\ref{algorithm_model:step_i+1}.
            \end{enumerate}
            \label{def:algorithm_model}
        \end{definition}
        
        \cref{thrm:how_to_prove_convergence} will show what such an algorithm will compute.
        
        \begin{theorem}[\texorpdfstring{\citep[Theorem 3, p. 14]{polak1971computational}}{}]
            \label{thrm:how_to_prove_convergence}
            Suppose that:
            \begin{enumerate}[label=(\roman*)]
                \item \(c(.)\) is either continuous at all non-desirable points \(z \in \mathcal{T}\), or else \(c(z)\) is bounded from below for \(z \in \mathcal{T}\);
                \item for every \(z \in \mathcal{T}\) which is not desirable, there exist an \(\varepsilon(z) > 0\) and a \(\delta(z) < 0\) such that:
                \begin{equation}
                    c(a(z^{\prime})) - c(z^{\prime}) \le \delta(z) < 0, \forall z^{\prime} \in B(z, \varepsilon(z)) = \{z \in \mathcal{T}: \lVert z^{\prime} - z \rVert_{\mathcal{B}} \le \varepsilon(z)\}.
                \end{equation}
            \end{enumerate}
            Then, either the sequence \(\{z_{i}\}\) constructed by algorithm defined in Definition~\ref{def:algorithm_model} is finite and its next to last element is desirable, or else it is infinite and every accumulation point of \(\{z_{i}\}\) is desirable.
        \end{theorem}
    \subsection{Proof of iLQR convergence}
        Since iLQR is an algorithm model as described in \cref{def:algorithm_model}, to prove its convergence property, one needs to assert that the 2 conditions in \cref{thrm:how_to_prove_convergence} are met.
        
        First condition is satisfied since the cost \(J(\mathbf{u}_{1:T})\) is continuous, and twice differentiable. We will prove that iLQR satisfies the second condition of \cref{thrm:how_to_prove_convergence}.

        From \eqref{eq:action_backtracking_linesearch} and \eqref{eq:theta1}, the sequence of actions obtained from iLQR \(\mathbf{u}_{1:T}\) and \(\theta_{1}\) depend on the nominal trajectory. Hence, we explicitly denote them as \(\mathbf{u}(\varepsilon, \hat{\mathbf{u}})\) and \(\theta_{1} (\hat{\mathbf{u}})\). In addition, since the loss function \(\ell(.)\) is assumed to be twice differentiable, the cost function and state-transition dynamics are continuous. Hence, both \(\mathbf{u}(\varepsilon, \hat{\mathbf{u}})\) and \(\theta_{1} (\hat{\mathbf{u}})\) are also continuous.

        The convergence proof for iLQR is presented in \cref{sec:ilqr_convergence_proof_subsection}. The proof requires some auxiliary lemmas in \cref{sec:auxiliary_lemmas_ilqr}.

        \subsubsection{Auxiliary lemmas}
        \label{sec:auxiliary_lemmas_ilqr}
            \begin{lemma}
                If \(\mathbf{u}(\varepsilon, \bar{\mathbf{u}})\) is a non-stationary sequence of actions, then \(\theta_{1} < 0\).
                \label{lmm:theta_negative}
            \end{lemma}
            \begin{proof}
                It is apparent from the derivation of iLQR that \(\mathbf{q}_{\mathbf{u}_{t}} = 0, \forall t \in \{1, \ldots, T\}\) if and only if \(\hat{\mathbf{u}}\) is a stationary sequence of actions. Hence, by negation, if \(\hat{\mathbf{u}}\) is a non-stationary sequence of actions, \(\mathbf{q}_{\mathbf{u}_{t}} \neq 0\) for some time steps \(t \in \{1, \ldots, T\}\).
                
                Also note that, if \(\mathbf{u}(\varepsilon, \bar{\mathbf{u}})\) is non-stationary, then \(\hat{\mathbf{u}}\) is also non-stationary. In addition, if the Hessian matrix w.r.t. action \(\mathbf{u}\) is assumed to be positive-definite (PSD), then from the construction, \(\mathbf{Q}_{\mathbf{u}_{t}, \mathbf{u}_{t}}\) and its inverse are also PSD.
                
                Th PSD of \(\mathbf{Q}_{\mathbf{u}_{t}, \mathbf{u}_{t}}^{-1}\) combining with \eqref{eq:theta1} leads to the fact that \(\theta_{1}(\hat{\mathbf{u}}) < 0\) for any non-stationary sequence of actions \(\hat{\mathbf{u}}\).
            \end{proof}
    
            \begin{lemma}
                If \(\mathbf{u}(\varepsilon, \bar{\mathbf{u}})\) is a non-stationary sequence of actions, then:
                \begin{equation}
                    J(\mathbf{u}(\varepsilon, \bar{\mathbf{u}})) - J(\bar{\mathbf{u}}) = \varepsilon \theta_{1}(\bar{\mathbf{u}}) + \mathcal{O}(\varepsilon^{2}).
                    \label{eq:A.2}
                \end{equation}
                \label{lmm:A.2}
            \end{lemma}
            \begin{proof}
                If we define \(\Delta Q_{t}\) as the resulting incremental change in the cost-to-go \(Q\) from the perturbation \(\Delta \mathbf{u}_{t}\) using Taylor expansion, then:
                \begin{equation}
                    \Delta Q_{t} = Q_{t}(\mathbf{x}_{t}, \mathbf{u}_{t} + \Delta \mathbf{u}_{t}) - Q_{t}(\mathbf{x}_{t}, \mathbf{u}_{t}) = \mathbf{q}_{\mathbf{u}_{t}}^{\top} \Delta \mathbf{u}_{t} + \mathcal{O}(\lVert \Delta \mathbf{u}_{t} \rVert^{2}).
                    \label{eq:A.9}
                \end{equation}
                Note that:
                \begin{equation}
                    \begin{aligned}[b]
                        J(\mathbf{u}(\varepsilon, \hat{\mathbf{u}})) - J(\hat{\mathbf{u}}) & = J(\mathbf{u}_{1:N} + \Delta \mathbf{u}_{1:N}) - J(\mathbf{u}_{1:N}) \\
                        & = \sum_{t=1}^{T} \Delta Q_{t} = \sum_{t=1}^{N} \mathbf{q}_{\mathbf{u}_{t}}^{\top} \Delta \mathbf{u}_{t} + \mathcal{O}\left( \lVert \Delta \mathbf{u}_{1:N} \rVert^{2} \right).
                    \end{aligned}
                    \label{eq:A.9a}
                \end{equation}
                Taking the derivative w.r.t. \(\varepsilon\) gives:
                \begin{equation}
                    \pdv{J(\mathbf{u}(\varepsilon))}{\varepsilon} = \sum_{t=1}^{T} \mathbf{q}_{\mathbf{u}_{t}}^{\top} \dv{\mathbf{u}_{t}}{\varepsilon}.
                    \label{eq:A.10}
                \end{equation}
                From \eqref{eq:action_backtracking_linesearch}, \(\mathbf{u}_{t}\) is a linear function w.r.t. \(\varepsilon\), resulting in:
                \begin{equation}
                    \dv{\mathbf{u}_{t}}{\varepsilon} = \mathbf{k}_{t} = - \mathbf{Q}_{\mathbf{u}_{t}, \mathbf{u}_{t}} \mathbf{q}_{\mathbf{u}_{t}}.
                \end{equation}
                Hence:
                \begin{equation}
                    \pdv{J(\mathbf{u}(\varepsilon))}{\varepsilon} = - \sum_{t=1}^{T} \mathbf{q}_{\mathbf{u}_{t}}^{\top} \mathbf{Q}_{\mathbf{u}_{t}, \mathbf{u}_{t}} \mathbf{q}_{\mathbf{u}_{t}} = \theta_{1}(\hat{\mathbf{u}}),
                \end{equation}
                which implies \eqref{eq:A.2}.
            \end{proof}
    
            \begin{lemma}
                If \(\mathbf{u}(\varepsilon, \bar{\mathbf{u}})\) is a non-stationary sequence of actions, then there exists \(\varepsilon_{1}\) such that:
                \begin{equation}
                    J(\mathbf{u}(\varepsilon, \bar{\mathbf{u}})) - J(\hat{\mathbf{u}}) \le \eta \varepsilon \theta_{1} (\hat{\mathbf{u}}), \forall \varepsilon \in [0, \varepsilon_{1}], \forall \eta \in (0.5, 1).
                \end{equation}
                \label{lmm:A.4}
            \end{lemma}
            \begin{proof}
                The result in Lemma~\ref{lmm:A.2} can be re-written as:
                \begin{equation}
                    J(\mathbf{u}(\varepsilon, \hat{\mathbf{u}})) - J(\hat{\mathbf{u}}) = \eta \varepsilon \theta_{1}(\hat{\mathbf{u}}) + (1 - \eta) \varepsilon \theta_{1}(\hat{\mathbf{u}}) + \mathcal{O}(\varepsilon^{2}), \forall \eta \in (0.5, 1).
                \end{equation}
                Hence, there must exist a value \(\varepsilon = \varepsilon_{1} \ge 0\) such that: \((1 - \eta)  \theta_{1}(\hat{\mathbf{u}}) + \mathcal{O}(\varepsilon^{2}) < 0\). This leads to:
                \begin{equation}
                    J(\mathbf{u}(\varepsilon, \hat{\mathbf{u}})) - J(\hat{\mathbf{u}}) \le \eta \varepsilon_{1} \theta_{1}(\hat{\mathbf{u}}) \le \eta \varepsilon \theta_{1}(\hat{\mathbf{u}}), \forall \varepsilon \in [0, \varepsilon_{1}].
                \end{equation}
                Note that the second inequality holds due to the negativity of \(\theta_{1}\) proved in Lemma~\ref{lmm:theta_negative}. This completes the proof.
            \end{proof}

            \begin{lemma}
                Given that \(\theta_{1}\) and any action sequence \(\mathbf{u}\) are continuous, and for any \(\delta > 0\) such that \(\lVert \hat{\mathbf{u}} - \mathbf{u} \rVert < \delta\), there exists \(\varepsilon_{2} > 0\) such that:
                \begin{equation*}
                    \theta_{1}(\mathbf{u}) - \varepsilon_{2} < \theta_{1}(\hat{\mathbf{u}}) < \theta_{1}(\mathbf{u}) + \varepsilon_{2}.
                \end{equation*}
                \label{lmm:A.3}
            \end{lemma}
            \begin{proof}
                Since \(\theta_{1}\) is assumed to be continuous, and the variable \(\mathbf{u}\) is also continuous, we can employ the definition of continuity of function \(\theta_{1}\) at \(\mathbf{u}\) to obtain the following:
                \begin{equation}
                    \forall \varepsilon_{2} > 0 \implies \exists \delta > 0: |\theta_{1}(\hat{\mathbf{u}}) - \theta_{1}(\mathbf{u}) | < \varepsilon_{2}, \forall \hat{\mathbf{u}} \in \{\hat{\mathbf{u}} \in \mathcal{U}: \lVert \hat{\mathbf{u}} - \mathbf{u} \rVert < \delta\}.
                    \label{eq:ddp_continuity}
                \end{equation}
                The negation form of the statement in \eqref{eq:ddp_continuity} can be expressed as:
                \begin{equation}
                    \forall \delta > 0 \implies \exists \varepsilon_{2} > 0: |\theta_{1}(\hat{\mathbf{u}}) - \theta_{1}(\mathbf{u}) | < \varepsilon_{2}, \forall \hat{\mathbf{u}} \in \{\hat{\mathbf{u}} \in \mathcal{U}: \lVert \hat{\mathbf{u}} - \mathbf{u} \rVert < \delta\}.
                \end{equation}
                Solving the inequation above completes the proof.
            \end{proof}

        \subsubsection{Convergence of iLQR}
        \label{sec:ilqr_convergence_proof_subsection}
            \begin{theorem}
                Let the state-transition dynamics \(f\) and the cost function \(c\) have continuous second partial derivatives w.r.t the continuous state \(\mathbf{x}\) and action \(\mathbf{u}\). If \(\mathbf{u}(\varepsilon, \hat{\mathbf{u}})\) denotes the successive application of iLQR, then any accumulation point of \(\mathbf{u}(\varepsilon, \hat{\mathbf{u}})\) is stationary w.r.t the finite-horizon cost \(J(\mathbf{u}_{1:T})\).
            \end{theorem}
            \begin{proof}
                We first determine the condition that \(\mathbf{u}(\varepsilon, \hat{\mathbf{u}})\) calculated in \eqref{eq:action_backtracking_linesearch} is the successor of iLQR at \(\hat{\mathbf{u}}\).  According to iLQR algorithm, \(\mathbf{u}(\varepsilon, \hat{\mathbf{u}}) = \mathrm{iLQR}(\hat{\mathbf{u}})\) only if:
                \begin{equation}
                    J(\mathbf{u}(\varepsilon, \hat{\mathbf{u}})) - J(\hat{\mathbf{u}}) \le \frac{\varepsilon}{2} \theta_{1}(\hat{\mathbf{u}}).
                \end{equation}
                Note that from Lemma~\ref{lmm:A.4}:
                \begin{equation}
                    J(\mathbf{u}(\varepsilon, \hat{\mathbf{u}})) - J(\hat{\mathbf{u}}) \le \eta \varepsilon \theta_{1}(\hat{\mathbf{u}}) < \frac{\varepsilon}{2} \theta_{1}(\hat{\mathbf{u}}), \forall \eta \in (0.5, 1), \forall \varepsilon \in [0, \varepsilon_{1}]
                \end{equation}
                which indicates that \(\mathbf{u}(\varepsilon, \hat{\mathbf{u}})\) is a successor of iLQR when \(\varepsilon \in [0, \varepsilon_{1}]\).

                If \(\mathbf{u}(\varepsilon, \hat{\mathbf{u}})\) is a successor of iLQR, the acceptance criterion combining with the result in Lemma~\ref{lmm:A.3} leads to:
                \begin{equation}
                    J(\mathrm{iLQR}(\hat{\mathbf{u}})) - J(\hat{\mathbf{u}}) \le \frac{\varepsilon}{2} \theta_{1}(\hat{\mathbf{u}}) < \frac{\varepsilon}{2} \left[ \theta_{1}(\mathbf{u}) + \varepsilon_{2} \right], \forall \hat{\mathbf{u}} \in \{\hat{\mathbf{u}} \in \mathcal{U}: \lVert \hat{\mathbf{u}} - \mathbf{u} \rVert < \delta\}.
                \end{equation}
                Note that if \(\delta\) is set to be small enough, then \(\varepsilon_{2}\) is also very small, resulting in \(\theta_{1}(\mathbf{u}) + \varepsilon_{2} < 0\). If we set \(\delta(\hat{\mathbf{u}}) = \delta\) and \(\varepsilon(\hat{\mathbf{u}}) = \frac{\varepsilon}{2} \left[ \theta_{1}(\mathbf{u}) + \varepsilon_{2} \right]\), then iLQR satisfies the second condition in \cref{thrm:how_to_prove_convergence}.
            \end{proof}
\newpage
\section[Linearisation]{Linearise the state-transition dynamics}
\label{sec:linearisation}
    \subsection{Stochastic gradient descent}
    \label{sec:linearised_dynamics_sgd}
        The transition dynamics is given as:
            \begin{equation}
                \begin{aligned}[b]
                \mathbf{x}_{t + 1} & = f(\mathbf{x}_{t}, \mathbf{u}_{t}) = \mathbf{x}_{t} - \alpha \nabla_{\mathbf{x}_{t}} \left[ \mathbf{u}_{t}^{\top} \pmb{\ell}(\mathbf{x}_{t}) \right].
                \end{aligned}
                \label{eq:dynamics_sgd}
            \end{equation}
            
            Applying Taylor's expansion to the first order around a state-action pair \((\hat{\mathbf{x}}_{t}, \hat{\mathbf{u}}_{t})\) gives:
            \begin{equation}
                \begin{aligned}[b]
                    \mathbf{x}_{t + 1} & = \hat{\mathbf{x}}_{t + 1} + \left( \mathbf{I}_{D} - \alpha \eval{\nabla_{\mathbf{x}_{t}}^{2} \left[ \mathbf{u}_{t}^{\top} \pmb{\ell}(\mathbf{x}_{t}) \right]}_{\substack{\mathbf{x}_{t} = \hat{\mathbf{x}}_{t} \\ \mathbf{u}_{t} = \hat{\mathbf{u}}_{t}}} \right) \left( \mathbf{x}_{t} - \hat{\mathbf{x}}_{t} \right) - \alpha \eval{\nabla_{\mathbf{x}_{t}}^{\top} \left[ \pmb{\ell}(\mathbf{x}_{t}) \right]}_{\mathbf{x}_{t} = \hat{\mathbf{x}}_{t}} \left( \mathbf{u}_{t} - \hat{\mathbf{u}}_{t} \right).
                \end{aligned}
                \label{eq:dynamics_taylor_expansion}
            \end{equation}
            
            It can also be written as:
            \begin{equation}
                \delta \mathbf{x}_{t + 1} = \left( \mathbf{I}_{D} - \alpha \eval{\nabla_{\mathbf{x}_{t}}^{2} \left[ \mathbf{u}_{t}^{\top} \pmb{\ell}(\mathbf{x}_{t}) \right]}_{\substack{\mathbf{x}_{t} = \hat{\mathbf{x}}_{t} \\ \mathbf{u}_{t} = \hat{\mathbf{u}}_{t}}} \right) \delta \mathbf{x}_{t} + \left(- \alpha \eval{\nabla_{\mathbf{x}_{t}}^{\top} \left[ \pmb{\ell}(\mathbf{x}_{t}) \right]}_{\mathbf{x}_{t} = \hat{\mathbf{x}}_{t}} \right) \delta \mathbf{u}_{t}.
            \end{equation}

            Hence, the coefficient matrices of the Taylor's series for the state-transition dynamics following the SGD update can be expressed as:
            \begin{subequations}
                \begin{empheq}[box=\fbox]{align}
                    \mathbf{F}_{\mathbf{x}_{t}} & = \mathbf{I}_{D} - \alpha \eval{\nabla_{\mathbf{x}_{t}}^{2} \left[ \mathbf{u}_{t}^{\top} \pmb{\ell}(\mathbf{x}_{t}) \right]}_{\substack{\mathbf{x}_{t} = \hat{\mathbf{x}}_{t} \\ \mathbf{u}_{t} = \hat{\mathbf{u}}_{t}}} 
                    \label{eq:Fx_sgd}\\
                    \mathbf{F}_{\mathbf{u}_{t}} & = - \alpha \eval{\nabla_{\mathbf{x}_{t}}^{\top} \left[ \pmb{\ell}(\mathbf{x}_{t}) \right]}_{\mathbf{x}_{t} = \hat{\mathbf{x}}_{t}}. \label{eq:Fu_sgd}
                \end{empheq}
            \end{subequations}

    \subsection{Adam}
    \label{sec:linearised_dynamics_adam}
        The gradient-based update for the parameter of interest using Adam keeps track of the mean and variance:
            \begin{equation}
                \begin{dcases}
                    \mathbf{m}_{t} & = \beta_{1} \mathbf{m}_{t - 1} + (1 - \beta_{1}) \mathbf{J}_{t}^{\top} \mathbf{u}_{t} \\
                    \mathbf{v}_{t} & = \beta_{2} \mathbf{v}_{t - 1} + (1 - \beta_{2}) \left( \mathbf{J}_{t}^{\top} \mathbf{u}_{t} \right) \odot \left( \mathbf{J}_{t}^{\top} \mathbf{u}_{t} \right),
                \end{dcases}
            \end{equation}
            where:
            \begin{itemize}
                \item \(\mathbf{m}_{0} = \pmb{0}\), \(\mathbf{v}_{0} = \pmb{0}\),
                \item \(\mathbf{J}_{t}\) is the Jacobian matrix of the validation losses for tasks in a minibatch \(t\)-th:
                    \begin{equation}
                        \mathbf{J}_{t} = \nabla_{\mathbf{x}_{t}} \left[ \pmb{\ell}(\mathbf{x}_{t}) \right].
                    \end{equation}
                \item \(\odot\) is the elementwise multiplication.
            \end{itemize}
            
            The corrected-bias estimators are then defined as:
            \begin{equation}
                \begin{dcases}
                    \hat{\mathbf{m}}_{t} & = \frac{\mathbf{m}_{t}}{1 - \beta_{1}^{t}}\\
                    \hat{\mathbf{v}}_{t} & = \frac{\mathbf{v}_{t}}{1 - \beta_{2}^{t}}.
                \end{dcases}
            \end{equation}
            The update or state-transition dynamics is then given as:
            \begin{equation}
                \mathbf{x}_{t + 1} = f(\mathbf{x}_{t}, \mathbf{u}_{t}) = \mathbf{x}_{t} - \alpha \frac{\hat{\mathbf{m}}_{t}}{\sqrt{\hat{\mathbf{v}}_{t}} + \epsilon}.
                \label{eq:dynamics_adam}
            \end{equation}

            The update for a new state of the model parameter can also be written by substitution:
            \begin{equation}
                \begin{aligned}[b]
                    \mathbf{x}_{t + 1} & = \mathbf{x}_{t} - \alpha \frac{\mathbf{m}_{t}}{1 - \beta_{1}^{t}} \frac{1}{\sqrt{\frac{\mathbf{v}_{t}}{1 - \beta_{2}^{t}}} + \epsilon}.
                    % & = \mathbf{x}_{t} - \frac{\alpha}{1 - \beta_{1}^{t}} \frac{\beta_{1} \mathbf{m}_{t - 1} + (1 - \beta_{1}) \mathbf{J}_{t}^{\top} \mathbf{u}_{t}}{\sqrt{\frac{\beta_{2} \mathbf{v}_{t - 1} + (1 - \beta_{2}) \mathrm{diag}\left[ \mathbf{J}_{t}^{\top} \mathbf{u}_{t} \right] \left( \mathbf{J}_{t}^{\top} \mathbf{u}_{t} \right)}{1 - \beta_{2}^{t}}} + \epsilon}.
                \end{aligned}
            \end{equation}
            
            The approximation using Taylor's expansion up to the first order will result with the following matrices:
            \begin{subequations}
                \begin{empheq}[box=\fbox]{align}
                    \mathbf{F}_{\mathbf{x}_{t}} & = \mathbf{I}_{D} - \frac{\alpha}{1 - \beta_{1}^{t}} \eval{\nabla_{\mathbf{x}_{t}} \left[ \frac{\mathbf{m}_{t}}{\sqrt{\frac{\mathbf{v}_{t}}{1 - \beta_{2}^{t}}} + \epsilon} \right]}_{\hat{\mathbf{x}}_{t}, \hat{\mathbf{u}}_{t}} \\
                    \mathbf{F}_{\mathbf{u}_{t}} & = - \frac{\alpha}{1 - \beta_{1}^{t}} \eval{\nabla_{\mathbf{u}_{t}} \left[ \frac{\mathbf{m}_{t}}{\sqrt{\frac{\mathbf{v}_{t}}{1 - \beta_{2}^{t}}} + \epsilon} \right]}_{\hat{\mathbf{x}}_{t}, \hat{\mathbf{u}}_{t}}.
                \end{empheq}
            \end{subequations}

            % Taylor's expansion to the first order of the state update using Adam can be presented as:
            % \begin{equation}
            %     \begin{aligned}[b]
            %         \mathbf{x}_{t + 1} & \approx \hat{\mathbf{x}}_{t} + \left\{ \mathbf{I}_{D} - \frac{\alpha}{1 - \beta_{1}^{t}} \eval{\nabla_{\mathbf{x}_{t}} \left[ \frac{\mathbf{m}_{t}}{\sqrt{\frac{\mathbf{v}_{t}}{1 - \beta_{2}^{t}}} + \epsilon} \right]}_{\hat{\mathbf{x}}_{t}, \hat{\mathbf{u}}_{t}} \right\} \left( \mathbf{x}_{t} - \hat{\mathbf{x}}_{t} \right)\\
            %         & \quad - \frac{\alpha}{1 - \beta_{1}^{t}} \eval{\nabla_{\mathbf{u}_{t}} \left[ \frac{\mathbf{m}_{t}}{\sqrt{\frac{\mathbf{v}_{t}}{1 - \beta_{2}^{t}}} + \epsilon} \right]}_{\hat{\mathbf{x}}_{t}, \hat{\mathbf{u}}_{t}} \left( \mathbf{u}_{t} - \hat{\mathbf{u}}_{t} \right).
            %     \end{aligned}
            % \end{equation}
            
            For simplicity, we assume that \(\mathbf{m}_{t - 1}\) and \(\mathbf{v}_{t - 1}\) are constant w.r.t. both \(\mathbf{x}_{t - 1}\) and \(\mathbf{u}_{t - 1}\). %Also, since Hessian matrix is intractable for large-scaled models, we assume that the Jacobian vector \(\mathbf{g}_{t}\) is constant w.r.t. the state \(\mathbf{x}_{t}\).
            
            The gradient w.r.t. \(\mathbf{u}_{t}\) on \(\mathbf{m}_{t}\) and \(\mathbf{v}_{t}\) can be expressed as:
            \begin{equation}
                \begin{aligned}[b]
                    \nabla_{\mathbf{u}_{t}} \left[ \mathbf{m}_{t} \right] & = (1 - \beta_{1}) \mathbf{J}_{t}^{\top} \\
                    \nabla_{\mathbf{u}_{t}} \left[ \mathbf{v}_{t} \right] & = 2 (1 - \beta_{2}) \mathbf{J}_{t}^{\top} \odot \left( \mathbf{J}_{t}^{\top} \mathbf{u}_{t} \right).
                \end{aligned}
            \end{equation}
            
            Note that given two functions \(f\) (a notation for a general function, not the state-transition dynamic) and \(g\)
            \begin{equation}
                \begin{aligned}[b]
                    \nabla_{\mathbf{u}_{t}} \left[ \frac{f}{g} \right] & = \nabla_{\mathbf{u}_{t}} \left[ f \odot \frac{1}{g} \right]\\
                    & = \nabla_{\mathbf{u}_{t}} \left[ f \right] \odot \frac{1}{g} + f \odot \nabla_{\mathbf{u}_{t}} \left[\frac{1}{g} \right] \\
                    & = \nabla_{\mathbf{u}_{t}} \left[ f \right] \odot \frac{1}{g} - f \odot \frac{1}{g^{2}} \odot \nabla_{\mathbf{u}_{t}} \left[g \right].
                \end{aligned}
            \end{equation}
            Hence:
            \begin{equation}
                \begin{aligned}[b]
                    \nabla_{\mathbf{u}_{t}} \left[ \frac{\mathbf{m}_{t}}{\sqrt{\frac{\mathbf{v}_{t}}{1 - \beta_{2}^{t}}} + \epsilon} \right] & = \frac{\nabla_{\mathbf{u}_{t}} \left[ \mathbf{m}_{t} \right]}{\sqrt{\frac{\mathbf{v}_{t}}{1 - \beta_{2}^{t}}} + \epsilon} - \frac{\mathbf{m}_{t}}{\left( \sqrt{\frac{\mathbf{v}_{t}}{1 - \beta_{2}^{t}}} + \epsilon \right)^{2}} \odot \frac{\nabla_{\mathbf{u}_{t}} \left[ \mathbf{v}_{t} \right]}{2\sqrt{(1 - \beta_{2}^{t}) \mathbf{v}_{t}}}.
                \end{aligned}
            \end{equation}

            To calculate the gradient w.r.t. \(\mathbf{x}_{t}\), the update of Adam is rewritten as:
            \begin{equation}
                \begin{dcases}
                    \mathbf{m}_{t} & = \beta_{1} \mathbf{m}_{t - 1} + (1 - \beta_{1}) \nabla_{\mathbf{x}_{t}} \left[ \mathbf{u}_{t}^{\top} \pmb{\ell}(\mathbf{x}_{t}) \right] \\
                    \mathbf{v}_{t} & = \beta_{2} \mathbf{v}_{t - 1} + (1 - \beta_{2}) \nabla_{\mathbf{x}_{t}} \left[ \mathbf{u}_{t}^{\top} \pmb{\ell}(\mathbf{x}_{t}) \right] \odot \nabla_{\mathbf{x}_{t}} \left[ \mathbf{u}_{t}^{\top} \pmb{\ell}(\mathbf{x}_{t}) \right].
                \end{dcases}
            \end{equation}
            The gradient w.r.t. \(\mathbf{x}_{t}\) on \(\mathbf{m}_{t}\) and \(\mathbf{v}_{t}\) can be expressed as:
            \begin{equation}
                \begin{aligned}[b]
                    \nabla_{\mathbf{x}_{t}} \left[ \mathbf{m}_{t} \right] & = (1 - \beta_{1}) \nabla_{\mathbf{x}_{t}}^{2} \left[ \mathbf{u}_{t}^{\top} \pmb{\ell}(\mathbf{x}_{t}) \right] \\
                    \nabla_{\mathbf{x}_{t}} \left[ \mathbf{v}_{t} \right] & = 2 (1 - \beta_{2}) \nabla_{\mathbf{x}_{t}} \left[ \mathbf{u}_{t}^{\top} \pmb{\ell}(\mathbf{x}_{t}) \right] \odot \nabla_{\mathbf{x}_{t}}^{2} \left[ \mathbf{u}_{t}^{\top} \pmb{\ell}(\mathbf{x}_{t}) \right].
                \end{aligned}
            \end{equation}
            Hence:
            \begin{equation}
                \begin{aligned}[b]
                    \nabla_{\mathbf{x}_{t}} \left[ \frac{\mathbf{m}_{t}}{\sqrt{\frac{\mathbf{v}_{t}}{1 - \beta_{2}^{t}}} + \epsilon} \right] & = \frac{\nabla_{\mathbf{x}_{t}} \left[ \mathbf{m}_{t} \right]}{\sqrt{\frac{\mathbf{v}_{t}}{1 - \beta_{2}^{t}}} + \epsilon} - \frac{\mathbf{m}_{t}}{\left( \sqrt{\frac{\mathbf{v}_{t}}{1 - \beta_{2}^{t}}} + \epsilon \right)^{2}} \odot \frac{\nabla_{\mathbf{x}_{t}} \left[ \mathbf{v}_{t} \right]}{2\sqrt{(1 - \beta_{2}^{t}) \mathbf{v}_{t}}}.
                \end{aligned}
            \end{equation}

\section[Quadraticisation]{Quadraticise cost function w.r.t. state \texorpdfstring{\(\mathbf{x}_{t}\)}{xt}}
    \label{sec:quadraticisation}
        The cost function consists of two terms: validation loss on the query subset and penalisation of the action \(\mathbf{u}_{t}\).

        \subsection{Quadraticise validation loss}
        \label{sec:approximation_loss}
            The loss term in \eqref{eq:cost_function} can be approximated to second order as:
            \begin{equation}
                \begin{aligned}[b]
                    \pmb{1}_{M}^{\top} \pmb{\ell}(\mathbf{x}_{t}) & \approx \eval{\pmb{1}_{M}^{\top} \pmb{\ell}(\mathbf{x}_{t})}_{\mathbf{x}_{t} = \hat{\mathbf{x}}_{t}} + \left(\mathbf{x}_{t} - \hat{\mathbf{x}}_{t}\right)^{\top} \eval{\nabla_{\mathbf{x}_{t}} \left[ \pmb{1}_{M}^{\top} \pmb{\ell}(\mathbf{x}_{t}) \right]}_{\mathbf{x}_{t} = \hat{\mathbf{x}}_{t}} \\
                    & \qquad + \frac{1}{2} \left(\mathbf{x}_{t} - \hat{\mathbf{x}}_{t}\right)^{\top} \eval{\nabla_{\mathbf{x}_{t}}^{2} \left[ \pmb{1}_{M}^{\top} \pmb{\ell}(\mathbf{x}_{t}) \right]}_{\mathbf{x}_{t} = \hat{\mathbf{x}}_{t}} \left(\mathbf{x}_{t} - \hat{\mathbf{x}}_{t}\right).
                \end{aligned}
                \label{eq:taylor_approximation_validation_loss}
            \end{equation}
    \subsection{Quadraticise the penalisation of the action \texorpdfstring{\(\mathbf{u}_{t}\)}{ut}}
    \label{sec:approximation_prior_u}
        The penalisation term is already in the quadratic form. Here, it is rewritten to follow the Taylor's series form.
        \begin{equation}
            \begin{aligned}[b]
                & \frac{\beta_{u}}{2} \left( \mathbf{u}_{t} - \mu_{u} \pmb{1}_{M} \right)^{\top} \left( \mathbf{u}_{t} - \mu_{u} \pmb{1}_{M} \right) \\
                & = \frac{\beta_{u}}{2} \left[ \left(\mathbf{u}_{t} - \hat{\mathbf{u}}_{t} \right) + \hat{\mathbf{u}}_{t} - \mu_{u} \pmb{1}_{M} \right]^{\top} \left[ \left(\mathbf{u}_{t} - \hat{\mathbf{u}}_{t} \right) + \hat{\mathbf{u}}_{t} - \mu_{u} \pmb{1}_{M} \right] \\
                & = \frac{1}{2} \left(\mathbf{u}_{t} - \hat{\mathbf{u}}_{t} \right)^{\top} \left( \beta_{u} \mathbf{I}_{M} \right) \left(\mathbf{u}_{t} - \hat{\mathbf{u}}_{t} \right) + \left(\mathbf{u}_{t} - \hat{\mathbf{u}}_{t} \right)^{\top} \beta_{u} \left( \hat{\mathbf{u}}_{t} - \mu_{u} \pmb{1}_{M} \right).
            \end{aligned}
            \label{eq:taylor_approximation_prior_u}
        \end{equation}

Given the locally-quadratic approximation of the cost function in \eqref{eq:taylor_approximation_validation_loss} and \eqref{eq:taylor_approximation_prior_u}, the coefficient matrices and vector of the quadratic form of the cost function can be written as:
\begin{subequations}
\begin{empheq}[box=\fbox]{align}
    \mathbf{C}_{\mathbf{x}_{t},\mathbf{x}_{t}} & = \eval{\nabla_{\mathbf{x}_{t}}^{2} \left[ \pmb{1}_{M}^{\top}\pmb{\ell}(\mathbf{x}_{t}) \right]}_{\mathbf{x}_{t} = \hat{\mathbf{x}}_{t}}\\
    \mathbf{C}_{\mathbf{x}_{t},\mathbf{u}_{t}} & = \mathbf{C}_{\mathbf{u}_{t},\mathbf{x}_{t}} = \mathbf{0} \\
    \mathbf{C}_{\mathbf{u}_{t},\mathbf{u}_{t}} & = \beta_{u} \mathbf{I}_{M}\\
    \mathbf{c}_{\mathbf{x}_{t}} & = \eval{\nabla_{\mathbf{x}_{t}} \left[ \pmb{1}_{M}^{\top} \pmb{\ell}(\mathbf{x}_{t}) \right]}_{\mathbf{x}_{t} = \hat{\mathbf{x}}_{t}}\\
    \mathbf{c}_{\mathbf{u}_{t}} & = \beta_{u} \left(\hat{\mathbf{u}}_{t} - \mu_{u} \pmb{1}_{M} \right).
\end{empheq}
\end{subequations}
\newpage
\section{Trajectory optimisation algorithm(s)}
\label{sec:TO_algorithms}

        \begin{algorithm}[ht]
            \caption{Implementation of iLQR backward to determine the controller of interest.}
            \label{algm:approximate_lqr}
            \begin{algorithmic}[1]
                \Procedure{iLQRbackward}{$\{ \hat{\mathbf{x}}_{t}, \hat{\mathbf{u}}_{t} \}_{t = 1}^{T}$}
                    \State \(\mathbf{V}_{T + 1} = \mathbf{0}\), and \(\mathbf{v}_{T + 1} = \mathbf{0}\) \Comment{Quadratic matrix and vector of cost-to-go}
                    \State \(\theta_{T + 1} = 0\) \Comment{Stopping criterion for a nominal trajectory}
                    \For{\(t = T : 1\)} \Comment{Backward through time}
                        \State \(\mathbf{F}_{\mathbf{x}_{t}}, \mathbf{F}_{\mathbf{u}_{t}}\) = linearise dynamics \Comment{see \appendixname~\ref{sec:linearisation}}
                        \State \(\mathbf{C}_{\mathbf{x}_{t}, \mathbf{x}_{t}}, \mathbf{C}_{\mathbf{x}_{t}, \mathbf{u}_{t}}, \mathbf{C}_{\mathbf{u}_{t}, \mathbf{x}_{t}}, \mathbf{C}_{\mathbf{u}_{t}, \mathbf{u}_{t}}, \mathbf{c}_{\mathbf{x}_{t}}, \mathbf{c}_{\mathbf{u}_{t}}\) = quadraticise cost function \Comment{see \appendixname~\ref{sec:quadraticisation}}
                        \Statex
                        % Matrix Q
                        \State \(\mathbf{Q}_{\mathbf{x}_{t}, \mathbf{x}_{t}} = \mathbf{C}_{\mathbf{x}_{t}, \mathbf{x}_{t}} + \mathbf{F}_{\mathbf{x}_{t}}^{\top} \mathbf{V}_{t + 1} \mathbf{F}_{\mathbf{x}_{t}}\) \Comment{2nd derivatives of cost-to-go}
                        
                        \State \(\mathbf{Q}_{\mathbf{x}_{t}, \mathbf{u}_{t}} = \mathbf{F}_{\mathbf{x}_{t}}^{\top} \mathbf{V}_{t + 1} \mathbf{F}_{\mathbf{u}_{t}}\)
                        
                        \State \( \mathbf{Q}_{\mathbf{u}_{t}, \mathbf{x}_{t}} = \mathbf{F}_{\mathbf{u}_{t}} ^{\top} \mathbf{V}_{t + 1} \mathbf{F}_{\mathbf{x}_{t}}\)
                        
                        \State \( \mathbf{Q}_{\mathbf{u}_{t}, \mathbf{u}_{t}} = \mathbf{C}_{\mathbf{u}_{t}, \mathbf{u}_{t}} + \mathbf{F}_{\mathbf{u}_{t}}^{\top} \mathbf{V}_{t + 1} \mathbf{F}_{\mathbf{u}_{t}}\)
                        \Statex
                        % vector q
                        \State \(\mathbf{q}_{\mathbf{x}_{t}} = \mathbf{c}_{\mathbf{x}_{t}} + \mathbf{F}_{\mathbf{x}_{t}}^{\top} \mathbf{v}_{t + 1}\) \Comment{1st derivatives of cost-to-go}
                        \State \(\mathbf{q}_{\mathbf{u}_{t}} = \mathbf{c}_{\mathbf{u}_{t}} + \mathbf{F}_{\mathbf{u}_{t}}^{\top} \mathbf{v}_{t + 1}\)
                        \Statex
                        % Linear controller
                        \State \(\mathbf{K}_{t} = -\mathbf{Q}_{\mathbf{u}_{t}, \mathbf{u}_{t}}^{-1} \mathbf{Q}_{\mathbf{u}_{t}, \mathbf{x}_{t}}\) \Comment{Linear controller}
                        \State \(\mathbf{k}_{t} = -\mathbf{Q}_{\mathbf{u}_{t}, \mathbf{u}_{t}}^{-1} \mathbf{q}_{\mathbf{u}_{t}}\)
                        \Statex
                        % Value functions
                        % \State \(\mathbf{V}_{t} = \mathbf{Q}_{\mathbf{x}_{t}, \mathbf{x}_{t}} + \mathbf{Q}_{\mathbf{x}_{t}, \mathbf{u}_{t}} \mathbf{K}_{t} + \mathbf{K}_{t}^{\top} \mathbf{Q}_{\mathbf{u}_{t}, \mathbf{x}_{t}} + \mathbf{K}_{t}^{\top} \mathbf{Q}_{\mathbf{u}_{t}, \mathbf{u}_{t}} \mathbf{K}_{t}\) \Comment{2nd derivatives of value function}
                        \State \(\mathbf{V}_{t} = \mathbf{Q}_{\mathbf{x}_{t}, \mathbf{x}_{t}} + \mathbf{Q}_{\mathbf{x}_{t}, \mathbf{u}_{t}} \mathbf{K}_{t}\) \Comment{2nd derivatives of value function}
                        % \State \(\mathbf{v}_{t} = \mathbf{q}_{\mathbf{x}_{t}} + \mathbf{Q}_{\mathbf{x}_{t}, \mathbf{u}_{t}} \mathbf{k}_{t} + \mathbf{K}_{t}^{\top} \mathbf{q}_{\mathbf{u}_{t}} + \mathbf{K}_{t}^{\top} \mathbf{Q}_{\mathbf{u}_{t}, \mathbf{u}_{t}} \mathbf{k}_{t}\) \Comment{1st derivatives of value function}
                        \State \(\mathbf{v}_{t} = \mathbf{q}_{\mathbf{x}_{t}} + \mathbf{Q}_{\mathbf{x}_{t}, \mathbf{u}_{t}} \mathbf{k}_{t}\) \Comment{1st derivatives of value function}
                        \Statex
                        \State \(\theta_{t} = \theta_{t + 1} - \mathbf{q}_{\mathbf{u}_{t}}^{\top} \mathbf{Q}_{\mathbf{u}_{t}, \mathbf{u}_{t}} \mathbf{q}_{\mathbf{u}_{t}}\)
                    \EndFor
                    \State \textbf{return} \(\{ \mathbf{K}_{t}, \mathbf{k}_{t} \}_{t = 1}^{T}, \theta_{1}\)
                \EndProcedure
            \end{algorithmic}
        \end{algorithm}

\newpage
\section{Examples of loss functions satisfying \texorpdfstring{\cref{assumption:loss_boundedness_lipschitz,assumption:loss_smoothness,assumption:hessian_loss_lipschitz}}{}}
\label{sec:examples_loss}
    In this section, we analyse some common loss functions including mean square error (MSE) and cross-entropy (CE) to see if they satisfy \cref{assumption:loss_boundedness_lipschitz,assumption:loss_smoothness,assumption:hessian_loss_lipschitz} made in \cref{sec:convergence_analysis}.

    \begin{table}[bh]
        \centering
        \caption{Examples of some common loss functions used with the assumption on the linearity in \cref{eq:assumption_linear_transform} that satisfy \cref{assumption:loss_boundedness_lipschitz,assumption:loss_smoothness,assumption:hessian_loss_lipschitz} specified in \cref{sec:convergence_analysis}.}
        \label{tab:examples_loss}
        \begin{tabular}{l c c c c}
            \toprule
            \multirow{2}{*}{} & \multicolumn{2}{c}{\textbf{Loss} (\cref{assumption:loss_boundedness_lipschitz})} & \bfseries Lipschitz gradient & \bfseries Lipschitz Hessian \\
            \cmidrule{2-3} 
            & Bounded & Lipschitz & (\cref{assumption:loss_smoothness}) & (\cref{assumption:hessian_loss_lipschitz}) \\
            \midrule
            MSE & \multirow{2}{*}{\checkmark\tablefootnote{when the loss is clipped}} & \checkmark\tablefootnote{when the norm of parameter vector is bounded} & \checkmark & \checkmark \\
            CE & & \checkmark & \checkmark & \checkmark \\
            \bottomrule
        \end{tabular}
    \end{table}

    Before analysing the loss functions of interest, it is important to note that the objective function mostly depends on the model used. For example, if \(g(\mathbf{s}; \mathbf{x})\) denotes the output of a model parameterised by \(\mathbf{x}\) for the input \(\mathbf{s}\), then the objective function of interest is \(\ell(g(\mathbf{s}; \mathbf{x}), \mathbf{y})\). Depending on the model \(g\) used, the objective function is different and would be very complicated if \(g\) represents a deep neural network since it would result in a high-dimensional, non-linear, non-convex function. Analysing such general function is, however, beyond the scope of this paper. To simplify, we assume that \(g\) is a linear function of \(\mathbf{x}\):
    \begin{equation}
        g(\mathbf{s}; \mathbf{x}) = \mathbf{S} \mathbf{x}.
        \label{eq:assumption_linear_transform}
    \end{equation}

    \paragraph{Boundedness \texorpdfstring{in \cref{assumption:loss_boundedness_lipschitz}}{}} both the loss functions of interest, MSE and CE, are theoretically unbounded, and thus, does not satisfy this assumption. However, we can simply clip the loss to ensure the satisfaction of the boundedness assumption.

    \subsection{Mean square error}
    \label{sec:mse_loss_example}
        The loss function is defined as:
        \begin{equation}
            \ell(\mathbf{x}) = \norm{\mathbf{Sx} - \mathbf{y}}_{2}^{2}.
        \end{equation}

        \paragraph{Lipschitz continuity} In general, MSE does not satisfy the Lipschitz continuity property. However, when the vector norm of \(\mathbf{x}\) is bounded: \(\norm{\mathbf{x}} \le X_{0}\), then MSE is Lipschitz-continuous.
        \begin{proof}
            The Lipschitz continuity means:
            \begin{equation}
                \norm{\norm{ \mathbf{S} \mathbf{x}_{1} - \mathbf{y}}_{2}^{2} - \norm{ \mathbf{S} \mathbf{x}_{2} - \mathbf{y}}_{2}^{2} } \le \mathrm{const.} \, \norm{\mathbf{x}_{1} - \mathbf{x}_{2}}.
            \end{equation}

            The left hand side term can be expanded as follows:
            \begin{equation}
                \begin{aligned}[b]
                    & \norm{ \norm{ \mathbf{S} \mathbf{x}_{1} - \mathbf{y}}_{2}^{2} - \norm{ \mathbf{S} \mathbf{x}_{2} - \mathbf{y}}_{2}^{2} } \\
                    & = \norm{ \mathbf{x}_{1}^{\top} \mathbf{S}^{\top} \mathbf{Sx}_{1} - 2 \left(\mathbf{S x}_{1}\right)^{\top} \mathbf{y} - \mathbf{x}_{2}^{\top} \mathbf{S}^{\top} \mathbf{Sx}_{2} + 2 \left(\mathbf{S x}_{2}\right)^{\top} \mathbf{y} }\\
                    & = \norm{ \mathbf{x}_{1}^{\top} \mathbf{S}^{\top} \mathbf{Sx}_{1} - \mathbf{x}_{2}^{\top} \mathbf{S}^{\top} \mathbf{Sx}_{2} - 2 (\mathbf{x}_{1} - \mathbf{x}_{2})^{\top} \mathbf{S}^{\top} \mathbf{y} } \\
                    & = \norm{ (\mathbf{x}_{1} - \mathbf{x}_{2})^{\top} \mathbf{S}^{\top} \mathbf{Sx}_{1} + \mathbf{x}_{2}^{\top} \mathbf{S}^{\top} \mathbf{S} (\mathbf{x}_{1} - \mathbf{x}_{2}) - 2 (\mathbf{x}_{1} - \mathbf{x}_{2})^{\top} \mathbf{S}^{\top} \mathbf{y} } \\
                    & \quad \text{(triangle inequality)} \\
                    & \le \norm{ (\mathbf{x}_{1} - \mathbf{x}_{2})^{\top} \mathbf{S}^{\top} \mathbf{Sx}_{1} } + \norm{ \mathbf{x}_{2}^{\top} \mathbf{S}^{\top} \mathbf{S} (\mathbf{x}_{1} - \mathbf{x}_{2}) } + 2 \norm{ (\mathbf{x}_{1} - \mathbf{x}_{2})^{\top} \mathbf{S}^{\top} \mathbf{y} } \\
                    & \quad \text{(submultiplicative inequality)} \\
                    & \le \left( \norm{\mathbf{S}^{\top} \mathbf{Sx}_{1}} + \norm{\mathbf{x}_{2}^{\top} \mathbf{S}^{\top} \mathbf{S}} + 2 \norm{\mathbf{S}^{\top} \mathbf{y}} \right) \norm{ \mathbf{x}_{1} - \mathbf{x}_{2} } \\
                    & \le \left( \norm{\mathbf{S}^{\top} \mathbf{S}} \norm{\mathbf{x}_{1}} + \norm{\mathbf{S}^{\top} \mathbf{S} } \norm{\mathbf{x}_{2}} + 2 \norm{\mathbf{S}^{\top} \mathbf{y}} \right) \norm{ \mathbf{x}_{1} - \mathbf{x}_{2} } \\
                    & \le 2 \left( X_{0} \norm{\mathbf{S}^{\top} \mathbf{S}} + \norm{\mathbf{S}^{\top} \mathbf{y}} \right) \norm{ \mathbf{x}_{1} - \mathbf{x}_{2} }.
                \end{aligned}
            \end{equation}
            where the norm is the operator norm if the argument is a matrix.
        \end{proof}

        \paragraph{Smoothness} MSE is smooth, or its gradient is Lipschitz-continuous.
        \begin{proof}
            The gradient of MSE can be written as:
            \begin{equation}
                \grad \ell(\mathbf{x}) = 2 \mathbf{S}^{\top} (\mathbf{Sx} - \mathbf{y}).
            \end{equation}

            Therefore, for any \(\mathbf{x}_{1}\) and \(\mathbf{x}_{2}\):
            \begin{equation}
                \begin{aligned}[b]
                    \norm{\grad \ell(\mathbf{x}_{1}) - \grad \ell(\mathbf{x}_{2}) } & = 2 \norm{ \mathbf{S}^{\top} (\mathbf{Sx}_{1} - \mathbf{y}) - \mathbf{S}^{\top} (\mathbf{Sx}_{2} - \mathbf{y}) } \\
                    & = 2 \norm{ \mathbf{S}^{\top} \mathbf{S} (\mathbf{x}_{1} - \mathbf{x}_{2}) } \\
                    & \le 2 \norm{\mathbf{S}^{\top} \mathbf{S}} \norm{\mathbf{x}_{1} - \mathbf{x}_{2}}.
                \end{aligned}
            \end{equation}

            Thus, MSE in this setting is Lipschitz-continuous.
        \end{proof}

        \paragraph{Lipschitz-continuous Hessian} Since the Hessian in this case is constant, it also satisfies the Lipschitz-continuity.

        % \paragraph{Bounded variance of gradient} This assumption holds as the result of \cref{assumption:loss_boundedness_lipschitz}. As shown in \cref{eq:bounded_grad}, \cref{assumption:loss_boundedness_lipschitz} leads to the boundedness of the gradient. As a result, its variance is also bounded by simply applying the triangle inequality.

    \subsection{Cross-entropy loss}
    \label{sec:ce_loss_example}
        The loss function can be defined as:
        \begin{equation}
            \ell(\mathbf{x}) = -\mathbf{y}^{\top} \ln \left[ \mathrm{softmax} \left( \mathbf{Sx} \right) \right],
        \end{equation}
        where:
        \begin{equation}
            \mathrm{softmax} \left( \mathbf{Sx} \right) = \frac{\exp \left( \mathbf{Sx} \right)}{\pmb{1}^{\top} \exp \left( \mathbf{Sx} \right)}
        \end{equation}
        is the softmax function.

        \paragraph{Lipschitz continuity} The cross-entropy loss is Lipschitz-continuous.
        \begin{proof}
            To prove, we employ the mean-value theorem and then find the upper-bound of the gradient norm.

            First, we calculate the gradient of the softmax function w.r.t. each element \(x_{i}\) in \(\mathbf{x}\):
            \begin{equation}
                \begin{aligned}[b]
                    \pdv{\mathrm{softmax} \left( \mathbf{Sx} \right)_{j}}{x_{i}} & = \pdv{}{x_{i}} \frac{\exp \left( \mathbf{Sx} \right)_{j}}{\pmb{1}^{\top} \exp \left( \mathbf{Sx} \right)} \\
                    & = \pdv{}{(\mathbf{Sx})_{j}} \frac{\exp \left( \mathbf{Sx} \right)_{j}}{\pmb{1}^{\top} \exp \left( \mathbf{Sx} \right)} \times \pdv{(\mathbf{Sx})_{j}}{x_{i}} \\
                    & = \mathrm{softmax} (\mathbf{Sx})_{j} \left[ \mathbbm{1}(j = i) - \mathrm{softmax} (\mathbf{Sx})_{i} \right] \mathbf{S}_{ij},
                \end{aligned}
                \label{eq:derivative_softmax}
            \end{equation}
            where the subscript denotes an element in the corresponding vector and \(\mathbbm{1}(.)\) denotes the indicator function.

            Thus, the derivative of the cross-entropy loss w.r.t. the parameter of interest \(\mathbf{x}\) can be written following the chain rule:
            \begin{equation}
                \begin{aligned}[b]
                    \pdv{\ell(\mathbf{x})}{x_{i}} & = - \pdv{}{x_{i}} \sum_{j = 1}^{n_{c}} \mathbf{y}_{j} \ln \left[ \mathrm{softmax} \left( \mathbf{Sx} \right)_{j} \right] \\
                    & = - \sum_{j = 1}^{n_{c}} \frac{\mathbf{y}_{j}}{\mathrm{softmax}(\mathbf{Sx})_{j}} \times \pdv{\mathrm{softmax}(\mathbf{Sx})_{j}}{x_{i}}\\
                    & = \sum_{j = 1}^{n_{c}} \mathbf{y}_{j} \left[ \mathbbm{1}(j = i) - \mathrm{softmax} (\mathbf{Sx})_{i} \right] \mathbf{S}_{ij},
                \end{aligned}
                \label{eq:ce_sm_derivative}
            \end{equation}
            where \(n_{c}\) is the number of classes.

            Therefore, one can apply the triangle inequality to obtain the following:
            \begin{equation}
                \begin{aligned}[b]
                    \norm{\pdv{\ell(\mathbf{x})}{x_{i}}} & \le \sum_{j = 1}^{n_{c}} \norm{\mathbf{y}_{j} \left[ \mathbbm{1}(j = i) - \mathrm{softmax} (\mathbf{Sx})_{i} \right] \mathbf{S}_{ij}} \\
                    & \le \sum_{j = 1}^{n_{c}} \underbrace{\norm{\mathbf{y}_{j}}}_{\le 1} \times \underbrace{\norm{\left[ \mathbbm{1}(j = i) - \mathrm{softmax} (\mathbf{Sx})_{i} \right]}}_{\le 2} \times \norm{\mathbf{S}_{ij}} \\
                    & \le 2 \sum_{j = 1}^{n_{c}} \norm{\mathbf{S}_{ij}}.
                \end{aligned}
            \end{equation}

            Since each element of the gradient is bounded, the gradient norm is also bounded.

            By applying the mean-value theorem, we can obtain the following:
            \begin{equation}
                \begin{aligned}[b]
                    \norm{\ell(\mathbf{x}_{1}) - \ell(\mathbf{x}_{2})} = \sup_{\mathbf{z} \in [\mathbf{x}_{1}, \mathbf{x}_{2}]} \norm{\grad_{\mathbf{z}} \ell(\mathbf{z})} \times \norm{\mathbf{x}_{1} - \mathbf{x}_{2}},
                \end{aligned}
            \end{equation}
            where: \(\mathbf{z} \in [\mathbf{x}_{1}, \mathbf{x}_{2}]\) denotes a vector z contained in the set of points between \(\mathbf{x}_{1}, \mathbf{x}_{2} \in \mathbb{R}^{D}\).

            Since the gradient norm is bounded, the mean-value theorem results in the Lipschitz continuity for the cross-entropy loss.
        \end{proof}

        \paragraph{Smoothness} The cross-entropy loss in this case is smooth, or in other words, its gradient is Lipschitz-continuous.
        \begin{proof}
            According to \eqref{eq:ce_sm_derivative}, the difference of derivative between \(\mathbf{x}_{1}\) and \(\mathbf{x}_{2}\) can be written as:
            \begin{equation}
                \begin{aligned}[b]
                    \pdv{\ell(\mathbf{x}_{1})}{x_{i}} - \pdv{\ell(\mathbf{x}_{2})}{x_i} & = \sum_{j = 1}^{n_{c}} \mathbf{y}_{j} \left[ \mathrm{softmax}(\mathbf{Sx}_{2})_{i} - \mathrm{softmax}(\mathbf{Sx}_{1})_{i} \right] \mathbf{S}_{ij} \\
                    & = \left[ \mathrm{softmax}(\mathbf{Sx}_{2})_{i} - \mathrm{softmax}(\mathbf{Sx}_{1})_{i} \right] \sum_{j = 1}^{n_{c}} \mathbf{y}_{j} \mathbf{S}_{ij}.
                \end{aligned}
            \end{equation}

            Therefore:
            \begin{equation}
                \begin{aligned}[b]
                    \norm{ \grad_{\mathbf{x}} \ell(\mathbf{x}_{1}) - \grad_{\mathbf{x}} \ell(\mathbf{x}_{2}) } & = \sqrt{ \sum_{i = 1}^{D} \left\{ \left[ \mathrm{softmax}(\mathbf{Sx}_{2})_{i} - \mathrm{softmax}(\mathbf{Sx}_{1})_{i} \right] \sum_{j = 1}^{n_{c}} \mathbf{y}_{j} \mathbf{S}_{ij} \right\}^{2} } \\
                    & \le \max_{i \in \{1, \ldots, D\}} \abs{ \sum_{j = 1}^{n_{c}} \mathbf{y}_{j} \mathbf{S}_{ij} } \times \sqrt{ \sum_{i = 1}^{D} \left[ \mathrm{softmax}(\mathbf{Sx}_{2})_{i} - \mathrm{softmax}(\mathbf{Sx}_{1})_{i} \right]^{2} } \\
                    & \le \max_{i \in \{1, \ldots, D\}} \abs{ \sum_{j = 1}^{n_{c}} \mathbf{y}_{j} \mathbf{S}_{ij} } \times \norm{\mathrm{softmax}(\mathbf{Sx}_{1}) - \mathrm{softmax}(\mathbf{Sx}_{2}) }.
                \end{aligned}
            \end{equation}

            And since \(\mathrm{softmax}\) function is Lipschitz-continuous\footnote{see Proposition 4 in \say{On the properties of the softmax function with application in game theory and reinforcement learning}, arXiv preprint arXiv:1704.00805.}, the gradient of cross-entropy loss, in this case, is also Lipschitz-continuous.
        \end{proof}

        \paragraph{Lipschitz continous Hessian} The Hessian matrix of the cross-entropy loss is Lipschitz-continuous.
        \begin{proof}
            To prove the Lipschitz continuity of the Hessian matrix for the cross-entropy loss, we will prove that the norm of its third order derivative tensor is bounded. According to \citep[Page 175]{nesterov2003introductory}, if the norm of the third order tensor, also known as \emph{Terssian}, is bounded, then the Hessian matrix will satisfies the Lipschitz continuity. To do this, we calculate each element in the Terssian tensor as follows:

            The second derivative of the cross-entropy loss w.r.t. an element \(x_{i}\) can be written as:
            \begin{equation}
                \begin{aligned}[b]
                    \pdv{\ell(\mathbf{x})}{x_{i}}{x_{k}} & = \pdv{}{x_{k}} \sum_{j = 1}^{n_{c}} \mathbf{y}_{j} \left[ \mathbbm{1}(j = i) - \mathrm{softmax} (\mathbf{Sx})_{i} \right] \mathbf{S}_{ij} \\
                    & = -\pdv{\mathrm{softmax} (\mathbf{Sx})_{i}}{x_{k}} \times \sum_{j = 1}^{n_{c}} \mathbf{y}_{j} \mathbf{S}_{ij} \\
                    & = - \mathrm{softmax} (\mathbf{Sx})_{k} \left[ \mathbbm{1}(k = i) - \mathrm{softmax} (\mathbf{Sx})_{i} \right] \mathbf{S}_{ik} \times \sum_{j = 1}^{n_{c}} \mathbf{y}_{j} \mathbf{S}_{ij}.
                \end{aligned}
            \end{equation}

            Thus, the third order derivative can also be derived as:
            \begin{equation}
                \begin{aligned}[b]
                    \frac{\partial^{3} \ell(\mathbf{x})}{\partial x_{i} \partial x_{k} \partial x_{l}} & = - \mathbf{S}_{ik} \times \left( \sum_{j = 1}^{n_{c}} \mathbf{y}_{j} \mathbf{S}_{ij} \right) \times \pdv{}{x_{l}} \left[ \mathrm{softmax} (\mathbf{Sx})_{k} \left[ \mathbbm{1}(k = i) - \mathrm{softmax} (\mathbf{Sx})_{i} \right] \right].
                \end{aligned}
                \label{eq:derivative_third_order}
            \end{equation}

            To this point, we can see that the third order derivative involves the derivative of softmax functions and some constants relating to input data \(\mathbf{S}\) and its label \(\mathbf{y}\). Also, as the derivative of softmax function shown in \eqref{eq:derivative_softmax} consists of softmax and indicator functions, the derivative of softmax function is bounded. Thus, we can conclude that the third order derivative in \eqref{eq:derivative_third_order} is also bounded with the bound depending on the norm of the input \(\max_{i, j} \abs{\mathbf{S}_{ij}}\) and label \(\max_{j} \abs{\mathbf{y}_{j}}\).

            When the norm of each element in the Terssian is bounded, its norm is also bounded. Thus, according to the result in \citep[Page 175]{nesterov2003introductory}, we can conclude that the Hessian is Lipschitz-continuous.
        \end{proof}

    \subsection{Values of the Lipschitz continuity constants}
        The values of the constants used in \cref{assumption:loss_boundedness_lipschitz,assumption:loss_smoothness,assumption:hessian_loss_lipschitz} heavily depend on some other parameters, including the maximum value of norm of the inputs and labels as well as the clipping norm value of the parameter \(\mathbf{x}\) (in case of MSE). Thus, we cannot give exact values in numbers for those constants, but only provide a few expressions to show the dependency of their values on the input and output.

        \begin{table}[h]
            \centering
            \caption{The values of some Lipschitz-continuity constants assumed in \cref{assumption:loss_boundedness_lipschitz,assumption:loss_smoothness,assumption:hessian_loss_lipschitz}.}
            \label{tab:my_label}
            \begin{tabular}{l c c}
                \toprule
                \bfseries Loss function & \(L\) & \(S\) \\
                \midrule
                MSE & \(\sup_{\mathbf{S}, \mathbf{y}} 2 \left( X_{0} \norm{\mathbf{S}^{\top} \mathbf{S}} + \norm{\mathbf{S}^{\top} \mathbf{y}} \right)\) & \(\sup_{\mathbf{S}} 2 \norm{\mathbf{S}^{\top} \mathbf{S}}\) \\
                CE & \(\sup_{\mathbf{S}} 2 \norm{\mathbf{S}}_{1}\) & - \\
                \bottomrule
            \end{tabular}
        \end{table}

\end{document}